%% file: main.tex
\documentclass[twoside]{article}
\usepackage[preprint]{aistats2027}
\usepackage[round]{natbib}
\usepackage{microtype}
\usepackage{graphicx}
\usepackage{subcaption}
\usepackage{booktabs}
\usepackage{longtable}
\usepackage{amsmath,amssymb,amsthm,mathtools}
\usepackage{thmtools,thm-restate}
\usepackage{algorithm,algorithmic}
\usepackage{xcolor}
\usepackage{xurl}
\definecolor{citationdarkblue}{RGB}{0,70,140}
\usepackage[colorlinks=true,citecolor=citationdarkblue,
 linkcolor=citationdarkblue,urlcolor=citationdarkblue]{hyperref}
\usepackage[capitalize,noabbrev]{cleveref}
\usepackage{enumitem}
\usepackage{tikz}
\usetikzlibrary{arrows.meta,positioning,calc,fit,shapes.misc,shapes.geometric,bayesnet}
\usepackage[disable,textsize=tiny]{todonotes}
\usepackage{appendix}
\usepackage{titletoc}

\newcommand{\mat}[1]{\boldsymbol{#1}}
\newcommand{\vect}[1]{\boldsymbol{#1}}

\newtheorem{assumption}{Assumption}

\newtheorem{theorem}{Theorem}
\newtheorem{lemma}[theorem]{Lemma}
\newtheorem{proposition}[theorem]{Proposition}

\newtheorem{corollary}[theorem]{Corollary}

\begin{document}
\runningtitle{Learning-to-Defer in Non-Stationary Time Series}
\runningauthor{Montreuil, Yu, Carlier, Ng, and Ooi}
\twocolumn[
\aistatstitle{Learning-to-Defer in Non-Stationary Time Series\\via Switching State-Space Models}
\aistatsauthor{%
\textbf{Yannis Montreuil\textsuperscript{1}, Letian Yu\textsuperscript{1}, Axel Carlier\textsuperscript{2}}\\
\textbf{Lai Xing Ng\textsuperscript{3}, Wei Tsang Ooi\textsuperscript{1}}}
\aistatsaddress{%
\textsuperscript{1}School of Computing, National University of Singapore\\
\textsuperscript{2}Universit\'e de Toulouse, F\'ed\'eration ENAC ISAE-SUPAERO ONERA\\
\textsuperscript{3}Institute for Infocomm Research, A*STAR, Singapore}
]

\raggedbottom
\setlength{\floatsep}{8pt plus 2pt minus 2pt}
\setlength{\abovecaptionskip}{4pt}
\begin{abstract}
Learning-to-defer (L2D) lets a predictor decide, at each round, whether to
issue its own forecast or pay for an expert's. In non-stationary time series
this decision must keep adapting, although deployment reveals only the
consulted expert's forecast while a historical archive records every expert
with the target. L2D-SLDS learns from this archive a switching state-space
model of the target and all expert forecasts, whose shared and
expert-specific states describe how experts move together and apart. Its
predictive law supplies the internal forecast and the expected cost of every
consultation, which a greedy router minimizes, and one consultation also
updates the beliefs about unconsulted and unavailable experts. We prove
sublinear regret against a changing conditional-risk oracle without
exploration, when the candidate models are accurate and either the archive
separates them or live feedback reveals cost differences. On three real
datasets, L2D-SLDS has the lowest cost among eight bandit routers and adapts
its consultation rate to the fee.
\end{abstract}

\input{Section/Introduction}
\input{Section/Relatedworks}
\input{Section/Background}
\input{Section/Approach}
\input{Section/Evaluation}
\input{Section/Conclusion_and_Impact}
\clearpage
\input{Section/AI_Statement}

\bibliography{biblio}
\bibliographystyle{plainnat}
\input{Section/Checklist}
\input{Section/Appendix}

\end{document}

%% file: Section/Introduction.tex
\section{Introduction}
\label{sec:introduction}

Learning-to-defer (L2D) studies a system that can either issue its own
prediction or hand the input to an external expert
\citep{madras2018predict,mozannar2021consistent}. In forecasting, the system
keeps its internal forecast or pays a fee to obtain an expert's, and
consulting is worthwhile when the expert's expected squared error plus the fee
is smaller than that of the internal forecast. Supervised L2D learns this
comparison from examples that record the target together with every expert's
prediction, jointly with the internal predictor or after it has been fitted
\citep{Narasimhan,mao2024regressionmultiexpertdeferral}.

In time series, this comparison changes over time: a forecaster that is
reliable under ordinary demand may degrade during a demand shift. The rule
must therefore be updated during deployment, where an expert's forecast may
require a paid service, a specialist consultation, or an expensive
computation, so only the consulted expert's forecast is revealed with the
target. This gives \emph{bandit feedback} for external predictions
\citep{montreuil2026online}. Contextual and non-stationary bandits learn from
such feedback \citep{li2010contextual,russac2019weighted}, as does prediction
with paid observations \citep{seldin2014prediction}, but they learn the loss
of each action. One subtlety is that the historical archive records more than
losses: it records each available expert's forecast together with the target on the
same date, and therefore how the forecasts move together. The loss-based
baselines do not explicitly model this joint forecast dependence, while
forecast combination methods, which model it
\citep{bates1969combination,mcalinn2019dynamic}, assume that every forecast
is observed at every date.

We develop \textbf{L2D-SLDS}, which learns a joint predictive law of the target
and the expert forecasts with a switching linear dynamical system
\citep{ghahramani2000variational}. A shared latent state describes variation
common to the experts, private states describe each expert's own variation,
and a hidden regime selects their dynamics. The model is fitted on the
archive, where every available expert is observed. During deployment, an
interacting multiple-model filter \citep{blom1988interacting} conditions it
on the target and, when an expert is consulted, on that expert's forecast, so
the two phases differ only in which channels are observed. The resulting law
gives the internal forecast and each action's expected squared error plus fee,
including the covariance between the expert and the target, and a greedy
router selects the action with the smallest predicted cost.

Because prediction and deferral share this law, a consulted forecast also
refines the beliefs about the other experts and the later internal forecasts.
The router therefore needs no exploration parameter: information about cost
differences comes from live feedback or from the archive, which can separate
accurate candidate models from inaccurate ones before feedback becomes
selective.

\noindent\textbf{Contributions.}
\begin{enumerate}[label=\textbf{\roman*)},leftmargin=*,nosep]
\item \textbf{Joint prediction and deferral under changing regimes.} A
switching state-space model whose shared and private states describe how
experts move together and apart as regimes change; one predictive law
supplies the internal forecast and every consultation cost, including the
expert--target covariance, before any expert is consulted
(Sections~\ref{sec:pr-model} and~\ref{sec:pr-routing}).
\item \textbf{Transfer across a changing pool of experts.} One consultation
updates the beliefs about unqueried and unavailable experts; experts can
leave, return with their current belief, or join with supplied emission parameters
and a private prior without refitting existing experts. A
structured variant has storage linear in the number of experts while
preserving current routing costs at each projection (Section~\ref{sec:pr-learning}).
\item \textbf{Regret without exploration.} Under explicit predictive-accuracy
and information conditions, sublinear regret against an
oracle that may change its action every round, from live feedback or from
the archive alone, with explicit $O(\sqrt T)$ rates and no term counting
changes of the best expert (Corollary~\ref{cor:pr-sublinear}).
\item \textbf{Evaluation of performance, mechanisms, and scaling.}
Comparisons with eight bandit routers on three real datasets give the
lowest mean cost, and a Jena fee sweep spans $100\%$ to $0\%$ consultation
(Table~\ref{tab:main-climate-cost}, Figure~\ref{fig:jena-consultation-cost}).
Routing-rule and covariance ablations, synthetic switching and stationary
controls, expert withdrawal and reentry, pool scaling to 64 experts, and
runtime and storage comparisons with baselines examine the model's
mechanisms and computational cost
(Appendices~\ref{app:greedy-policy-comparison}--\ref{app:jena-computational-scaling}).
\end{enumerate}

%% file: Section/Relatedworks.tex
\section{Related Work}
\label{sec:related-work}

\noindent\textbf{Learning-to-defer.}
L2D extends selective prediction
\citep{Chow_1970,Bartlett_Wegkamp_2008,Geifman_El-Yaniv_2017}
by allowing the system to choose an external expert when its expected cost
is lower \citep{madras2018predict,mozannar2021consistent,Verma2022LearningTD}.
Surrogate losses give consistency guarantees for this decision, including
multi-expert and regression formulations
\citep{charusaie2022sample,mao2023twostage,mao2024regressionmultiexpertdeferral,mao2025mastering,montreuil2026beyond},
and recent extensions address top-$k$ consultation, expert-conditioned
advice, and query budgets
\citep{montreuil2026why,montreuil2026learningtodeferexpertconditionedadvice,desalvo2025budgeted}.
Closer to our setting, deferral has been studied in non-stationary sequential
decision-making through model-based reinforcement learning
\citep{joshi2021learning}, with changing expert pools
\citep{montreuil2026online}, and with missing expert predictions in batch
data \citep{nguyen2025probabilistic}. Unlike deferral as control, our
routing decision changes which forecasts are observed but not the target.

\noindent\textbf{Partial feedback and historical information.}
Feedback-graph and contextual bandits study decisions with side observations
or selectively revealed costs
\citep{mannor2011bandits,alon2015online,li2010contextual}, and bandit feedback
also arises in online Markov decision processes with adversarial rewards
\citep{neu2010online}. Non-stationary methods adapt through discounting or
change detection \citep{russac2019weighted,liu2018change,besson2022efficient},
and historical data can reduce exploration for stationary arms
\citep{shivaswamy2012history}. Greedy contextual guarantees rely on
information conditions such as covariate diversity \citep{bastani2021mostly};
ours instead concern observable cost differences or historical model
separation, complementing information-theoretic analyses of non-stationary
bandits \citep{min2023information} (see Appendix~\ref{app:related-work}).

%% file: Section/Background.tex
\section{Background}
\label{sec:background}

\subsection{Learning-to-Defer}
\label{sec:background_l2d}
We begin with the usual supervised deferral decision, which will become
a sequential forecasting decision in Section~\ref{sec:pr-problem}.
For a context $\vect{x}\in\mathcal X$ and scalar target $y$, let
$h_0(\vect{x})=f_\theta(\vect{x})$ be the internal prediction and
$h_k(\vect{x})$ the prediction of external expert $k\in\{1,\ldots,K\}$.
Here $\theta$ denotes the internal predictor's parameters.
Using squared loss and fees $\beta_0=0$, $\beta_k\geq0$, the cost
of choosing $k$ is
\begin{equation}
\label{eq:l2d_cost}
 C_k(\vect{x},y)=(h_k(\vect{x})-y)^2+\beta_k.
\end{equation}
A router $\pi(\cdot\mid\vect{x})$ is a distribution over
$\{0,\ldots,K\}$. Standard offline L2D minimizes the population risk
\begin{equation}
\label{eq:l2d_objective}
 \mathcal R(\pi,\theta)
 =\mathbb E_{(\vect{x},y)\sim\mathcal D}
       \sum_{k=0}^K\pi(k\mid\vect{x})C_k(\vect{x},y),
\end{equation}
using supervised expert comparisons
\citep{madras2018predict,mozannar2021consistent,mao2024regressionmultiexpertdeferral}.
We retain historical expert supervision, but allow the predictive
distribution and the available feedback to change during deployment.

\subsection{Switching State-Space Models}
\label{sec:background_slds}
A state-space model describes observations through a latent process
that carries information across time \citep{shumway2006time}.
For example, a linear-Gaussian model has
\begin{align}
 \vect{s}_t&=\mat A\vect{s}_{t-1}+\vect{\omega}_t,
 &\vect{\omega}_t&\sim\mathcal N(\vect0,\mat Q),
 \label{eq:lgssm_state}\\
 o_t&=\mat H\vect{s}_t+\varepsilon_t,
 &\varepsilon_t&\sim\mathcal N(0,R).
 \label{eq:lgssm_obs}
\end{align}
Here $\mat A$ propagates the state, $\mat H$ maps it to the observation,
and $\mat Q,R$ describe process and observation noise. With a Gaussian
initial state and independent innovations, Kalman filtering updates
the Gaussian state law from each new observation
\citep{kalman1960new,welch1995introduction}.

An SLDS adds a discrete regime $z_t\in\{1,\ldots,M\}$ that selects
the dynamics and noise parameters
\citep{bengio1994input,ghahramani2000variational,fox2008nonparametric,hu2024modeling,geadah2024parsing}.
Conditioning on the observed channels then updates every state coordinate
correlated with them, which in our construction includes the states of
unobserved experts. Bold lowercase and uppercase symbols denote vectors and
matrices.

%% file: Section/Approach.tex
\section{Probabilistic Learning and Routing}
\label{sec:pr-main}

We present L2D-SLDS in the order in which information becomes available:
historical joint observations establish the relationships between experts,
which deployment then uses to forecast, route, and learn from selective
feedback.

\begin{figure}[t]
\centering
\resizebox{0.92\linewidth}{!}{\input{figures/probabilistic_model}}
\caption{One generative step of L2D-SLDS. The regime selects dynamics
and emission variances; the plate includes unavailable retained experts.
After acting, $r_t$ is always observed and $d_{t,k}$ only if $I_t=k$.
The shared state $\vect g_t$ links queried and unqueried experts.}
\label{fig:l2d-slds-graph}
\end{figure}

\subsection{Historical Supervision and Selective Access}
\label{sec:pr-problem}
At round $t$, the system observes context $\vect{x}_t\in\mathcal X$
and the available external experts
$\mathcal E_t\subseteq\{1,\ldots,K\}$. It produces an internal
forecast $f_t$ and chooses $I_t$ from
$\mathcal A_t=\{0\}\cup\mathcal E_t$. Action zero issues $f_t$;
action $k\ne0$ obtains and issues expert $k$'s forecast $h_{t,k}$.
Writing $h_{t,0}=f_t$, the realized cost is
\[
 C_{t,k}=(h_{t,k}-y_t)^2+\beta_k,\qquad
 \beta_0=0,\quad 0\leq\beta_k\leq\beta_{\max}.
\]
Unlike $f_\theta$ in Section~\ref{sec:background_l2d}, the internal
forecast is recomputed at every round from a learned predictive law. The target $y_t$ is observed after the forecast
interval, whereas an external forecast is observed only when it is
selected. Consultation therefore changes which quantities are observed,
not the target or the forecasts themselves, and each interval is
completed before the next decision.

Contexts, availabilities, targets, and forecasts are generated by an
unknown law, which we call the truth, and our guarantees are stated
against it. The router acts on the information $\mathcal P_t$ available
before round $t$. It consists of the archive, the past contexts,
availabilities, targets, actions, and queried forecasts, and the current
$\vect x_t$ and $\mathcal E_t$; at a calibration date $s$, it consists of
the archive before $s$, the current context, and the available archived
channels. Conditional expectations
given $\mathcal P_t$ are taken under the truth, so they describe what can
be predicted from the information the router actually holds.

The archive records the target and the issued forecast of every
then-available expert. We split it chronologically: a fitting segment
learns the structural parameters, and a later calibration segment
initializes the filter states and model weights. Its forecasts are issued
or rolling-origin out-of-sample predictions, so each historical comparison
uses only information available at its forecast date
(Appendix~\ref{app:historical-fitting}). Targets and forecasts lie in a
known open range $(-B,B)$, so $C_*=4B^2+\beta_{\max}$ bounds every cost.
For an unbounded response, a bounded transformation fitted on historical
data is applied to the target and every forecast, and the objective is
squared error on that scale, as in our experiments
(Appendix~\ref{app:forecast-timing}).

\subsection{A Joint Model of Targets and Expert Forecasts}
\label{sec:pr-model}
To compare the actions, the system needs a joint predictive law of the
target and the expert forecasts; modeling them jointly is also what allows
one observed forecast to inform another. We write this law in inverse-link
coordinates, using $b(v)=B\tanh(v)$ and its inverse
$\iota(y)=\operatorname{arctanh}(y/B)$ on $(-B,B)$. A reference predictor
$f_{\rm ref}:\mathcal X\to\mathbb R$ on the inverse-link scale is fixed
after the fitting segment (in our experiments,
$f_{\rm ref}(\vect x_t)=\iota(y_{t-1})$), and we define
\[
 r_t=f_{\rm ref}(\vect{x}_t)-\iota(y_t),\qquad
 d_{t,k}=\iota(h_{t,k})-f_{\rm ref}(\vect{x}_t).
\]
The variable $r_t$ describes how the target departs from the reference,
and $d_{t,k}$ describes how expert $k$ disagrees with it; their sum is the
expert's inverse-link residual. Because the reference function is fixed during
deployment, these channels describe the target and the experts rather than
the internal forecast, which changes as the system learns.

Each \emph{candidate}, indexed by $j$, is a fitted instance of the
following model. Variation shared by the experts is carried by a state
$\vect g_t\in\mathbb R^{d_g}$, and variation specific to expert $k$ by a
private state $\vect u_{t,k}\in\mathbb R^{d_\alpha}$. A regime
$z_t\in\{1,\ldots,M\}$ selects their dynamics:
\begin{equation}
\label{eq:main-state-model}
\begin{aligned}
 \vect g_t&=\mat A^{(g)}_{z_t}\vect g_{t-1}+\vect\eta^{(g)}_t,\\
 \vect u_{t,k}&=\mat A^{(u)}_{z_t}\vect u_{t-1,k}+\vect\eta^{(u)}_{t,k},\\
 r_t&=\vect\psi_t^\top\mat\Gamma_{\rm ref}\,\vect g_t+\varepsilon_{t,r},\\
 d_{t,k}&=\vect\psi_t^\top(\mat\Gamma_k\vect g_t+\vect u_{t,k})
          +\varepsilon_{t,k}.
\end{aligned}
\end{equation}
The feature map $\vect\psi_t=\vect\psi(\vect x_t)\in\mathbb R^{d_\alpha}$
is declared before fitting, and the loadings
$\mat\Gamma_{\rm ref},\mat\Gamma_k\in\mathbb R^{d_\alpha\times d_g}$
describe how the shared state enters the target and each expert. The
regime follows a row-stochastic matrix $\mat\Pi$, with
$\Pr(z_t=m\mid z_{t-1}=\ell)=\mat\Pi_{\ell m}$, and given the regime path
the innovations are independent centered Gaussians with covariances
$\mat Q_m^{(g)},\mat Q_m^{(u)}\succ0$ and emission variances
$R_{m,r},R_{m,k}>0$, starting from a Gaussian state independent of the
initial regime.

Each component has a distinct role (Figure~\ref{fig:l2d-slds-graph}). The
shared state lets a change observed through the target or through one
expert move the predicted forecasts of the others; the private states keep
the variation that belongs to a single expert; and the regime lets both,
and the relationships between them, change over time
(synthetic controls in Appendix~\ref{app:completed-mechanism-controls}). The private dynamics
are common to all experts, which differ through their loadings, emission
variances, and private trajectories, so adding an expert adds one
regime-independent loading matrix and one emission variance per regime,
rather than a new dynamical model. Within each candidate,
context and availability are conditionally exogenous given the observed past:
they specify the emission rows and action set without an additional
latent-state likelihood. Every predictive and
routing quantity is invariant to invertible changes of shared-state coordinates
(Appendices~\ref{app:red-anchor} and~\ref{app:pr-identification}). To
obtain costs, the joint law is mapped back through
$Y=b(f_{\rm ref}(\vect{x}_t)-r)$ and $H_k=b(f_{\rm ref}(\vect{x}_t)+d_k)$,
the target and the forecast of expert $k$ as random variables under the
candidate. The loss is evaluated on $H_k-Y$, so the link belongs to the
observation model and the objective remains squared prediction error.

\subsection{Selective Updates and Cross-Expert Information}
\label{sec:pr-learning}
Within each regime, a candidate maintains a Gaussian distribution over the
stacked shared and private states with its full covariance. The experts
whose states it keeps form its \emph{retained registry}, which includes
fitted experts that are currently unavailable. During historical replay,
the target and every available expert are observed; during deployment,
the target and at most one consulted expert are. Both phases therefore
apply the same Gaussian conditioning to different observed channels, and
the dependence learned from the full archive remains available once
feedback becomes selective. An interacting multiple-model recursion
\citep{blom1988interacting} propagates the regime probabilities and
replaces each mixture over predecessor regimes by the Gaussian with the
same mean and covariance. This moment matching keeps switching inference
tractable, and it is exact for a single regime
(Appendix~\ref{app:dense-filter}). To see how selective feedback transfers
information, fix one regime after the target has been assimilated, and
consider what the forecast of one queried expert says about another.

\begin{restatable}[Cross-expert conditioning within a regime]{proposition}{CrossExpertConditioning}
\label{prop:main-expert-transfer}
Within a jointly Gaussian component after target assimilation,
let expert $k$ be queried and let $\ell\ne k$ be any retained
expert, available or not. Let $\mu_\ell,\mu_k$, $\sigma_{\ell\ell},\sigma_{kk}$,
and $\sigma_{\ell k}$ be the conditional means, variances, and covariance
of $d_\ell,d_k$, where $\sigma_{kk}>0$ includes emission noise.
Conditioning on $d_k$ gives
\begin{equation}
\label{eq:main-expert-transfer}
\begin{aligned}
 \mu_\ell^+&=\mu_\ell+
       \frac{\sigma_{\ell k}}{\sigma_{kk}}(d_k-\mu_k),\\
 \sigma_{\ell\ell}^+&=\sigma_{\ell\ell}
                         -\frac{\sigma_{\ell k}^2}{\sigma_{kk}}.
\end{aligned}
\end{equation}
Hence $\sigma_{\ell\ell}^+\leq\sigma_{\ell\ell}$, with equality if and
only if $\sigma_{\ell k}=0$. The same formulas hold for every scalar
linear function $\vect a^\top\vect s$ of the stacked state, using its
mean, variance, and covariance with $d_k$ in place of
$\mu_\ell,\sigma_{\ell\ell},\sigma_{\ell k}$.
\end{restatable}

Within this component, the fractional variance reduction for $d_\ell$
equals its squared conditional correlation with $d_k$.
The most useful choice of $\vect a$ is a coordinate of the shared state or
of expert $\ell$'s private state. These states persist, so the update
carries into the next forecasts of expert $\ell$ and of the target, even
though $h_{t,\ell}$ itself is never observed (Appendix~\ref{app:dense-filter}).
The query also reweights the regimes through its likelihood, so the
variance reduction holds within each regime rather than pointwise for the
mixture. In both cases only observed channels enter the update, and an
unqueried forecast is inferred from the model rather than used as a label.

\noindent\textbf{A changing pool of experts.}
Because L2D-SLDS routes from a retained joint belief, the pool can change
without refitting the existing experts.
When expert $\ell$ becomes unavailable, its channel leaves the observed
set, but its private state keeps propagating under the fitted dynamics and
is corrected whenever it is correlated with the target or a queried expert.
When $\ell$ returns, routing uses this current belief instead of
restarting (Proposition~\ref{prop:retained-registry}; reentry experiments in
Appendix~\ref{app:extended-deployed-reentry}). A new expert enters
with declared emission parameters and a private prior, and immediately uses the
current shared belief (Appendix~\ref{app:pr-identification}). Selective
feedback can also benefit future internal prediction:
Proposition~\ref{prop:pr-query-value} expresses its conditional Bayes-risk
reduction as the expected squared change in the future target's conditional mean.

\noindent\textbf{Cost in the number of experts.}
The fitted dynamics and emission parameters grow linearly with $K$, since each new expert contributes
only its loadings and emission variances, and a live correction conditions
on at most two scalar channels whatever the size of the pool. Retaining the
full covariance over $D=d_g+Kd_\alpha$ state coordinates uses
$O(MD^2+M^2)$ storage per candidate
(Appendix~\ref{app:dense-filter}). For large pools,
a structured representation keeps every shared--private covariance and
drops only the dependence between private states given the shared state.
Its maintained covariance storage is then linear in $K$, and every current target--expert pair
law, hence every current routing cost, is preserved at the projection step
(Appendix~\ref{app:red-star}; accuracy, runtime, and storage comparisons in
Appendices~\ref{app:jena-pool-scaling}--\ref{app:jena-computational-scaling}).

\subsection{Internal Prediction and Greedy Deferral}
\label{sec:pr-routing}
Historical fitting produces $L$ candidates, which may differ in regime
count, dimensions, or initialization. Candidate $j$ runs its own filter and
provides a predictive law $q_{t,j}$ of $(Y,H_1,\ldots,H_K)$ given
$\mathcal P_t$. Within a regime, the target and each target--expert pair
are bounded transforms of Gaussian variables, so the moments below are
computed by one- and two-dimensional quadrature and then mixed over
regimes. The candidates are combined with weights $w_{t,j}$ that sum to
one. With $s_\epsilon(v)=(1-\epsilon)v+\epsilon/2$ for $\epsilon\in(0,1)$,
we set
\begin{equation}
\label{eq:main-prediction}
\begin{aligned}
 a_{t,j}&=s_\epsilon\!\left(
               \frac{\mathbb E_{q_{t,j}}Y+B}{2B}\right),\\
 a_t&=\sum_jw_{t,j}a_{t,j},\qquad f_t=2Ba_t-B.
\end{aligned}
\end{equation}
The internal predictor is therefore learned through state inference and
model weighting. Equivalently,
$f_t=(1-\epsilon)\sum_jw_{t,j}\mathbb E_{q_{t,j}}Y$, so smoothing shrinks
the forecast slightly toward the center of the range; it keeps every
normalized moment in $[\epsilon/2,1-\epsilon/2]$, which makes the
logarithmic score below finite. The guarantees use $\epsilon=1/(T+1)$ for a
horizon $T$.

Given this common $f_t$, each candidate sets $H_0=f_t$ and computes
\begin{equation}
\label{eq:main-routing}
\begin{aligned}
 c_{t,j,k}
   &=s_\epsilon\!\left(
       \frac{\mathbb E_{q_{t,j}}(H_k-Y)^2+\beta_k}{C_*}\right),\\
 c_{t,k}&=\sum_jw_{t,j}c_{t,j,k},\qquad
 I_t\in\arg\min_{k\in\mathcal A_t}c_{t,k}.
\end{aligned}
\end{equation}
Ties are broken in a fixed order, and every candidate evaluates the same
deployed internal prediction.

To interpret the rule, suppose an external forecast is uncertain but
tends to vary with the target. Its expected error then depends on this
dependence as well as on its mean:
\[
\begin{aligned}
 \mathbb E(H_k-Y)^2
 &= (\mathbb EH_k-\mathbb EY)^2+\operatorname{Var}(H_k)\\
 &\quad+\operatorname{Var}(Y)-2\operatorname{Cov}(H_k,Y).
\end{aligned}
\]
The joint law supplies the covariance term before the expert is consulted,
from the expert's learned dependence on the target without requiring its
current forecast, so an expert whose forecast tracks the target can
be preferred to one whose mean is closer but whose forecast is independent
of the target under the predictive law.
For the internal action, which is fixed before acting, the cost reduces to
$(f_t-\mathbb EY)^2+\operatorname{Var}(Y)$, and the common affine smoothing
preserves the ordering of costs.

\subsection{Historical Fitting and Predictive Adaptation}
\label{sec:pr-score}
Structural fitting optimizes the regularized joint historical likelihood of the
target and all archived expert channels with a Monte Carlo version of
expectation--maximization \citep{dempster1977maximum}, whose E-step samples
latent trajectories and regime paths (Appendix~\ref{app:pr-estimation}).
Calibration then replays the later segment through each frozen candidate,
initializing its filter and scoring its pre-feedback predictions.
During deployment, states, regime probabilities, and weights adapt;
structural parameters remain fixed.

The weight update evaluates precisely the moments used by the internal
predictor and router. Once $y_t$ is observed, put
$\tilde y_t=(y_t+B)/(2B)$ and $V_t=C_{t,I_t}/C_*$, and use the average
Bernoulli score
\begin{equation}
\label{eq:main-score}
\begin{aligned}
 \ell_{\rm B}(z,q)&=-z\log q-(1-z)\log(1-q),\\
 \ell_{t,j}&=\tfrac12\{
   \ell_{\rm B}(\tilde y_t,a_{t,j})+\ell_{\rm B}(V_t,c_{t,j,I_t})\},\\
 w_{t+1,j}&=\frac{w_{t,j}e^{-\ell_{t,j}}}
                         {\sum_iw_{t,i}e^{-\ell_{t,i}}}.
\end{aligned}
\end{equation}
Each channel's expected score is minimized by its true conditional mean,
so the score is proper for bounded means without binary or independent
channels, and concavity of $e^{-\ell}$ gives
$e^{-\ell_t}\geq\sum_jw_{t,j}e^{-\ell_{t,j}}$ for the score $\ell_t$ of the
aggregate moments, which lets the aggregate compete with the best
candidate. Calibration starts from uniform weights and averages the scores
of the target and every observed action cost, with the internal cost at the
common forecast. Its terminal weights $\nu_j$ initialize deployment with
$w_{1,j}=(1-\zeta)\nu_j+\zeta/L$, where the mixing weight $0<\zeta<1$ keeps
every candidate in play. Without fixed share, the weights favor
candidates with smaller cumulative scores, while changing conditions are tracked within each
candidate by switching and filtering
(Algorithm~\ref{alg:main-slds}; Appendix~\ref{app:historical-deployment}).

\begin{algorithm}[t]
\caption{L2D-SLDS: historical fitting and selective deployment}
\label{alg:main-slds}
\begin{algorithmic}[1]
\REQUIRE Fitting and calibration archives; candidates; $f_{\rm ref}$,
$\vect\psi$; fees, $B$, $T$, $\epsilon$, $\zeta$.
\STATE Fit candidate parameters on the fitting archive and freeze them.
\STATE Replay calibration with full feedback from uniform weights; keep
the terminal filter states and weights $\nu_j$, and set
$w_{1,j}=(1-\zeta)\nu_j+\zeta/L$.
\FOR{each deployment round $t=1,\ldots,T$}
\STATE Observe $\vect x_t$ and $\mathcal E_t$; propagate and mix all
retained states, including unavailable experts.
\STATE Compute and cache $f_t$ and the candidate costs for
$k\in\mathcal A_t$ by \eqref{eq:main-prediction}--\eqref{eq:main-routing}.
\STATE Choose $I_t\in\arg\min_{k\in\mathcal A_t}c_{t,k}$; issue $f_t$ if
$I_t=0$, otherwise acquire and issue $h_{t,I_t}$.
\STATE When $y_t$ arrives, compute $C_{t,I_t}$ and update $w_{t+1,j}$
from the cached moments by \eqref{eq:main-score}.
\STATE Condition each candidate on the target and only the queried
forecast, updating its states and regime probabilities.
\ENDFOR
\end{algorithmic}
\end{algorithm}

\subsection{Information and Regret}
\label{sec:pr-guarantee}
Routing affects both which expert is chosen and how well the system
predicts internally, so we evaluate the two together. For any $\mathcal P_t$-measurable
$\varphi_t\in[-B,B]$, replace the internal cost by
$C^\varphi_{t,0}=(\varphi_t-y_t)^2$, keep the external costs, and let
$k_t^\varphi\in\arg\min_{k\in\mathcal A_t}
\mathbb E[C^\varphi_{t,k}\mid\mathcal P_t]$. Define
\[
 \operatorname{Reg}_T(\varphi)=
 \mathbb E\sum_{t=1}^T(C_{t,I_t}-C^\varphi_{t,k_t^\varphi}).
\]
The oracle may change its action at every date and use the Bayes
prediction $y^\star_t=\mathbb E[y_t\mid\mathcal P_t]$. It holds exactly the
router's information, so the regret measures how well the router uses what
it has observed; the value of a query for later rounds is quantified by
Proposition~\ref{prop:pr-query-value}.

\noindent\textbf{Learning from live observations.}
We first condition on the archive. The moments $a_{t,j}$ and $c_{t,j,k}$
predict the true normalized target mean
$a^\star_t=\mathbb E[\tilde y_t\mid\mathcal P_t]$ and costs
$c^\star_{t,k}=\mathbb E[C_{t,k}/C_*\mid\mathcal P_t]$. With
$\operatorname{kl}(p,q)=p\log\frac pq+(1-p)\log\frac{1-p}{1-q}$, the
Bernoulli divergence of a prediction $q$ from a true mean $p$, define
\[
\begin{aligned}
 \Delta_{t,j}
 &=\operatorname{kl}(a^\star_t,a_{t,j})
   +\operatorname{kl}(c^\star_{t,I_t},c_{t,j,I_t}),\\
 \mathcal B_T
 &=\min_j\left\{2\log(1/w_{1,j})+
                  \mathbb E\sum_{t=1}^T\Delta_{t,j}\right\}.
\end{aligned}
\]
The term $2\log(1/w_{1,j})\leq2\log(L/\zeta)$ is the price of not knowing
the best candidate in advance. With the concavity inequality of
Section~\ref{sec:pr-score} and $\operatorname{kl}(p,q)\geq2(p-q)^2$, the
score update bounds the expected cumulative squared error of the aggregate
target and selected-cost moments by $\mathcal B_T/2$
(Lemma~\ref{lem:pr-budget}).

A greedy choice, however, compares all available costs, while only the
selected one is scored. We therefore assume that errors on the observed
channels control errors in \emph{cost differences}: almost surely, for
every available $k$,
\[
\begin{aligned}
 &\left|(c^0_{t,I_t}-c^0_{t,k})-(c^\star_{t,I_t}-c^\star_{t,k})\right|\\
 &\qquad\leq
 \Lambda\left\|(a_t^0-a^\star_t,c^0_{t,I_t}-c^\star_{t,I_t})\right\|_2+\xi_t,
\end{aligned}
\]
where $a_t^0,c_{t,k}^0$ are the unsmoothed aggregate moments, $\Lambda\geq0$
is deterministic, and $\xi_t\geq0$ is $\mathcal P_t$-measurable. This holds
with $\xi_t=0$ whenever the always-observed target reveals the latent
information that drives the cost differences. For example, if the truth
and the candidates share state-conditional emission moments and differ only
in their probabilities over two latent states whose normalized target means
differ by $\varrho>0$, then $\Lambda=2/\varrho$ for any number of experts and
regime changes (Appendix~\ref{app:greedy-bridge}). Information about cost
differences then reaches the router through the target at every round, so
the greedy rule needs no exploration; information
about private states that the target does not carry enters $\xi_t$.

\begin{restatable}[Greedy prediction and deferral regret]{theorem}{GreedyRegret}
\label{thm:pr-main}
Under the protocol of Section~\ref{sec:pr-problem}, conditionally on the
archive, and under the cost-contrast condition,
\begin{equation}
\label{eq:pr-regret}
\begin{split}
 \operatorname{Reg}_T(\varphi)\leq{}&
 C_*\Lambda\sqrt{T\mathcal B_T/2}
 +C_*\Lambda\epsilon T/\sqrt2\\
 &+C_*\mathbb E\sum_t\xi_t+2B^2\mathcal B_T .
\end{split}
\end{equation}
\end{restatable}

Greedy optimality bounds a wrong decision by its cost-contrast error, and
the target score bounds the additional internal error. The bound is
sublinear whenever $\mathcal B_T$ and $\mathbb E\sum_t\xi_t$ are $o(T)$.
An \emph{exact} candidate has unsmoothed moments equal to $a^\star_t$ and
$c^\star_{t,k}$; Appendix~\ref{app:pr-theorem} gives its constant error budget.

\noindent\textbf{Using full historical supervision.}
The archive gives a second source of information about cost differences,
independent of which experts the router consults. It need not identify one
exact model; it suffices that weight moves away from inaccurate candidates.
Condition on fitting, and fix a nonempty set $\mathcal G$ of accurate
candidates with a reference $j^\circ\in\mathcal G$: at every calibration
and deployment date, the conditional target law and each target--expert
pair law of every $j\in\mathcal G$ are within $\chi^2$ divergence
$\kappa_t$ of the truth. Here $\chi^2(P\Vert Q)=\int(dP/dQ-1)^2\,dQ$ takes
the truth as $P$ and the candidate as $Q$, with $P\ll Q$; it controls every
bounded moment, hence the forecast and every cost.

An inaccurate rival $j$ loses weight when its forecasts differ from the
reference's on the scored channels: the target and every available cost
during calibration, and the target and the selected cost during deployment.
At each scored date $s$, let $\gamma_{s,j}$ be the mean squared difference
between the two candidates' smoothed forecasts on these channels, and
$e_s$ the reference's mean squared error relative to its Bernoulli variance
(\eqref{eq:robust-evidence} in Appendix~\ref{app:archive-robust}). The score contraction
$\mathbb E[\exp\{-(\ell_{s,j}-\ell_{s,j^\circ})/4\}\mid\mathcal P_s]
\leq\exp\{e_s/2-\gamma_{s,j}/4\}$ (Lemma~\ref{lem:robust-quarter-affinity})
shows that separation shrinks the rival's expected fourth-root weight ratio
at each date, while reference error offsets it. Assume the total reference
error over both phases is at most $A$, and that every $j\notin\mathcal G$
accumulates separation at least $N\gamma_{\rm H}^2$ over the $N$
calibration dates and $n\gamma_{\rm L}^2$ over every live prefix of length
$n$, almost surely given fitting, for deterministic
$A,\gamma_{\rm H},\gamma_{\rm L}\geq0$ (the historical or the live
separation may be zero). With uniform initial calibration weights, put
\[
\begin{aligned}
 U&=(L-1)e^{A/2}\Bigl(e^{-N\gamma_{\rm H}^2/4}
       +L\Bigl(\frac{\zeta}{L(1-\zeta)}\Bigr)^{1/4}\Bigr),\\
 \bar b_t&=\min\{1,Ue^{-(t-1)\gamma_{\rm L}^2/4}\}.
\end{aligned}
\]
The two terms of $U$ are the weight that historical evidence leaves on
inaccurate candidates and the weight that $\zeta$ restores at deployment;
live separation makes their sum decay.

\begin{restatable}[Historical and live evidence]{theorem}{HistoricalRegret}
\label{thm:pr-archive-robust}
Under the protocol of Section~\ref{sec:pr-problem} and the preceding
divergence, reference-error, and separation conditions, with deterministic
$\kappa_t$ and exact predictive integration,
\[
\begin{split}
 \operatorname{Reg}_T(y^\star)\leq{}&
 C_*\sum_t\sqrt{\kappa_t}
 +B^2\sum_t(\sqrt{\kappa_t}+\epsilon)^2\\
 &+(2C_*+4B^2)\sum_t\bar b_t .
\end{split}
\]
If these conditions hold jointly with fitting-conditioned probability at
least $1-\delta$, the bound holds with an additional $C_*T\delta$.
\end{restatable}

The first two terms measure predictive approximation, including the
moment-matching error of the switching filter, and the last bounds the
weight on inaccurate candidates. The theorem needs neither an exact
candidate nor the cost-contrast condition (Lemma~\ref{lem:robust-stitched}),
and for the SLDS, $\kappa_t$ is controlled by the errors in regime
probabilities and Gaussian moments (Appendix~\ref{app:archive-robust-slds}).

\begin{corollary}[Sublinear regret]
\label{cor:pr-sublinear}
Let $\epsilon=1/(T+1)$, and fix $L$, $B$, $\beta_{\max}$, and $\Lambda$.
Use the almost-sure hypotheses of the respective regret theorem.
\begin{enumerate}[label=(\roman*),leftmargin=*,nosep]
\item If some candidate is exact and the cost-contrast condition holds
with $\xi_t=0$, then for fixed $0<\zeta<1$,
$\operatorname{Reg}_T(y^\star)=O\bigl(\sqrt{T\log(eL/\zeta)}\bigr)$.
\item If $L\geq2$, every candidate in $\mathcal G$ satisfies
$\kappa_t\leq c/T$ in both phases, and $c$ and $\gamma_{\rm H}>0$ are fixed,
then sufficiently many $N=\Theta(\log(T+1))$ calibration dates and a
sufficiently small mixing weight $\zeta=O((T+1)^{-4})$ give
$\operatorname{Reg}_T(y^\star)=O(\sqrt T)$, even when
$\gamma_{\rm L}=0$.
\end{enumerate}
\end{corollary}

Appendices~\ref{app:pr-theorem} and~\ref{app:archive-robust-rate} give the
constants. In (ii), the archive separates candidates before feedback
becomes selective. Unlike bounds for piecewise-stationary bandits
\citep{besson2022efficient}, neither rate explicitly counts changes of the
best expert; their effect enters through predictive accuracy.

%% file: figures/probabilistic_model.tex
% One generative step of a candidate L2D-SLDS (formerly panel (a)).
\begin{tikzpicture}[
  x=0.96cm,y=0.96cm,
  font=\small,
  latent/.style={circle,draw=black!80,line width=0.6pt,
    minimum size=7.5mm,inner sep=1pt,fill=white},
  given/.style={rectangle,draw=black!45,fill=black!5,
    minimum height=5.5mm,inner sep=3pt},
  model/.style={-{Latex[length=1.7mm]},line width=0.65pt,draw=black!80}
]
\draw[draw=black!40,line width=0.5pt] (4.2,-0.18) rectangle (6.95,2.6);
\node[font=\footnotesize,anchor=north] at (5.58,-0.26)
  {$k=1,\ldots,K$ retained experts};

\node[given] (zp) at (0.55,2.9) {$z_{t-1}$};
\node[given] (gp) at (0.55,-0.45) {$\vect g_{t-1}$};
\node[given] (up) at (6.2,2.05) {$\vect u_{t-1,k}$};
\node[latent] (z) at (2.35,2.05) {$z_t$};
\node[latent] (g) at (2.35,0.4) {$\vect g_t$};
\node[latent] (u) at (4.8,2.05) {$\vect u_{t,k}$};
\node[latent,fill=black!15] (r) at (0.55,1.25) {$r_t$};
\node[latent,dashed,fill=citationdarkblue!5] (d) at (4.8,0.4) {$d_{t,k}$};
\node[font=\footnotesize,anchor=north,align=center] at (0.55,0.69)
  {observed\\after acting};
\node[font=\footnotesize,anchor=north] at (2.35,-0.15)
  {shared state};

\draw[model] (zp) -- (z);
\draw[model] (gp) -- (g);
\draw[model] (up) -- (u);
\draw[model] (z) -- (g);
\draw[model] (z) -- (u);
\draw[model] (z) -- (r);
\draw[model] (z) -- (d);
\draw[model] (g) -- (r);
\draw[model] (g) -- (d);
\draw[model] (u) -- (d);
\end{tikzpicture}

%% file: Section/Evaluation.tex
\section{Empirical Evaluation}
\label{sec:evaluation}
We compare eight bandit routers on three real datasets and study
consultation fees, routing-rule and covariance ablations, synthetic
switching and stationary controls, and expert withdrawal and reentry.
Jena pools of four to 64 experts test prediction cost and computational
scaling, including runtime and storage against all baselines.
Appendix~\ref{app:experiments-greedy} gives each study its own subsection
with protocols, tables, and curves.

\begin{figure*}[t]
\centering
\includegraphics[width=\textwidth]{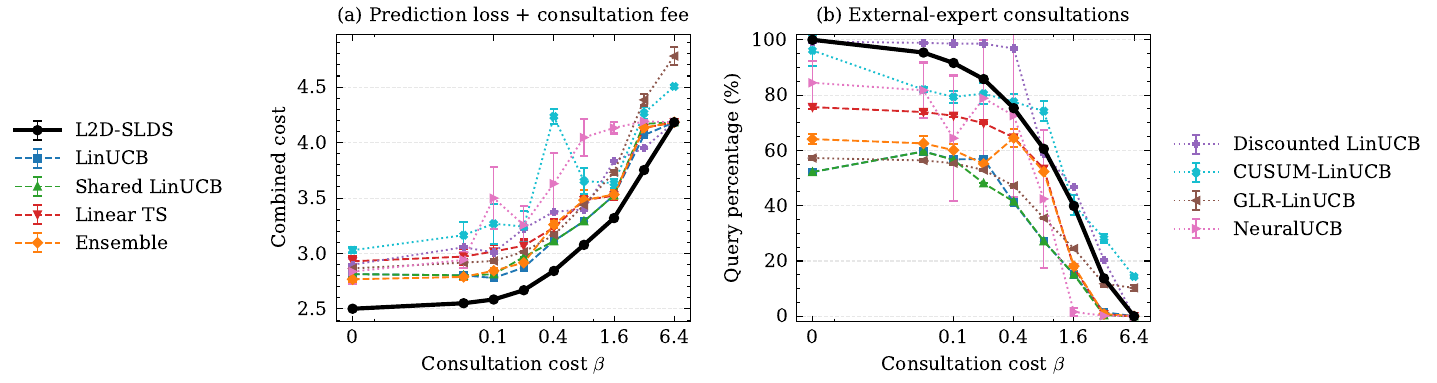}
\caption{Consultation-cost sensitivity on Jena: (a) heldout combined cost
and (b) percentage of rounds consulting an external expert, as the common
fee $\beta$ increases. Baselines are reselected on validation at each fee;
mean $\pm$ one sample SD over three seeds. The fee axis is linear near
zero and logarithmic thereafter.}
\label{fig:jena-consultation-cost}
\end{figure*}

\noindent\textbf{Setting.}
Delhi, Jena, and Melbourne each use four prefix-trained ridge experts,
revealed only when consulted. Cost is bounded-scale squared error plus
the fee ($0.22$ in Delhi, zero otherwise); results are means and sample
SDs over three seeds (Appendices~\ref{app:exp-setting}--\ref{app:exp-metrics}).
Appendix~\ref{app:exp-licenses} gives dataset licenses and reproduction code.

\noindent\textbf{Baselines.}
The eight routers are LinUCB and its shared-parameter variant
\citep{li2010contextual}, linear Thompson sampling
\citep{agrawal2013thompson}, ensemble sampling \citep{lu2017ensemble},
discounted LinUCB \citep{russac2019weighted}, CUSUM- and GLR-based
change-detection LinUCB \citep{liu2018change,besson2022efficient}, and
NeuralUCB \citep{zhou2020neuralcontextualbanditsucbbased}. All use the same
archive, fees, feedback protocol, and internal learner along their own
feedback history, selecting among 16 configurations on chronological
validation (Appendix~\ref{app:greedy-baselines}).

\begin{table}[t]
\caption{Heldout combined cost (mean $\pm$ sample SD, three seeds);
best means in bold. All routers use full-covariance internal learners.}
\label{tab:main-climate-cost}
\centering
\scriptsize
\setlength{\tabcolsep}{1pt}
\begin{tabular}{lrrr}
\toprule
System & Delhi & Jena & Melbourne\\
\midrule
L2D-SLDS & $\mathbf{0.954}\pm0.000$ & $\mathbf{2.502}\pm0.000$ & $\mathbf{2.321}\pm0.000$\\
LinUCB & $0.966\pm0.000$ & $2.814\pm0.000$ & $2.394\pm0.000$\\
Shared LinUCB & $0.966\pm0.000$ & $2.814\pm0.000$ & $2.372\pm0.000$\\
Linear TS & $0.992\pm0.019$ & $2.930\pm0.050$ & $2.360\pm0.023$\\
Ensemble & $1.003\pm0.021$ & $2.768\pm0.021$ & $2.428\pm0.067$\\
Discounted LinUCB & $1.127\pm0.000$ & $2.899\pm0.000$ & $2.382\pm0.000$\\
CUSUM-LinUCB & $0.981\pm0.012$ & $3.030\pm0.027$ & $2.458\pm0.020$\\
GLR-LinUCB & $0.990\pm0.012$ & $2.865\pm0.012$ & $2.473\pm0.021$\\
NeuralUCB (diagonal) & $0.966\pm0.000$ & $2.837\pm0.115$ & $2.390\pm0.151$\\
\bottomrule
\end{tabular}
\end{table}

\noindent\textbf{Comparison with baselines.}
L2D-SLDS has the lowest mean combined cost on all three datasets
(Table~\ref{tab:main-climate-cost}). In Delhi, it consults on eight of 472
rounds, paying a mean fee of $0.004$. In Jena and Melbourne, free access
leads to consultation at every round, so the gain comes from expert
selection; Jena cost is $9.6\%$ below the best baseline ($2.502$ versus $2.768$).

\noindent\textbf{Deciding when to consult.}
On Jena, we vary the common fee from $0$ to $6.4$, reselecting every
baseline on validation at each fee (Appendix~\ref{app:jena-consultation-cost}).
L2D-SLDS lowers consultation from $100\%$ to $0\%$, with $60.6\%$ at
$\beta=0.8$ (Figure~\ref{fig:jena-consultation-cost}). It has the lowest
combined cost at every tested fee below $6.4$ and ties the best baselines
at $6.4$, where it uses only its internal predictor.

\begin{table}[t]
\caption{Ablations, heldout combined cost (mean $\pm$ SD, three seeds).
Top: routing rules on one learner. Middle: structured covariance.
Bottom: withdrawn experts, retained or reset, and the best of LinUCB,
discounted LinUCB, and linear TS.}
\label{tab:main-ablations}
\centering
\scriptsize
\setlength{\tabcolsep}{2pt}
\begin{tabular}{lrrr}
\toprule
Variant & Delhi & Jena & Melbourne\\
\midrule
Greedy (ours) & $\mathbf{0.954}\pm0.000$ & $\mathbf{2.502}\pm0.000$ & $\mathbf{2.321}\pm0.000$\\
Inverse-gap & $1.095\pm0.007$ & $2.766\pm0.022$ & $2.361\pm0.032$\\
Joint IX & $0.981\pm0.003$ & $3.071\pm0.033$ & $2.325\pm0.008$\\
\midrule
Structured & $0.954\pm0.000$ & $2.502\pm0.000$ & $2.325\pm0.000$\\
\midrule
Retain & $\mathbf{0.954}\pm0.000$ & $\mathbf{2.906}\pm0.000$ & $\mathbf{2.321}\pm0.000$\\
Reset & $\mathbf{0.954}\pm0.000$ & $2.915\pm0.000$ & $2.327\pm0.000$\\
Best baseline & $0.966\pm0.000$ & $3.089\pm0.000$ & $2.355\pm0.023$\\
\bottomrule
\end{tabular}
\end{table}

\noindent\textbf{Greedy routing versus exploration.}
With the same fitted learner and update rule, we replace greedy routing
by inverse-gap weighting
\citep{foster2020beyond} and exponential weights with implicit exploration
over routing policies (Joint IX; \citealp{kocak2014efficient}), tuned on
validation (Appendix~\ref{app:greedy-policy-comparison}). Greedy has the
lowest cost on every dataset (Table~\ref{tab:main-ablations}, top).

\noindent\textbf{Regime switching.}
Known-parameter synthetic controls isolate the regime mechanism:
retaining the switching candidate reduces cost by $43\%$ on switching
data, with costs $0.116$ versus $0.114$ for one regime on stationary data
(Appendix~\ref{app:completed-mechanism-controls}).

\noindent\textbf{Expert withdrawal and reentry.}
With three experts withdrawn on successive windows, L2D-SLDS keeps the
lowest cost among the tested routers; retention improves on reset in Jena
and Melbourne, particularly just after return
(Table~\ref{tab:main-ablations}, Appendix~\ref{app:extended-deployed-reentry}).

\noindent\textbf{Covariance, pool size, and computation.}
Structured covariance matches full covariance in Delhi and Jena and stays
within $0.004$ in Melbourne (Table~\ref{tab:main-ablations}, middle).
Across Jena pools of four to 64 experts, their relative mean-cost difference
stays below $0.15\%$; at 64 experts, structured covariance reduces posterior-array
storage by $95.5\%$ (Appendix~\ref{app:jena-pool-scaling}).
System timings against all baselines are similar because the common
forecasting computation dominates; greedy routing requires no additional
policy buffers (Appendix~\ref{app:jena-computational-scaling}).

%% file: Section/Conclusion_and_Impact.tex
\section{Conclusion}
L2D-SLDS learns a joint switching law from historical supervision to
forecast, defer greedily, and update unqueried and absent experts from
selective feedback. Regret is sublinear under the stated predictive-accuracy
and information conditions, without exploration. It outperforms eight
bandit routers on three datasets, with fee, mechanism, and scaling studies
supporting its design.

%% file: Section/AI_Statement.tex
\subsection*{AI Use Statement}
In this work, we used generative AI tools to assist in deriving the proofs of
Theorems~\ref{thm:pr-archive-robust} and~\ref{thm:pr-main}, and
Corollary~\ref{cor:pr-sublinear}, and
to design and write the experimental code. We carefully reviewed every part of
the AI-assisted proofs and code. We also used generative AI tools to polish
the writing. We did not use them for translation or for qualitative or
thematic data analysis, and no synthetic data were generated by AI tools: the
synthetic trajectories are sampled from the specified generative model. We
take responsibility for the final content of this work, including text,
claims, and artifacts produced with the aid of generative AI.

%% file: Section/Checklist.tex
\section*{Checklist}
% Questions are reproduced verbatim from the official AISTATS 2027 pack.
% Answers describe the current draft; authors must confirm submission assets.
\begin{enumerate}
 \item For all models and algorithms presented, check if you include:
 \begin{enumerate}
  \item A clear description of the mathematical setting, assumptions, algorithm, and/or model.
  [Yes] Section~\ref{sec:pr-main} and Appendices~\ref{app:redesign}--\ref{app:predictive-routing}.
  \item An analysis of the properties and complexity (time, space, sample size) of any algorithm.
  [Yes] Appendices~\ref{app:dense-filter}, \ref{app:pr-estimation}, and \ref{app:predictive-routing}.
  \item (Optional) Anonymized source code, with specification of all dependencies, including external libraries.
  [Yes] The supplementary source archive includes dependencies and instructions for all three datasets (Appendix~\ref{app:exp-licenses}).
 \end{enumerate}
 \item For any theoretical claim, check if you include:
 \begin{enumerate}
  \item Statements of the full set of assumptions of all theoretical results.
  [Yes] Assumptions accompany the results in Section~\ref{sec:pr-guarantee} and Appendices~\ref{app:historical-deployment}--\ref{app:predictive-routing}.
  \item Complete proofs of all theoretical results.
  [Yes] Appendices~\ref{app:historical-deployment}--\ref{app:predictive-routing}.
  \item Clear explanations of any assumptions.
  [Yes] Section~\ref{sec:pr-guarantee} explains the information conditions, with structural examples in the appendix.
 \end{enumerate}
 \item For all figures and tables that present empirical results, check if you include:
 \begin{enumerate}
  \item The code, data, and instructions needed to reproduce the main experimental results (either in the supplemental material or as a URL).
  [Yes] The supplementary code includes data-download commands, the experimental pipeline, and table and curve generation; Appendix~\ref{app:exp-licenses} gives the dataset sources.
  \item All the training details (e.g., data splits, hyperparameters, how they were chosen).
  [Yes] Appendix~\ref{app:experiments-greedy}.
  \item A clear definition of the specific measure or statistics and error bars (e.g., with respect to the random seed after running experiments multiple times).
  [Yes] Appendix~\ref{app:experiments-greedy} defines the per-run means and sample standard deviations.
  \item A description of the computing infrastructure used. (e.g., type of GPUs, internal cluster, or cloud provider).
  [No] Hardware is specified for the computational comparison
  (Appendix~\ref{app:jena-computational-scaling}), but not for all studies.
 \end{enumerate}
 \item If you are using existing assets (e.g., code, data, models) or curating/releasing new assets, check if you include:
 \begin{enumerate}
  \item Citations of the creator If your work uses existing assets.
  [Yes] Appendix~\ref{app:exp-licenses} attributes the datasets and links to the releases used.
  \item The license information of the assets, if applicable.
  [Yes] Appendix~\ref{app:exp-licenses} identifies the dataset release licenses.
  \item New assets either in the supplemental material or as a URL, if applicable.
  [Yes] The anonymized source archive is supplied as supplementary material.
  \item Information about consent from data providers/curators.
  [Not Applicable] The reported experiments use existing public time series rather than newly collected participant data.
  \item Discussion of sensible content if applicable, e.g., personally identifiable information or offensive content.
  [Not Applicable] The experiments use public univariate time series.
 \end{enumerate}
 \item If you used crowdsourcing or conducted research with human subjects, check if you include:
 \begin{enumerate}
  \item The full text of instructions given to participants and screenshots.
  [Not Applicable]
  \item Descriptions of potential participant risks, with links to Institutional Review Board (IRB) approvals if applicable.
  [Not Applicable]
  \item The estimated hourly wage paid to participants and the total amount spent on participant compensation.
  [Not Applicable]
 \end{enumerate}
\end{enumerate}

%% file: Section/Appendix.tex
\clearpage
\appendix
\onecolumn

\section{Notation and Reading Guide}
\label{app:notation}
The appendix follows the information flow of L2D-SLDS.
Appendix~\ref{app:historical-deployment} specifies the historical and
live observation protocol. Appendices~\ref{app:redesign} and
\ref{app:pr-estimation} give the joint model, selective filtering,
and structural estimation. Appendix~\ref{app:predictive-routing}
then derives the greedy algorithm and its guarantees.

The regret analysis in Appendix~\ref{app:predictive-routing} follows one
route. It first fixes the protocol, the greedy algorithm, and the
moment-score budget, which bounds the error of the scored target and
selected-cost moments. Two sources of information then connect this
budget to decisions. Live feedback suffices when observed errors control
cost contrasts (Appendix~\ref{app:pr-theorem}).
Historical calibration can instead suppress inaccurate candidates while
allowing several approximate ones (Appendix~\ref{app:archive-robust}),
with the Gaussian-mixture certificate and archive--horizon rate in
Appendices~\ref{app:archive-robust-slds}
and~\ref{app:archive-robust-rate}. The analysis closes with the value
of a query for future forecasts and the scope of the comparator.
Appendix~\ref{app:related-work} provides further literature context.
Appendix~\ref{app:experiments-greedy} gives the datasets, setting,
metrics, and baselines, followed by the baseline comparison, consultation-fee
sweep, routing and regime controls, and expert reentry. It closes with
covariance, expert-pool scaling, and computational comparisons.

\begin{table}[htbp]
\caption{Principal notation. Bold symbols denote vectors or matrices;
scalar targets, forecasts, scores, and costs are unbold.}
\label{tab:notation}
\centering
\begin{tabular}{ll}
\toprule
Symbol & Meaning\\
\midrule
$K,\ M,\ L$ & Number of external experts, regimes, and predictive candidates\\
$\mathcal E_t,\ \mathcal A_t$ & Available external experts; $\{0\}\cup\mathcal E_t$\\
$I_t,\ f_t,\ h_{t,k},\ y_t$ & Selected action, internal forecast, expert forecast, target\\
$B,\ \beta_k,\ C_*$ & Output bound, access fee, maximum action cost\\
$\mathcal P_t$ & Observed pre-action information\\
$\vect g_t,\ \vect u_{t,k},\ z_t$ & Shared state, private state, regime\\
$\mat\Pi,\ \omega_{t,m}$ & Regime transition matrix and filtering probabilities\\
$q_{t,j},\ w_{t,j}$ & Candidate predictive law and aggregation weight\\
$a_{t,j},\ c_{t,j,k}$ & Smoothed normalized target mean and action cost\\
$a^\star_t,\ c^\star_{t,k}$ & True conditional normalized target mean and action cost\\
$\epsilon,\ \zeta$ & Moment smoothing and initial mixing weight\\
$\mathcal R(\pi,\theta)$ & Population risk of the supervised deferral system\\
$\mathcal B_T,\ \operatorname{Reg}_T(\varphi)$ & Predictive-moment budget and expected regret\\
\bottomrule
\end{tabular}
\end{table}

\input{Section/AppendixParts/historical_deployment}
\input{Section/AppendixParts/anchored_filter}
\input{Section/AppendixParts/structured_covariance_control}
\input{Section/AppendixParts/identification_estimation}
\input{Section/AppendixParts/greedy_routing}

\input{Section/AppendixParts/related_work_extended}
\input{Section/AppendixParts/experiments}

%% file: Section/AppendixParts/historical_deployment.tex
\section{Historical Training and Selective Deployment}
\label{app:historical-deployment}

\subsection{Forecast Time, Access Cost, and Target Definition}
\label{app:finance-timing}
\label{app:forecast-timing}
A deployment round is a completed forecast interval, not necessarily a
calendar day. Let $\tau_t$ be the forecast issue time and $\sigma_t>\tau_t$
the time at which its outcome is observable. The analyzed protocol assumes
$\sigma_t\leq\tau_{t+1}$, with feedback processed before the next action.
For example, non-overlapping forecasts of a fixed horizon satisfy this
feedback schedule.

At $\tau_t$, the router knows the timestamped context, expert availability,
the internal forecast, and the predicted cost of each action. It chooses
an expert without seeing that expert's current prediction. A selected
expert returns its forecast before the decision deadline. That forecast
is used as issued: the current action does not blend it into a new
internal forecast or choose again after seeing it. At $\sigma_t$, the
target arrives and the completed observation is assimilated once.
The next forecast decision follows assimilation of the completed observation.

The fee is the price of obtaining a live forecast, whether from an
expensive model, a proprietary service, or a human expert, expressed in the
units of the prediction loss. Consulting an expert changes which forecasts
are observed, not the target or the expert's forecast. Historical full
feedback comes from an available archive, and its acquisition cost is
outside the deployment objective.

\noindent\textbf{Bounded targets and declared transformations.}
When the task does not provide a known range, we use a declared
transformation. For an unbounded scalar response $v_t$ and original-scale expert forecast
$\widehat c^\star_{t,k}$, an alternative bounded forecasting task is
defined explicitly. Fix a training-chosen location $c\in\mathbb R$
and scale $s>0$, and take
\begin{equation}
\label{eq:bounded-target-transform}
 y_t=B\tanh((v_t-c)/s),\qquad
 h_{t,k}=B\tanh((\widehat c^\star_{t,k}-c)/s),\qquad B>0.
\end{equation}
The target and issued expert forecasts then lie in $(-B,B)$.
All methods use the same transform, fitted scale, and fee convention, and
the objective is squared error on the transformed scale.

\subsection{Leakage-Free Historical Fitting}
\label{app:historical-fitting}
Split the archive chronologically into a fitting segment
$\mathcal D_{\rm fit}$ and a later calibration segment
$\mathcal D_{\rm cal}$, followed by deployment. All targets used in either
historical segment must have matured before deployment begins. Forecasts
in both segments must be historical issued predictions or rolling-origin
out-of-sample predictions, not retrospective predictions from an expert
trained on their future targets. Out-of-sample forecasting evaluation is
the relevant convention \citep{tashman2000outofsample}; no i.i.d. split
of timestamped rows is assumed.

Fit the reference predictor, feature transformations, target scale, and
candidate structural parameters using only the fitting segment. Freeze
them before calibration. A calibration loss must be computed from a
forecast made before assimilating its outcome, even though every expert's
outcome is subsequently available. Hyperparameter choices must not use
deployment costs. When a forecast window overlaps a split boundary,
exclude its unresolved label from the earlier segment.

For every historical round $s$, let $\mathcal V_s$ be the set of external
forecasts actually archived for that round. Full training feedback means
$\mathcal V_s=\mathcal E_s$; it does not create records for experts that
were unavailable. Fit the joint observed emission once:
\[
 q_\Theta\bigl(y_s,(h_{s,k})_{k\in\mathcal V_s}
                 \mid\mathcal H_{s-1},x_s\bigr).
\]
In the anchored model this is the likelihood of one inverse-link target
residual $r_s$ and the archived disagreements
$(d_{s,k})_{k\in\mathcal V_s}$, including the parameter-independent
Jacobian from Appendix~\ref{app:pr-observed-likelihood}.
The conditional Gaussian filter can assimilate these channels
sequentially without duplicating the target. During deployment,
$\mathcal V_t=\varnothing$ for an internal action and
$\mathcal V_t=\{I_t\}$ otherwise. The observation equations and learned
parameters do not change when the mask changes.

Replay the calibration archive chronologically through each frozen
candidate to initialize its deployment state. Use the final filtered
state, not an arbitrary stationary reset; propagate through any gap
before the first live round. Exact-realizability statements require the
correct conditional initial state law given the archive. Approximate
replay remains covered by the predictive-error budget.

\subsection{Initialize Model Weights from Historical Moment Scores}
\label{app:archive-weights}
Historical forecasts can also choose the initial aggregation weights.
One implementable default scores all archived experts using the same
bounded moment quantities as deployment. At each calibration round,
compute the common internal prediction and all pre-feedback forecasts.
After the outcome is observed, use
\[
 \ell^{\rm cal}_{s,j}=
 \frac{\ell_{\rm B}(\tilde y_s,a_{s,j})+
       \sum_{k\in\{0\}\cup\mathcal V_s}
       \ell_{\rm B}(C_{s,k}/C_*,c_{s,j,k})}
      {2+|\mathcal V_s|}.
\]
This uses the target channel and every observable action-cost channel,
including the internal cost of the same deployed calibration prediction.
Initialize positive weights and update them by exponential weights with
this score, without fixed share. The resulting weights $\nu_j$ are
prequential score weights, not necessarily Bayesian posterior masses.
The weighted-geometric-mean argument of Lemma~\ref{lem:pr-budget} also
applies to this average of any finite number of channels; no independence
of their errors is needed. Deployment guarantees additionally use the
live information conditions of Appendix~\ref{app:pr-theorem} or the
historical separation conditions of Appendix~\ref{app:archive-robust}.

For a declared $0<\zeta<1$, initialize deployment by
\begin{equation}
\label{eq:archive-floor}
 w_{1,j}=(1-\zeta)\nu_j+\zeta/L.
\end{equation}
This keeps every candidate in play, with $w_{1,j}\geq\zeta/L$, while
historically accurate candidates receive more mass; $\zeta$ is fixed
before deployment.
Filters then update on the target and selected expert only; structural
parameters stay frozen in the default protocol.

\begin{algorithm}[t]
\caption{Historical fitting and calibration, then selective deployment}
\label{alg:historical-live}
\begin{algorithmic}[1]
\REQUIRE Chronological $\mathcal D_{\rm fit}$ and later $\mathcal D_{\rm cal}$,
with targets matured before deployment; $L$ candidate specifications
(regime counts, state dimensions, features, and estimation settings);
target-scale rule and reference specification; fees;
$0<\zeta<1$, $0<\epsilon<1$, horizon $T$.
\ENSURE Frozen candidates, terminal calibration filter states, and normalized live weights $w_1$.
\STATE Verify forecast timestamps, label maturity, and archived expert
coverage $\mathcal V_s\subseteq\mathcal E_s$.
\STATE Fit the reference and transforms on $\mathcal D_{\rm fit}$ only
(or use already frozen earlier fits); fix $B$ and $C_*$.
\STATE Fit each SLDS candidate by Algorithm~\ref{alg:pr-structural-fitting};
freeze structural parameters and retain its state after replaying
$\mathcal D_{\rm fit}$ from its declared initial law.
\STATE Initialize calibration log weights to $-\log L$ for the uniform prior.
\FOR{each calibration round $s$, in chronological order}
\STATE Observe $\vect x_s$ and availability; propagate each candidate
to round $s$, including all retained expert states.
\STATE Using the current calibration weights, compute $a_{s,j}$ and
the common $f_s$ by \eqref{eq:pr-target}; compute $c_{s,j,k}$ for
$k\in\{0\}\cup\mathcal V_s$ by \eqref{eq:pr-costs}. Cache these moments.
\STATE Reveal $y_s$ and archived $h_{s,k}$ for $k\in\mathcal V_s$.
Compute $\ell^{\rm cal}_{s,j}$ from the cached moments using
Appendix~\ref{app:archive-weights}.
\STATE Subtract $\ell^{\rm cal}_{s,j}$ from each log weight and normalize
with log-sum-exp, without fixed share.
\STATE Assimilate the observed target and expert channels once using
each candidate's observation model; for an SLDS, these are $r_s$ and
the archived $d_{s,k}$, $k\in\mathcal V_s$.
\ENDFOR
\STATE Let $\nu_j$ be the terminal normalized calibration weights;
set $w_{1,j}=(1-\zeta)\nu_j+\zeta/L$ by \eqref{eq:archive-floor}.
\STATE Run Algorithm~\ref{alg:pr-complete} from these states and weights;
its first propagation includes any gap since the final calibration round.
\end{algorithmic}
\end{algorithm}

The resulting weights and terminal filter states enter the deployment
algorithm without a reset. Theorem~\ref{thm:pr-main} gives the live
moment-budget guarantee, and Theorem~\ref{thm:pr-archive-robust}
quantifies how full-feedback calibration suppresses inaccurate
candidates. Their proofs are collected in
Appendix~\ref{app:predictive-routing}.

%% file: Section/AppendixParts/anchored_filter.tex
\section{Joint Predictive Model and Selective Filtering}
\label{app:redesign}

Historical fitting and live deployment use the same target--expert
observation law. The change between them is the observation mask:
all archived experts can inform historical fitting, while a live
update uses the target and selected expert. We first specify the
anchored law and full-covariance recursion used in the main paper,
then analyze structured representations of the same Gaussian moments.
We use scalar target and expert channels throughout, while the shared
and private latent states can have different vector dimensions.

\subsection{Anchor the Model to a Fixed Reference}
\label{app:red-anchor}

Fix a measurable reference predictor \(f_{\rm ref}:\mathcal X\to\mathbb R\)
before deployment, in inverse-link coordinates. It can be zero or fitted
on an earlier training segment, but is not updated during the analyzed run.
Use $b(v)=B\tanh(v)$ and $\iota(y)=\operatorname{arctanh}(y/B)$.
This emission model requires observations in $(-B,B)$. Define
\[
 r_t=f_{\rm ref}(\vect{x}_t)-\iota(y_t),\qquad
 d_{t,k}=\iota(h_{t,k})-f_{\rm ref}(\vect{x}_t),\qquad
 \delta_{t,0}=\iota(f_t)-f_{\rm ref}(\vect{x}_t).
\]
Here \(d_{t,k}\) is the inverse-link expert disagreement. For
$\widetilde e_{t,k}=\iota(h_{t,k})-\iota(y_t)$ and $h_{t,0}=f_t$, the identities
\begin{equation}
\label{eq:red-anchor-identities}
 \widetilde e_{t,0}=\delta_{t,0}+r_t,\qquad
 \widetilde e_{t,k}=d_{t,k}+r_t
\end{equation}
hold without statistical assumptions. The offset \(\delta_{t,0}\) is known
before acting. The target reveals \(r_t\) every round, and querying expert
\(k\) reveals \(d_{t,k}\). Thus the observation protocol is unchanged.
The internal smoothing ensures $f_t\in(-B,B)$, so its inverse exists.
Evaluated residuals are $e_{t,k}=h_{t,k}-y_t$ and must not be
confused with $\widetilde e_{t,k}$.

Retain a switching shared state \(\vect{g}_t\in\mathbb R^{d_g}\) and private
states \(\vect{u}_{t,k}\in\mathbb R^{d_\alpha}\) for external experts only.
Fix a predictable feature vector
\(\vect{\psi}_t\in\mathbb R^{d_\alpha}\).
With \(\mat{H}_t^u=\vect{\psi}_t^\top\),
\(\mat{H}_{t,r}=\mat{H}_t^u\mat{\Gamma}_{\rm ref}\), and \(\mat{H}_{t,k}=\mat{H}_t^u\mat{\Gamma}_k\), use the inverse-link model
\begin{equation}
\label{eq:red-anchor-emissions}
 r_t=\mat{H}_{t,r}\vect{g}_t+\varepsilon_{t,r},\qquad
 d_{t,k}=\mat{H}_{t,k}\vect{g}_t+\mat{H}_t^u \vect{u}_{t,k}+\varepsilon_{t,k}.
\end{equation}
Given the states and regime \(z_t=m\), noises are independent, centered
Gaussian with positive variances \(R_{m,r}\) and \(R_{m,k}\).
All \(\mat{H}\) matrices have one row. Use the switching dynamics stated in Section~\ref{sec:pr-model}; the
shared and private transition matrices are square with dimensions
\(d_g\) and \(d_\alpha\), respectively.

The decomposition assigns the known motion of the internal forecast to
\(\delta_{t,0}\), leaving the latent process to describe target variation
and expert disagreement. Thus the predictor can adapt without requiring
a separate internal private state to explain its own changes.
If an expert reacts to past actions, its disagreement law can condition
on that observed history.

\noindent\textbf{Shared target innovation.}
Under \eqref{eq:red-anchor-emissions}, conditional on the pre-action
information, current regime, and current latent states,
\[
 \operatorname{Cov}(\widetilde e_{t,0},\widetilde e_{t,k})=R_{m,r},\qquad
 \operatorname{Cov}(\widetilde e_{t,j},\widetilde e_{t,k})=R_{m,r}
 \quad(j\ne k,\ j,k\ne0).
\]
Each external inverse-link residual has conditional variance \(R_{m,r}+R_{m,k}\).
These terms arise because all such residuals subtract the same inverse-link target.
Modeling \(r_t\) and \(d_{t,k}\) separately preserves these identities
when constructing expert residuals. After the nonlinear link,
the corresponding cross-covariance of evaluated residuals is the
conditional covariance of $Y$, not $R_{m,r}$: conditioned on states and
regime, the target and different expert emissions are independent,
and expanding $\operatorname{Cov}(H_j-Y,H_k-Y)$ leaves
$\operatorname{Var}(Y)$ for $j\ne k$. The same expansion holds with
the fixed internal prediction in place of $H_j$.

\subsection{Full-Covariance Selective Filtering}
\label{app:dense-filter}
Stack the continuous states as
$\vect s_t=(\vect g_t^\top,\vect u_{t,1}^\top,\ldots,
\vect u_{t,K}^\top)^\top\in\mathbb R^D$, with
$D=d_g+Kd_\alpha$. Write $\mat F_m$ for the block-diagonal
transition in mode $m$ and $\mat Q_m$ for its process covariance.
At the end of round $t-1$, a candidate retains mode weights
$\omega_{t-1,\ell}$ and full Gaussian moments
$(\vect\mu_{t-1,\ell},\mat P_{t-1,\ell})$.
These mode weights are distinct from the model-aggregation weights
$w_{t,j}$.
The regime count is fixed within each candidate; different candidates
may use different counts. Here $M$ denotes the count of the candidate
being filtered.

For a row-stochastic regime matrix $\mat{\Pi}$, the predictive mode
weight and conditional predecessor weights are
\[
 \bar\omega_{t,m}=\sum_\ell\omega_{t-1,\ell}\mat{\Pi}_{\ell m},
 \qquad
 \rho_{\ell m}=
 \frac{\omega_{t-1,\ell}\mat{\Pi}_{\ell m}}{\bar\omega_{t,m}}.
\]
Only modes with $\bar\omega_{t,m}>0$ require conditional state
parameters. For each such mode, moment matching gives
\begin{align*}
 \vect\mu_m^{\rm mix}&=\sum_\ell\rho_{\ell m}\vect\mu_{t-1,\ell},\\
 \mat P_m^{\rm mix}&=\sum_\ell\rho_{\ell m}
 \left[\mat P_{t-1,\ell}
 +(\vect\mu_{t-1,\ell}-\vect\mu_m^{\rm mix})
  (\vect\mu_{t-1,\ell}-\vect\mu_m^{\rm mix})^\top\right],\\
 \vect\mu_{t,m}^-&=\mat F_m\vect\mu_m^{\rm mix},\qquad
 \mat P_{t,m}^-=\mat F_m\mat P_m^{\rm mix}\mat F_m^\top+\mat Q_m .
\end{align*}
This replacement of a regime-history mixture by a Gaussian is the
IMM approximation. It preserves all within-mode covariance blocks,
including private--private dependence.

Let $\vect o_t^{\rm inv}$ contain $r_t$ and the observed expert
disagreements, and let $\mat H_{t,m}^{\rm obs}$ contain the
corresponding rows of the joint observation matrix. Conditional
independence of emission noises gives the matching block-diagonal
matrix $\mat R_{t,m}^{\rm obs}$. Suppressing $t,m$ in this update,
set
\[
 \vect e=\vect o^{\rm inv}-\mat H^{\rm obs}\vect\mu^-,
 \qquad
 \mat S=\mat H^{\rm obs}\mat P^-(\mat H^{\rm obs})^\top
                                      +\mat R^{\rm obs}.
\]
The Gaussian conditioning equations are
\begin{equation}
\label{eq:dense-selective-update}
\begin{aligned}
 \mat K&=\mat P^-(\mat H^{\rm obs})^\top\mat S^{-1},\\
 \vect\mu^+&=\vect\mu^-+\mat K\vect e,\\
 \mat P^+&=(\mat I-\mat K\mat H^{\rm obs})\mat P^-
             (\mat I-\mat K\mat H^{\rm obs})^\top
                 +\mat K\mat R^{\rm obs}\mat K^\top .
\end{aligned}
\end{equation}
The last line is the Joseph form of the covariance update, equal
in exact arithmetic to
$\mat P^--\mat K\mat S\mat K^\top$.
Positive observation noise makes $\mat S$ positive definite;
inverse notation denotes a linear solve. Each unqueried state receives
the correction in the corresponding rows of $\mat K$.
The likelihood
$\mathcal L_{t,m}=\mathcal N(\vect o_t^{\rm inv};
\mat H_{t,m}^{\rm obs}\vect\mu_{t,m}^-,\mat S_{t,m})$
then updates
$\omega_{t,m}\propto\bar\omega_{t,m}\mathcal L_{t,m}$.
The bounded-link Jacobian is common to modes and cancels in this
normalization.

\noindent\textbf{Information passed to an unqueried expert.}
\CrossExpertConditioning*
\begin{proof}
After conditioning on the target within a Gaussian component, the
pair $(d_\ell,d_k)$ is Gaussian. Its Schur complement gives
\eqref{eq:main-expert-transfer}. In particular,
$\sigma_{\ell k}^2/\sigma_{kk}\geq0$, with strict reduction when
$\sigma_{\ell k}\ne0$. More generally, let $\mat P$ be the component
covariance after target assimilation and $\vect h^\top$ the row of
$d_k$, so that $\sigma_{kk}=\vect h^\top\mat P\vect h+R_k>0$.
Sequential scalar conditioning in \eqref{eq:dense-selective-update}
gives $\mat P^+=\mat P-\mat P\vect h\vect h^\top\mat P/\sigma_{kk}$.
For a fixed vector $\vect a$, the scalar $\vect a^\top\vect s$ has
covariance $\vect a^\top\mat P\vect h$ with $d_k$. Hence its variance
decreases by $(\vect a^\top\mat P\vect h)^2/\sigma_{kk}$, and its mean
moves by $(\vect a^\top\mat P\vect h/\sigma_{kk})(d_k-\mu_k)$. The
vector $\vect a$ may select a coordinate of any retained expert's
private state, available or not.
\end{proof}
Because all covariances in \eqref{eq:main-expert-transfer} are
post-target, this information is additional to the target. The regime
likelihood transfers information a second way, through the mixture
weights. A mixture's variance can increase on a particular observation;
the exact-law risk reduction holds in conditional expectation
(Proposition~\ref{prop:pr-query-value}).

\begin{proposition}[Retained registry under changing availability]
\label{prop:retained-registry}
Adopt the convention of Section~\ref{sec:pr-model} that context and
availability are exogenous given the observed past, and use the stated
Gaussian initial continuous-state law independent of the initial regime
and subsequent innovations. Fix structural parameters and a regime
history $z_{1:t}$. (i) Filtering the full stacked state with every
unavailable or unqueried channel omitted from the mask gives its exact
conditional law given the observed history and $z_{1:t}$, including
the private states of unavailable experts. (ii) Marginalizing one
expert's private block leaves the conditional law of the remaining
blocks unchanged at that operation. (iii) Replacing a retained block
by a declared prior independent of the remaining blocks reproduces
the joint conditional law if and only if that block's conditional
marginal equals the prior and is independent of the remaining blocks.
\end{proposition}
\begin{proof}
For a fixed regime history, the transition and observed inverse-link
emission factors are Gaussian in the continuous trajectory. An omitted
channel's emission density integrates to one, so the observed-data
posterior conditions only on observed rows, which
\eqref{eq:dense-selective-update} does exactly; block-diagonal
propagation applies to every block, observed or not. Induction over
the observed rounds proves (i). Marginalizing a Gaussian block does
not change the law of the others. The restarted joint law is the
product of the declared prior and the retained conditional marginal,
which equals the original law exactly under the conditions in (iii).
\end{proof}
IMM uses these masked Gaussian operations within each moment-matched
component; its collapse of regime histories remains approximate.

\noindent\textbf{Computation.}
Storage is $O(MD^2+M^2)$ per candidate. Generic dense propagation
costs $O(MD^3)$ and mode mixing costs $O(M^2D^2)$.
For $p_t$ observed scalar channels, correction costs
$O(M(D^2p_t+Dp_t^2+p_t^3))$; sequential scalar conditioning has
the same joint posterior and likelihood. Live feedback has
$p_t\leq2$, whereas historical full feedback has at most $K+1$.
Thus the live update uses a small observation matrix while retaining
the dependence learned from the full archive.

%% file: Section/AppendixParts/structured_covariance_control.tex
\subsection{Structured-Covariance Representation}
\label{app:red-star}

For comparison with full covariance, consider a representation that
retains every shared--private marginal with linear expert storage.
Let $\mathcal K_t\subseteq\{1,\ldots,K\}$ be the retained external
registry, including temporarily unavailable experts whose states are kept.
Suppress \(t,m\) temporarily and write
\begin{equation}
\label{eq:red-star-law}
 q(\vect{g},(\vect{u}_k)_k)=
 \mathcal N(\vect{g};\vect{\mu}_g,\mat{V})
 \prod_k\mathcal N\bigl(\vect{u}_k;\vect{\mu}_k+\mat{L}_k(\vect{g}-\vect{\mu}_g),\mat{D}_k\bigr),
 \qquad \mat{V}\succ0,\quad \mat{D}_k\succ0,
\end{equation}
where \(\mat{L}_k\in\mathbb R^{d_\alpha\times d_g}\).
This is a conditional-independence model: private states are independent
given \(\vect{g}\), not marginally independent of \(\vect{g}\). It stores
\[
 \operatorname{Cov}(\vect{u}_k,\vect{g})=\mat{L}_k\mat{V},\quad
 \operatorname{Var}(\vect{u}_k)=\mat{D}_k+\mat{L}_k\mat{V}\mat{L}_k^\top,\quad
 \operatorname{Cov}(\vect{u}_j,\vect{u}_k)=\mat{L}_j\mat{V}\mat{L}_k^\top\quad(j\ne k).
\]
Memory per mode is
\(O(d_g^2+|\mathcal K_t|(d_\alpha^2+d_\alpha d_g))\),
linear in the number of retained experts.
The online initial law must also be specified. A supplied dense Gaussian
prior is projected by preserving its $(\vect{g},\vect{u}_k)$ marginals, using
\eqref{eq:red-star-project} without a time update. This is exact only
when the prior already has the conditional factorization; otherwise
initialization contributes another approximation. Offline smoothing
uses the declared dense prior. The linear
memory statement concerns the maintained filter state, not the dense
offline smoother or storage of an arbitrary dense initial covariance.

\noindent\textbf{Exact measurement updates within the conditional family.}

Consider one observed channel
\[
 v=\mat{H}_g \vect{g}+\mat{H}_u \vect{u}_j+\varepsilon,\qquad
 \varepsilon\sim\mathcal N(0,R),\quad R>0.
\]
Define
\[
 \mat{A}=\mat{H}_g+\mat{H}_u\mat{L}_j,\quad N=\mat{H}_u\mat{D}_j\mat{H}_u^\top+R,\quad
 \bar v=\mat{H}_g\vect{\mu}_g+\mat{H}_u\vect{\mu}_j,\quad
 \Omega=\mat{A}\mat{V}\mat{A}^\top+N,
\]
and \(b=v-\bar v\), \(\mat{K}_g=\mat{V}\mat{A}^\top\Omega^{-1}\),
\(\mat{G}_j=\mat{D}_j\mat{H}_u^\top N^{-1}\). Inverse notation denotes a
positive-definite solve. The update is
\begin{equation}
\label{eq:red-star-correct}
\begin{aligned}
 \vect{\mu}_g^+&=\vect{\mu}_g+\mat{K}_gb,&
 \mat{V}^+&=\mat{V}-\mat{K}_g\Omega \mat{K}_g^\top,\\
 \mat{L}_j^+&=\mat{L}_j-\mat{G}_j\mat{A},&
 \mat{D}_j^+&=\mat{D}_j-\mat{G}_jN\mat{G}_j^\top,\\
 \vect{\mu}_j^+&=\vect{\mu}_j+\mat{G}_jb+\mat{L}_j^+(\vect{\mu}_g^+-\vect{\mu}_g).
\end{aligned}
\end{equation}
For \(k\ne j\),
\[
 \mat{L}_k^+=\mat{L}_k,\qquad \mat{D}_k^+=\mat{D}_k,\qquad
 \vect{\mu}_k^+=\vect{\mu}_k+\mat{L}_k(\vect{\mu}_g^+-\vect{\mu}_g).
\]
For the always-observed target channel \(r_t\), use
\(\mat{A}=\mat{H}_{t,r}\), \(N=R_{m,r}\), \(\bar v=\mat{H}_{t,r}\vect{\mu}_g\), and update
every private mean by the last formula; no internal private state is needed.
The per-mode inverse-link likelihood is \(\mathcal N(v;\bar v,\Omega)\).
For a bounded observed channel $o$, multiply by
$[B(1-(o/B)^2)]^{-1}$ to obtain its density in observation
coordinates. This Jacobian is independent of the mode and latent state,
so it cancels from posterior conditioning and normalized mode weights.

These formulas follow by first integrating the queried private
conditional to update the shared Gaussian, then conditioning the queried
private state given the shared state. The remaining private conditionals
are unchanged. Sequential assimilation of the independent channel
likelihoods gives the same joint posterior within this family.

\noindent\textbf{Prediction and mode mixing.}

A time update generally leaves the conditional family. Its projection must
therefore be explicit. Suppose \(\mat{T}_k=\operatorname{Cov}(\vect{u}_k,\vect{g})\) and
\(\mat{U}_k=\operatorname{Var}(\vect{u}_k)\) are available before propagation.
For destination-mode matrices \(\mat{A}_g,\mat{A}_u\) and independent process noises,
\[
 \mat{V}^-=\mat{A}_g\mat{V}\mat{A}_g^\top+\mat{Q}_g,\quad
 \mat{T}_k^-=\mat{A}_u\mat{T}_k\mat{A}_g^\top,\quad
 \mat{U}_k^-=\mat{A}_u\mat{U}_k\mat{A}_u^\top+\mat{Q}_u.
\]
Means propagate by \(\mat{A}_g,\mat{A}_u\). Project by
\begin{equation}
\label{eq:red-star-project}
 \mat{L}_k^-=\mat{T}_k^-(\mat{V}^-)^{-1},\qquad
 \mat{D}_k^-=\mat{U}_k^--\mat{L}_k^-\mat{V}^-(\mat{L}_k^-)^\top.
\end{equation}
This preserves every \((\vect{g},\vect{u}_k)\) marginal of the Gaussian predictive law.
It discards only private--private dependence conditional on \(\vect{g}\).

Preserving these pair marginals preserves the current target--expert
emission laws, hence current routing costs under the bounded link.
Subsequent conditioning uses the projected private--private covariance,
so the preservation statement concerns the current pair laws and costs.

Let \(\omega_{t-1,\ell}\) denote posterior regime probabilities within
one candidate model, distinct from aggregation weights \(w_{t,j}\).
For a row-stochastic transition matrix \(\mat{\Pi}\), compute
\[
 \bar\omega_{t,m}=\sum_\ell\omega_{t-1,\ell}\mat{\Pi}_{\ell m},
 \qquad
 \rho_{\ell m}=\frac{\omega_{t-1,\ell}\mat{\Pi}_{\ell m}}{\bar\omega_{t,m}}
 \quad\text{when }\bar\omega_{t,m}>0.
\]
Modes with zero predictive mass are skipped; their conditional state
parameters are immaterial until a declared positive prior or transition
gives them mass. For the positive-mass IMM mixtures, compute joint pair
moments, not merely separate marginal moments:
\[
 \mat{T}_{k,\mathrm{mix}}^{(m)}
 =\sum_\ell\rho_{\ell m}
 \left[\mat{L}_k^{(\ell)}\mat{V}^{(\ell)}
 +(\vect{\mu}_k^{(\ell)}-\bar{\vect{\mu}}_k^{(m)})
  (\vect{\mu}_g^{(\ell)}-\bar{\vect{\mu}}_g^{(m)})^\top\right].
\]
The analogous covariance formulas include their between-mode mean
outer products for \(\mat{V}_{\rm mix}\) and \(\mat{U}_{k,\rm mix}\).
Apply the destination-mode propagation and
\eqref{eq:red-star-project}. Mixture collapse remains approximate.
After assimilating the target and all observed expert channels, update
\[
 \omega_{t,m}
 =\frac{\bar\omega_{t,m}\mathcal L_{t,m}}
        {\sum_a\bar\omega_{t,a}\mathcal L_{t,a}}.
\]
Here \(\mathcal L_{t,m}\) is the product of sequential innovation
densities within mode \(m\), each evaluated after conditioning on the
previously assimilated channels of that round. Multiplying their
unconditional marginal densities instead would incorrectly assume
independent observations. Gaussian positive noise covariances give
positive innovation densities; evaluate their products in log space.
A new expert can enter with \(\mat{L}_k=0\), its declared private prior, and the
current shared belief. Removing a private conditional factor integrates it
to one, so pruning leaves the retained joint law unchanged at that operation.

\noindent\textbf{Target--Expert Moments.}

For a fixed mode, put \(\mat{A}_k=\mat{H}_{t,k}+\mat{H}_t^u\mat{L}_k\).
The nonlinear link is applied separately to the target and forecast,
so the router needs their joint law, not only the residual moments.
Within each mode, the pair $(r_t,d_{t,k})$ is jointly Gaussian, with
\[
\begin{aligned}
 \mathbb E r_t&=\mat{H}_{t,r}\vect{\mu}_g,&
 \mathbb E d_{t,k}&=\mat{H}_{t,k}\vect{\mu}_g+\mat{H}_t^u\vect{\mu}_k,\\
 \operatorname{Var}(r_t)&=\mat{H}_{t,r}\mat{V}\mat{H}_{t,r}^\top+R_{m,r},&
 \operatorname{Cov}(r_t,d_{t,k})&=\mat{H}_{t,r}\mat{V}\mat{A}_k^\top,\\
 \operatorname{Var}(d_{t,k})&=
 \mat{A}_k\mat{V}\mat{A}_k^\top+\mat{H}_t^u\mat{D}_k(\mat{H}_t^u)^\top+R_{m,k}.
\end{aligned}
\]
These pair moments determine the Gaussian law that is mapped through
$b$. For scalar outputs, the router integrates
$[b(f_{\rm ref}(\vect{x}_t)+d_{t,k})-b(f_{\rm ref}(\vect{x}_t)-r_t)]^2$
under this law, as specified in Appendix~\ref{app:pr-protocol}.

\noindent\textbf{Online computation.}
The structured recursion operates on shared and shared--private blocks,
so it never needs to assemble a dense covariance across all experts.
Write $K_t=|\mathcal K_t|$ and $d=d_g+d_\alpha$ for this calculation.
For one candidate, dense-block operation counts are
$O(M^2(1+K_t)d^3)$ for mixing predecessor regimes and
$O(M(1+K_t)d^3)$ for propagation and projection. These bounds include
the positive-definite solves in \eqref{eq:red-star-project}; each operation
is repeated over individual expert blocks, not pairs of experts.
A scalar correction updates the shared block, at most one private block,
and the $K_t$ private means. With the covariance-stabilizing matrix
products used in the implementation, its cost across regimes is
$O(M(d_g^3+d_\alpha^3+K_td_\alpha d_g))$.
There are at most two such corrections per deployment round.

Forming the target and available target--expert Gaussian pairs costs
$O(M(1+|\mathcal E_t|)d^2)$. If each one- or two-dimensional integral
uses at most $J$ scalar integrand evaluations, bounded moments add
$O(M(1+|\mathcal E_t|)J)$ operations. Thus for fixed $M,d_g,d_\alpha$
and quadrature budget $J$, total deployment work is linear in the retained
expert count. With adaptive quadrature, $J$ denotes the actual maximum
evaluation count for that round.
Historical full-feedback filtering instead performs $1+|\mathcal V_t|$
scalar corrections per date. These counts describe the structured
recursion; offline trajectory smoothing uses the dense calculation in
Appendix~\ref{app:pr-gibbs-estimation}.

%% file: Section/AppendixParts/identification_estimation.tex
\section{Identification and Estimation Across the Two Feedback Phases}
\label{app:pr-estimation}

L2D-SLDS learns three distinct objects. State filtering updates the
conditional law of the latent states, model aggregation updates the weights
$w_{t,j}$ from cached moments, and structural estimation fits dynamics,
loadings, and noise before deployment, after which they are frozen.

For the scalar anchored SLDS, fix a predictable feature vector
$\vect{\psi}_t\in\mathbb R^{d_\alpha}$ and use
\[
 \mat{H}_t^u=\vect{\psi}_t^\top,\qquad \mat{H}_{t,r}=\vect{\psi}_t^\top \mat{\Gamma}_{\rm ref},\qquad
 \mat{H}_{t,k}=\vect{\psi}_t^\top \mat{\Gamma}_k,
 \quad \mat{\Gamma}_{\rm ref},\mat{\Gamma}_k\in\mathbb R^{d_\alpha\times d_g}.
\]
The feature map is fixed before evaluating a prediction; it must not use
the current target or any unqueried prediction. Each dictionary member
specifies the feature map, latent dimensions, initial distribution,
regime transition law, process matrices, emission loadings, and noise
covariances. For the stated smoother, declare a Gaussian initial
continuous state $\vect{s}_1\sim\mathcal N(\vect{m}_0,\mat{P}_0)$, $\mat{P}_0\succ0$, independent
of the initial regime and subsequent innovations, and declare noise
floors in advance. If an expert is absent, propagate its retained private
state without an emission update. At first entry use a declared prior.
Pruning and restarting a previously observed expert is an approximation
to its previous posterior, even though integrating its factor out leaves
the currently retained marginals unchanged.

\subsection{Observed-Data Likelihood and the Selection Mechanism}
\label{app:pr-observed-likelihood}
Let $\mathcal V_t$ denote the observed external identities and write
$\vect{o}_t=(y_t,(h_{t,k})_{k\in\mathcal V_t})$. In historical full feedback
$\mathcal V_t=\mathcal E_t$; in deployment it is empty or the selected singleton.
Let $\mathcal H_{t-1}$ contain the logged observed history. With the
observed context and availability included in $\vect{x}_t$, a fixed-parameter
candidate has the sequential likelihood
\begin{equation}
\label{eq:est-likelihood}
 \mathcal L_n(\Theta)=\sum_{t=1}^n
  \log q_\Theta(\vect{o}_t\mid\mathcal H_{t-1},\vect{x}_t,I_t).
\end{equation}
This conditions on the context according to the specified conditional
state-space model. If context or availability has an additional
state-dependent likelihood, it must be modeled as described in
the joint observation model before conditioning.

For a fixed logged policy, the joint sequential likelihood also contains
$\prod_t p_t(I_t\mid\mathcal H_{t-1},\vect{x}_t)$. The policy is a fixed
function of the observed history and algorithm state, independent of
the unknown environment parameter. These factors therefore cancel when
comparing $\Theta$ values. They are not recomputed under each hypothetical
environment parameter. This is why no inverse-propensity weighting is
needed for the correctly specified observed-data likelihood.

In the anchored SLDS the observed emissions are $r_t$ and the observed
$(d_{t,k})_{k\in\mathcal V_t}$. With the fixed bounded link, the absolute
Jacobian determinant of this transformation is
\[
 J(\vect{o}_t)=\frac{1}{B[1-(y_t/B)^2]}
 \prod_{k\in\mathcal V_t}\frac{1}{B[1-(h_{t,k}/B)^2]},
\]
where equal-valued observations still contribute separate factors.
This factor is independent of structural parameters when
the link, scale and reference are fixed. Include $\log J(\vect{o}_t)$ to report
the bounded-data log density; it cancels in parameter comparisons and
conditional state updates. The internal residual is a deterministic
function of $r_t$ and the issued $f_t$; adding it as an independent emission would count
the same target twice. Unqueried emission factors integrate to one.
Neither an imputed cost nor a model-generated expert prediction enters
the observed-data log likelihood as if it had been observed.

\subsection{What Can Be Identified}
\label{app:pr-identification}
Identification concerns equality of observable laws, not numerical
uniqueness of a fitted objective. Even with unlimited data, the shared
coordinates of the anchored SLDS are generally not unique. For any
invertible $\mat{T}_g\in\mathbb R^{d_g\times d_g}$, transform
\[
 \vect{g}'_t=\mat{T}_g\vect{g}_t,\quad
 \mat{A}_m^{(g)\prime}=\mat{T}_g\mat{A}_m^{(g)}\mat{T}_g^{-1},\quad
 \mat{Q}_m^{(g)\prime}=\mat{T}_g\mat{Q}_m^{(g)}\mat{T}_g^\top,\quad
 \mat{\Gamma}'_k=\mat{\Gamma}_k\mat{T}_g^{-1},\quad \mat{\Gamma}'_{\rm ref}=\mat{\Gamma}_{\rm ref}\mat{T}_g^{-1},
\]
and push forward the complete initial continuous-state law by
$\mat{T}=\operatorname{diag}(\mat{T}_g,\mat{I}_{d_\alpha},\ldots,\mat{I}_{d_\alpha})$
acting on $\vect{s}_1=(\vect{g}_1^\top,\vect{u}_{1,1}^\top,\ldots,\vect{u}_{1,K}^\top)^\top$.
If $\vect{s}_1\sim\mathcal N(\vect{m}_0,\mat{P}_0)$, the transformed model starts
from $\mathcal N(\mat{T}\vect{m}_0,\mat{T}\mat{P}_0\mat{T}^\top)$.
This includes the shared--private cross-covariances, which become
$\mat{T}_g\operatorname{Cov}(\vect{g}_1,\vect{u}_{1,k})$.
Private dynamics, private process covariances, emission variances, and
the regime law are unchanged. The transformed process covariance remains
block diagonal, so these parameters again define an anchored SLDS.
Couple the initial state and innovations through the same coordinate map.
Then $\vect{s}'_t=\mat{T}\vect{s}_t$ is preserved by the transformed dynamics,
and each emission and its joint observable distribution is unchanged since
$\mat{\Gamma}'_k\vect{g}'_t=\mat{\Gamma}_k\vect{g}_t$ for every channel. Regime labels can also be permuted
together with their transition, dynamics, and emission parameters.
If a contraction constraint in a fixed Euclidean norm is imposed, only
transformations preserving that constraint remain admissible; orthogonal
transformations still demonstrate coordinate nonuniqueness. Fixing
loadings can select a coordinate convention; the predictive quantities
used for routing are invariant under the admissible transformations.

The regret theorem needs the conditional target mean and each
target--expert squared-cost mean. These quantities are invariant under
the latent-coordinate transformations above.

\noindent\textbf{Historical coverage and live identifiability.}
Historical fitting uses each expert's actual archived coverage and
an informative feature design; deployment transfers information through
the specified shared structure. Theorem~\ref{thm:pr-archive-robust}
quantifies the required historical model separation directly in the
scored moments. A never-observed new expert enters with supplied emission
parameters and a private-state prior, using justified shared structure
or prior information to connect it to the learned model.

\subsection{Anchored SLDS Parameter Updates}
\label{app:pr-slds-estimation}
For the continuous-state model, define a finite fitting window and hold
its Gaussian initial continuous-state law fixed while optimizing. If that law is a
parameterized stationary distribution, its log-density also enters
the objective. An exact E-step uses joint smoothed moments of
$(z_t,\vect{g}_{t-1},\vect{g}_t,\vect{u}_{t-1,k},\vect{u}_{t,k})$ where required.
These are computed by offline smoothing, separately from online filtering.
The trajectory sampler below estimates them by Monte Carlo, with each
parameter update checked against its frozen smoothing objective;
exact E-steps additionally support an observed-likelihood monotonicity check.

Write $\langle\cdot\rangle$ for the chosen smoothing expectation, and
hold these statistics fixed throughout each conditional maximization.
For shared dynamics in mode $m$, form
\[
 \begin{aligned}
 N_m&=\sum_{t=2}^n\langle\mathbf1\{z_t=m\}\rangle,\\
 \mat{S}_{10,m}&=\sum_{t=2}^n\langle\mathbf1\{z_t=m\}\vect{g}_t\vect{g}_{t-1}^\top\rangle,\\
 \mat{S}_{00,m}&=\sum_{t=2}^n
       \langle\mathbf1\{z_t=m\}\vect{g}_{t-1}\vect{g}_{t-1}^\top\rangle.
 \end{aligned}
\]
With nonsingular $\mat{S}_{00,m}$ and no matrix constraint,
$\mat{A}_m^{(g)+}=\mat{S}_{10,m}\mat{S}_{00,m}^{-1}$. The residual second moment gives
$\mat{Q}_m^{(g)+}=N_m^{-1}\sum_{t=2}^n
\langle\mathbf1\{z_t=m\}(\vect{g}_t-\mat{A}_m^{(g)+}\vect{g}_{t-1})
(\vect{g}_t-\mat{A}_m^{(g)+}\vect{g}_{t-1})^\top\rangle$ when $N_m>0$.
The derivative of the quadratic objective is
$2\mat{Q}_m^{-1}(\mat{A}_m\mat{S}_{00,m}-\mat{S}_{10,m})$, which verifies the update and
its orientation. For shared private dynamics, sum the corresponding
statistics over the retained external private chains, excluding any
internal private chain: the anchored model has none.
Missing emission times do not remove a chain's transition terms.
With $K$ private chains present throughout the fitting window, their
pooled process covariance divides the summed residual second moments
by $K N_m$, rather than $N_m$.

For the regime probabilities, let $n_{ij}=\sum_{t=2}^n
\langle\mathbf1\{z_{t-1}=i,z_t=j\}\rangle$ and
$a_j=\langle\mathbf1\{z_1=j\}\rangle$, with initial regime probability
$p_{1,j}=\Pr(z_1=j)$. Independent symmetric
Dirichlet$(1+\tau)$ priors give the conditional MAP updates
\[
 \pi_{ij}^{+}=\frac{n_{ij}+\tau}{\sum_{h=1}^M n_{ih}+M\tau},
 \qquad p_{1,j}^{+}=\frac{a_j+\tau}{1+M\tau},\qquad \tau>0.
\]

For loadings, stack matrix columns in $\vect{b}_k=\operatorname{vec}(\mat{\Gamma}_k)$.
The common design row is $\mat{X}_t=\vect{g}_t^\top\otimes\vect{\psi}_t^\top$
of dimension $1\times(d_gd_\alpha)$. Use the scalar response
$\widetilde d_{t,k}=d_{t,k}-\vect{\psi}_t^\top \vect{u}_{t,k}$ for an external
expert on rounds $\mathcal T_k^{\rm obs}=\{t:k\in\mathcal V_t\}$, and
$\widetilde d_{t,\rm ref}=r_t$ on every
round for the reference loading. At fixed positive emission variances,
the minimizer of expected negative log likelihood plus
$\lambda\|\vect{b}_k\|^2/2$, $\lambda\geq0$, satisfies
\[
 \left(\sum_{t\in\mathcal T_k^{\rm obs}}\sum_m
  \left\langle\mathbf1\{z_t=m\}\frac{\mat{X}_t^\top \mat{X}_t}{R_{m,k}}\right\rangle
       +\lambda \mat{I}\right)\vect{b}_k^{+}
 =\sum_{t\in\mathcal T_k^{\rm obs}}\sum_m
 \left\langle\mathbf1\{z_t=m\}
       \frac{\mat{X}_t^\top\widetilde d_{t,k}}{R_{m,k}}\right\rangle.
\]
For $k={\rm ref}$ use $R_{m,r}$ and all target observations.
The expectation of $\mat{X}_t^\top\widetilde d_{t,k}$ includes shared--private
cross moments. Replacing it by a product of means is a different update.
The Hessian is positive semidefinite; it is positive definite when
$\lambda>0$ or the unregularized expected design is full rank.

If common-norm stable dynamics are desired, solve the constrained
quadratic problem of Appendix~\ref{app:red-matrices} with $\mat{Q}_m$ fixed.
For a covariance update with $N_m>0$ and residual moment $\mat{S}_m\succeq0$,
one well-defined noise-floor update solves, in precision $\mat{\Omega}$,
\[
 \min_{q_{\max}^{-1}\mat I\preceq\mat\Omega\preceq q_{\min}^{-1}\mat I}
 \frac{N_m}{2}\{-\log\det\mat{\Omega}+\operatorname{tr}(\mat{S}_m\mat{\Omega})\},
 \qquad \mat{Q}_m^{+}=\mat{\Omega}^{-1},
 \quad 0<q_{\min}\leq q_{\max}<\infty.
\]
Here $\mat{S}_m$ is the residual second moment divided by $N_m$.
The precision objective is strictly convex and the feasible set is
compact, so the minimizer exists and is unique. This is a blockwise
convex update at fixed dynamics. Zero-occupancy
blocks retain their previous parameters. Scalar emission variances use
the same objective with dimension one.

When parameters change, the old filtered state is not automatically the
posterior under the new parameters. Rerun the filter on the fitting
window from its declared initial law before forecasting the next round,
or explicitly retain this discrepancy in the approximation budget.
No prediction already issued is rescored using the refitted parameters.
If a latent coordinate transformation is used for reporting or
normalization, transform the stored means, covariances, and conditional
loadings at the same time. No automatic coordinate normalization is
required by the default algorithm.

\noindent\textbf{Scalar initialization used in the experiments.}
For $d_g=d_\alpha=1$ and $\psi_t=1$, initialize from the fitting
residuals $r_t,d_{t,k}$, before calibration. The initial regime law is
uniform and $\mat\Pi=0.85\mat I_M+0.15\vect1\vect1^\top/M$.
Let $v_r=\max\{\widehat{\operatorname{Var}}(r),0.02\}$ and, for each
expert with at least two observed forecasts, set
\[
 \bar\Gamma_k=
 \frac{\widehat{\operatorname{Cov}}(r,d_k)}
      {\max\{\widehat{\operatorname{Var}}(r),0.02\}},\qquad
 v_k=\max\{\widehat{\operatorname{Var}}(d_k-\bar\Gamma_k r),0.02\}.
\]
The moments in this display use that expert's observed fitting rows,
with empirical covariance and variance normalized by their row count.
Use $\bar\Gamma_k=1,v_k=0.1$ if fewer than two rows are observed.
Set $\Gamma_{\rm ref}=1$ and $\Gamma_k=\bar\Gamma_k+\delta_k$, where
the seeded perturbations are independent $\delta_k\sim\mathcal N(0,0.03^2)$;
the variance $v_k$ above is computed before this perturbation.
For mode $m$, let $b_m=(m-1)/(M-1)$ if $M>1$ and $b_1=0$ otherwise.
The scalar transition and noise initializations are
\[
\begin{aligned}
 A_m^{(g)}&=0.3+0.4b_m,& A_m^{(u)}&=0.3,\\
 Q_m^{(g)}&=v_r(0.2+0.4b_m),& Q_m^{(u)}&=0.02,\\
 R_{m,r}&=(0.7+0.4b_m)v_r/2,& R_{m,k}&=(0.7+0.4b_m)v_k.
\end{aligned}
\]
Clip every process and emission variance to $[10^{-4},10]$, and use
$\vect m_0=\vect0$, $\mat P_0=\max\{v_r,0.1\}\mat I_{K+1}$.
The fitting constraints and penalties are $\rho=0.98$, $\lambda=0.01$,
and $\tau=0.1$. For the one- and two-regime benchmark candidates,
initialization seeds are 102 and 103, and the first E-step seeds are
201 and 301; the latter increment by one at each iteration.
This specifies the starting model without using calibration or deployment outcomes.

\subsection{A Trajectory Sampler for the SLDS E-Step}
\label{app:pr-gibbs-estimation}
An implementable offline E-step must specify how to obtain the joint
moments above. Stack the continuous states as
$\vect{s}_t=(\vect{g}_t^\top,\vect{u}_{t,1}^\top,\ldots,\vect{u}_{t,K}^\top)^\top$, with block-diagonal
transition $\mat{F}_m=\operatorname{diag}(\mat{A}_m^{(g)},\mat{A}_m^{(u)},\ldots,\mat{A}_m^{(u)})$
and process covariance $\mat{Q}_m$. For a prescribed regime path $z_{1:n}$,
the full continuous-state posterior is Gaussian. A dense Kalman forward
pass followed by backward smoothing gives its means $\vect{\mu}_t$, covariances
$\mat{P}_t$, and lag covariances. With filtered covariance $\mat{P}_t^f$ and
next-step predictive covariance $\mat{P}_{t+1}^-$, define
\[
 \mat{J}_t=\mat{P}_t^f\mat{F}_{z_{t+1}}^\top(\mat{P}_{t+1}^-)^{-1}.
\]
Gaussian conditioning gives
\[
 \begin{aligned}
 \vect{\mu}_t&=\vect{\mu}_t^f+\mat{J}_t(\vect{\mu}_{t+1}-\vect{\mu}_{t+1}^-),\\
 \mat{P}_t&=\mat{P}_t^f+\mat{J}_t(\mat{P}_{t+1}-\mat{P}_{t+1}^-)\mat{J}_t^\top,\\
 \operatorname{Cov}(\vect{s}_{t+1},\vect{s}_t\mid z_{1:n},\vect{o}_{1:n})&=\mat{P}_{t+1}\mat{J}_t^\top.
 \end{aligned}
\]
The lag orientation matches $\mat{S}_{10,m}$ above. To sample a trajectory,
first draw $\vect s_n\sim\mathcal N(\vect\mu_n^f,\mat P_n^f)$.
Then, for $t=n-1,\ldots,1$, draw from the Gaussian conditional
\[
 \vect s_t\mid\vect s_{t+1},z_{1:n},\vect o_{1:n}
 \sim\mathcal N\!\left(
 \vect\mu_t^f+\mat J_t(\vect s_{t+1}-\vect\mu_{t+1}^-),\,
 \mat P_t^f-\mat J_t\mat P_{t+1}^-\mat J_t^\top\right).
\]
The conditional uses filtered, rather than smoothed, time-$t$ moments:
given $\vect s_{t+1}$ and the regime path, later observations add no
information about $\vect s_t$. Applying Gaussian conditioning to the
filtered joint law of $(\vect s_t,\vect s_{t+1})$, whose cross-covariance
is $\mat P_t^f\mat F_{z_{t+1}}^\top$, gives this mean and covariance.
The backward factorization of the smoothing law therefore gives an exact
continuous trajectory conditional on the regime path.

Conversely, conditional on a sampled continuous trajectory, the regime
path has a finite-state chain factorization. Its time-$t$ emission factor
is the product of the observed inverse-link target/expert Gaussian densities and,
for $t\geq2$, the transition density
$\mathcal N(\vect{s}_t;\mat{F}_m \vect{s}_{t-1},\mat{Q}_m)$. The destination regime $m=z_t$
indexes this density. The fixed initial continuous law is independent
of $z_1$ in this implementation. A regime-dependent initial law would
add its density at $t=1$. Forward filtering followed by discrete backward
sampling therefore draws the entire regime trajectory from its exact
conditional law.

Alternating these two draws defines a Gibbs kernel. Writing $\pi$ for
the joint working posterior, its regime marginal is invariant because
\[
 \sum_z\int\pi(z\mid\vect{o})\pi(\vect{s}\mid z,\vect{o})\pi(z'\mid\vect{s},\vect{o})\,d\vect{s}
 =\int\pi(\vect{s}\mid\vect{o})\pi(z'\mid\vect{s},\vect{o})\,d\vect{s}=\pi(z'\mid\vect{o}).
\]
For each retained regime draw, use the conditional Gaussian moments
instead of products of sampled state coordinates. This integrates out
continuous-state Monte Carlo variation at that conditional step.
The retained regime draws form a correlated Markov chain; finite
burn-in and sample budgets give a Monte Carlo E-step. Each parameter
update is checked for increase of its frozen conditional objective.
For small problems, enumerating all regime paths gives an independent
exact E-step and additionally permits an observed-likelihood monotonicity check.

The dense smoother costs $O(nD^3)$ per pass in the generic implementation,
where $D=d_g+Kd_\alpha$, and stores $O(nD^2)$ Gaussian moments. This
offline computation stores a complete trajectory, whereas the
full-covariance deployment filter stores only current per-regime moments,
using $O(MD^2+M^2)$ memory. The structured alternative in
Appendix~\ref{app:red-star} reduces covariance storage as the expert
pool grows. Particle Gibbs with ancestor sampling provides
another established trajectory-inference route
\citep{lindsten2014particle}; like the sampler above, it is an offline
smoother rather than an online filter.

\begin{algorithm}[t]
\caption{Historical structural estimation with trajectory moments}
\label{alg:pr-structural-fitting}
\begin{algorithmic}[1]
\REQUIRE Fitting archive; observed-channel masks; fixed features/reference;
initial parameters and latent law; iteration and smoothing budgets;
$0<q_{\min}\leq q_{\max}$, $0<\rho<1$, $\lambda>0$, $\tau>0$.
\STATE Initialize positive regime probabilities, noise eigenvalues in
$[q_{\min},q_{\max}]$, and dynamics with operator norm at most $\rho$.
Use the scalar recipe above for the experimental configuration.
\STATE Work on a copy of the estimator. Construct each proposed parameter
update separately until all iteration checks pass.
\FOR{each fitting iteration}
\STATE Hold current parameters fixed. Obtain joint smoothing moments by
exact inference where tractable, or the declared trajectory sampler.
\STATE For each regime $m$, accumulate occupancies, transition counts, and
the mode-weighted moments $\langle\mathbf1\{z_t=m\}\vect{s}_t\rangle$ and
$\langle\mathbf1\{z_t=m\}\vect{s}_t\vect{s}_t^\top\rangle$ for the emission
updates; for $t\geq2$, also accumulate
$\langle\mathbf1\{z_t=m\}\vect{s}_t\vect{s}_t^\top\rangle$,
$\langle\mathbf1\{z_t=m\}\vect{s}_{t-1}\vect{s}_{t-1}^\top\rangle$, and
$\langle\mathbf1\{z_t=m\}\vect{s}_t\vect{s}_{t-1}^\top\rangle$ for the dynamics
and process-covariance updates.
\STATE Update initial and transition regime probabilities from the expected
counts with the Dirichlet$(1+\tau)$ prior.
\STATE At fixed process covariances, solve each noise-weighted dynamics
quadratic over $\|\mat{A}_m\|_{\rm op}\leq\rho$ by projected steps, recording the
remaining convex optimality gap and whether its tolerance was reached.
\STATE Update process covariances from residual second moments with the
declared eigenvalue bounds; pool private-chain statistics when their
dynamics are shared.
\STATE Solve the loading normal equations with ridge $\lambda$, including
the shared--private cross moments and only observed emission times.
\STATE Update emission variances from expected squared residuals with the
declared bounds; retain unobserved blocks unchanged.
\STATE Check that the frozen smoothing objective plus prior has not
decreased. Restore the previous parameters if this numerical check fails.
\STATE For an exact E-step, additionally check finite penalized observed
likelihoods and their nondecrease; for a Monte Carlo E-step, use the
frozen-objective acceptance check above.
\STATE Accept the proposed parameters in the working copy after successful checks;
on an update or checking failure, retain the previous parameters and
report the failure.
\ENDFOR
\STATE Replay the fitting archive using the fitted parameters from the
declared initial law and the filter that will be used at deployment.
\STATE Commit the fitted parameters and replayed state together only if
the entire fitting and replay procedure succeeds.
\STATE Return fitted parameters, final filtering state, and recorded optimization
and inference diagnostics.
\end{algorithmic}
\end{algorithm}

\subsection{Constrained Dynamics Updates}
\label{app:red-matrices}
For common-norm contraction, constrain each shared and private transition
matrix to $\|\mat A_m\|_{\rm op}\leq\rho<1$ in the same coordinate
system. This controls arbitrary switching, whereas a separate spectral-radius
condition on each mode does not. With fixed $\mat Q_m\succ0$, minimize
the conditional quadratic below over this operator-norm ball.
For
$f(\mat{A})=\tfrac12\operatorname{tr}
\{\mat{Q}^{-1}(\mat{A}\mat{S}_{00}\mat{A}^\top-\mat{A}\mat{S}_{10}^\top-\mat{S}_{10}\mat{A}^\top)\}$,
\[
 \nabla f(\mat{A})=\mat{Q}^{-1}(\mat{A}\mat{S}_{00}-\mat{S}_{10}),\qquad
 L_f=\|\mat{Q}^{-1}\|_{\rm op}\|\mat{S}_{00}\|_{\rm op}
\]
is a valid Frobenius-norm gradient Lipschitz constant. If $L_f>0$,
projected gradient steps of size $1/L_f$ solve the actual constrained
quadratic iteratively. The Euclidean projection onto the operator-norm
ball clips the singular values to $\rho$ and is applied to each gradient
iterate.
For a feasible starting point, the projection variational inequality
and smoothness give
$f(\mat{A}^+)\leq f(\mat{A})-\tfrac{L_f}{2}\|\mat{A}^+-\mat{A}\|_F^2$.
When $L_f=0$ for valid posterior second moments, there is no informative
transition design and the previous feasible matrix is retained.
A finite iteration budget produces a generalized conditional
maximization step whose objective increase is checked explicitly.

For a residual covariance $\mat{S}\succeq0$, the bounded precision update
above has the eigenvectors of $\mat{S}$ and covariance eigenvalues
$\min\{q_{\max},\max\{q_{\min},\lambda_i(\mat{S})\}\}$.
To verify this, diagonalize $\mat{S}$, use the spectral bounds and Hadamard's
determinant inequality to remove off-diagonal precision entries without
increasing the objective, and minimize each scalar function
$-\log\omega+\lambda_i(\mat{S})\omega$ on
$[q_{\max}^{-1},q_{\min}^{-1}]$.

%% file: Section/AppendixParts/greedy_routing.tex
\section{Greedy Predictive Routing: Algorithm and Guarantees}
\label{app:predictive-routing}
Condition on the fitting archive, including the fitted dictionary. Unless
an expectation explicitly includes calibration, also condition on that
archive and its terminal filtering states. The proposed router is the
deterministic minimizer of the aggregate predicted cost, and every
guarantee below compares its conditional risk with the same-history
oracle of Section~\ref{sec:pr-guarantee}.

The appendix follows the order in which information enters the argument.
Appendices~\ref{app:pr-protocol}--\ref{app:pr-likelihood} define the
cached moments, state the algorithm, and prove that the score update
controls the error of the moments it observes: the target and the
selected cost. This budget says nothing yet about unselected actions.
Appendix~\ref{app:pr-theorem} closes that gap with live feedback,
assuming observed errors control cost contrasts, and proves
Theorem~\ref{thm:pr-main}. Appendix~\ref{app:archive-robust}
closes it instead with historical calibration, which suppresses the
weight of inaccurate candidates without requiring an exact one, and
proves Theorem~\ref{thm:pr-archive-robust}.
The final subsection explains how queried feedback can improve future
internal forecasts and specifies the scope of the comparator.

\subsection{Protocol and Decision-Relevant Moments}
\label{app:pr-protocol}
Let $B>0$, $y_t,h_{t,k}\in[-B,B]$ almost surely,
$\beta_0=0$, $0\leq\beta_k\leq\beta_{\max}$, and
$C_*=4B^2+\beta_{\max}$. Let $\mathcal P_t$ be the
pre-action information, including observed history, current context,
availability, and computed forecasts. The action is
$\mathcal P_t$-measurable. It reveals, but does not change, the current
potential expert forecast or target. Feedback is processed before the
next decision. All conditional expectations below use this observed
history, not the past values of unqueried experts.

Each of $L\geq1$ causal candidates $q_{t,j}$ predicts a joint law of
bounded $Y,H_k$. In the anchored SLDS, this is the pushforward
\[
 Y=B\tanh(f_{\rm ref}(x_t)-r),\qquad
 H_k=B\tanh(f_{\rm ref}(x_t)+d_k).
\]
Gaussian conditioning in inverse-link coordinates is exact for a
conditional-Gaussian input; propagation and mode mixing are projected.
One- and two-dimensional quadrature integrates the bounded moments
within each mode, including target--expert dependence; integration
error is accounted for in Appendix~\ref{app:archive-robust-slds}.

For $0<\epsilon<1$, set $s_\epsilon(v)=(1-\epsilon)v+\epsilon/2$.
Compute the target moments before evaluating any cost:
\begin{equation}
\label{eq:pr-target}
 a^0_{t,j}=\frac{\mathbb E_{q_{t,j}}Y+B}{2B},\quad
 a_{t,j}=s_\epsilon(a^0_{t,j}),\quad
 a_t=\sum_jw_{t,j}a_{t,j},\quad f_t=2Ba_t-B.
\end{equation}
Put $H_0=h_{t,0}=f_t$. Every candidate evaluates this same $f_t$:
\begin{equation}
\label{eq:pr-costs}
 \begin{split}
 C_{t,k}&=(h_{t,k}-y_t)^2+\beta_k,\qquad
 c^0_{t,j,k}=\frac{\mathbb E_{q_{t,j}}(H_k-Y)^2+\beta_k}{C_*},\\
 c_{t,j,k}&=s_\epsilon(c^0_{t,j,k}),\qquad
 c_{t,k}=\sum_jw_{t,j}c_{t,j,k}.
 \end{split}
\end{equation}
Write $a_t^0=\sum_jw_{t,j}a_{t,j}^0$ and
$c^0_{t,k}=\sum_jw_{t,j}c^0_{t,j,k}$. All unsmoothed moments lie in
$[0,1]$. An external squared cost equals
$\operatorname{Var}(H_k)+\operatorname{Var}(Y)
-2\operatorname{Cov}(H_k,Y)+(\mathbb E H_k-\mathbb E Y)^2+\beta_k$;
the internal cost is $\operatorname{Var}(Y)+(f_t-\mathbb E Y)^2$.
Independent marginal forecasts generally do not determine external costs.

\subsection{The Complete Greedy Update}
\label{app:pr-algorithm}
Use a fixed, predictable tie-breaking order and select
\begin{equation}
\label{eq:pr-routing}
 I_t\in\arg\min_{k\in\mathcal A_t}c_{t,k}
       =\arg\min_{k\in\mathcal A_t}c^0_{t,k}.
\end{equation}
The equality follows because smoothing is the same strictly increasing
affine map for every action. No exploration parameter is required.
If only action zero is available, it is selected.
After the outcome arrives, form the two observed channels and score
cached pre-action forecasts:
\begin{equation}
\label{eq:pr-logloss}
 \begin{split}
 \tilde y_t&=(y_t+B)/(2B),\qquad V_t=C_{t,I_t}/C_*,\\
 \ell_{t,j}&=\tfrac12\{\ell_{\rm B}(\tilde y_t,a_{t,j})+
                      \ell_{\rm B}(V_t,c_{t,j,I_t})\},\qquad
 \ell_{\rm B}(z,q)=-z\log q-(1-z)\log(1-q).
 \end{split}
\end{equation}
This is a proper score for two conditional means, not a joint density
likelihood or a Bayesian posterior over structural parameters. Its
virtual binary outcomes are integrated out, not sampled; no independence
of target and cost is assumed. The update is
\begin{equation}
\label{eq:pr-weight-update}
 w_{t+1,j}=
 \frac{w_{t,j}e^{-\ell_{t,j}}}{\sum_iw_{t,i}e^{-\ell_{t,i}}}.
\end{equation}
The proposed update uses no fixed share. For one candidate the weight
remains one. Because scores use cached forecasts, scoring and filtering
may be ordered either way.

\begin{algorithm}[t]
\caption{Greedy L2D-SLDS deployment (no fixed share)}
\label{alg:pr-complete}
\begin{algorithmic}[1]
\REQUIRE $L$ frozen candidates and terminal calibration filter states;
positive normalized $w_1$ from Algorithm~\ref{alg:historical-live};
reference $f_{\rm ref}$; $B>0$, fees $\beta_k$, $C_*=4B^2+\beta_{\max}$;
$0<\epsilon<1$; horizon $T$.
\ENSURE Actions $I_t$, issued forecasts $h_{t,I_t}$, and updated filter states and weights.
\STATE Store $\log w_{1,j}$ for every $j$.
\FOR{$t=1,\ldots,T$}
\STATE Observe $\vect x_t$ and $\mathcal E_t$; set $\mathcal A_t=\{0\}\cup\mathcal E_t$.
\STATE Propagate each candidate to round $t$. For an SLDS, use the IMM
mixing and prediction in Appendix~\ref{app:dense-filter} on the full
retained state, including unavailable experts; do not reset returning experts.
\STATE Compute $a_{t,j}$ for each candidate by \eqref{eq:pr-target};
set $a_t=\sum_jw_{t,j}a_{t,j}$ and the common forecast $f_t=2Ba_t-B$.
\STATE Set $H_0=f_t$. For every $j$ and $k\in\mathcal A_t$, compute
$c_{t,j,k}$ by \eqref{eq:pr-costs}; set $c_{t,k}=\sum_jw_{t,j}c_{t,j,k}$.
\STATE Cache all $a_{t,j}$ and $c_{t,j,k}$, together with $f_t$.
\STATE Set $I_t$ to the first minimizer of $c_{t,k}$ in the order $0,1,\ldots,K$ restricted to $\mathcal A_t$.
\IF{$I_t=0$}
\STATE Set $h_{t,0}=f_t$ and issue $f_t$.
\ELSE
\STATE Pay $\beta_{I_t}$, obtain $h_{t,I_t}$ before issuing the prediction,
and issue that forecast.
\ENDIF
\STATE When $y_t$ matures, compute $C_{t,I_t}=(h_{t,I_t}-y_t)^2+\beta_{I_t}$,
$\tilde y_t=(y_t+B)/(2B)$, and $V_t=C_{t,I_t}/C_*$.
\STATE Compute $\ell_{t,j}$ by \eqref{eq:pr-logloss} using the cached
$a_{t,j}$ and $c_{t,j,I_t}$, for every $j$.
\STATE Store $\log w_{t+1,j}=\log w_{t,j}-\ell_{t,j}
-\operatorname{LSE}_{i}(\log w_{t,i}-\ell_{t,i})$ for every $j$.
\STATE Form $r_t=f_{\rm ref}(\vect x_t)-\iota(y_t)$ and, only if
$I_t\ne0$, $d_{t,I_t}=\iota(h_{t,I_t})-f_{\rm ref}(\vect x_t)$.
\STATE For each SLDS candidate, build the observation mask from these
channels and apply \eqref{eq:dense-selective-update} once to each full
stacked Gaussian state. Update its regime weights $\omega_{t,m}$ using
the same observed-channel likelihoods (Appendix~\ref{app:dense-filter}).
\ENDFOR
\end{algorithmic}
\end{algorithm}

Here $\operatorname{LSE}_i(v_i)=v_{\max}+
\log\sum_i\exp(v_i-v_{\max})$, with $v_{\max}=\max_i v_i$.
Keep normalized log weights between rounds and exponentiate them only
to form weighted moments. For $L=1$ the stored log weight is zero. A previously unregistered expert first needs the declared
emission parameters and private-state prior of Appendix~\ref{app:pr-identification}.

\subsection{A Predictive-Moment Budget Under Greedy Feedback}
\label{app:pr-likelihood}
Let $a^\star_t=\mathbb E[\tilde y_t\mid\mathcal P_t]$,
$c^\star_{t,k}=\mathbb E[C_{t,k}/C_*\mid\mathcal P_t]$, and define
\[
 \Delta_{t,j}=\operatorname{kl}(a^\star_t,a_{t,j})+
                   \operatorname{kl}(c^\star_{t,I_t},c_{t,j,I_t}),\qquad
 s_t=(a_t-a^\star_t)^2+(c_{t,I_t}-c^\star_{t,I_t})^2.
\]
Here $\operatorname{kl}$ is Bernoulli KL, with $0\log0=0$.
\begin{lemma}[Moment-score budget]
\label{lem:pr-budget}
For positive normalized initial weights, the no-share update satisfies
\begin{equation}
\label{eq:pr-budget}
 \mathbb E\sum_t\{\operatorname{kl}(a^\star_t,a_t)+
           \operatorname{kl}(c^\star_{t,I_t},c_{t,I_t})\}
 \leq\mathcal B_T
 :=\min_{j}\left\{2\log(1/w_{1,j})+
                         \mathbb E\sum_t\Delta_{t,j}\right\}.
\end{equation}
Consequently, $\mathbb E\sum_ts_t\leq\mathcal B_T/2$.
\end{lemma}
\begin{proof}
For $\vect x\in(0,\infty)^d$ and $r_i\geq0$ with $\sum_i r_i=1$, the Hessian
of $F(\vect x)=\prod_i x_i^{r_i}$ obeys, for every $\vect h\in\mathbb R^d$,
\[
 \vect h^\top\nabla^2F(\vect x)\vect h=F(\vect x)\left[
 (\sum_i r_ih_i/x_i)^2-\sum_i r_i(h_i/x_i)^2\right]\leq0.
\]
Weighted Cauchy--Schwarz proves the sign. Composing with
$(a,1-a,c,1-c)$ and exponents
$(\tilde y_t/2,(1-\tilde y_t)/2,V_t/2,(1-V_t)/2)$ proves
$e^{-\ell_t}\geq\sum_jw_{t,j}e^{-\ell_{t,j}}$, where $\ell_t$
is the score of the aggregate moments. For each fixed candidate,
the update gives
$\log(w_{t+1,j}/w_{t,j})=-\ell_{t,j}
-\log\sum_iw_{t,i}e^{-\ell_{t,i}}$.
Summing over rounds and using the preceding concavity inequality yields
\[
 \sum_t\ell_t\leq\sum_t\ell_{t,j}
       +\log(w_{T+1,j}/w_{1,j})
 \leq\sum_t\ell_{t,j}+\log(1/w_{1,j}).
\]
This holds simultaneously for every fixed candidate on the realized
score table, including adaptive filters. Since $I_t$ is predictable,
conditional expectation of twice each score is
$h_2(a^\star_t)+h_2(c^\star_{t,I_t})+\Delta_{t,j}$, with the analogous
aggregate expression. The entropies cancel. All scores are bounded by
$\log(2/\epsilon)$, ensuring integrability. Taking expectations for each fixed candidate
and then minimizing over its index proves the KL bound.
Finally $\operatorname{kl}(p,q)\geq2(p-q)^2$: the second derivative
in $p$ of the difference is $1/[p(1-p)]-4\geq0$, and its value
and derivative vanish at $p=q$. Continuity handles $p=0,1$.
\end{proof}

The lemma controls the selected cost only. A regret proof must still
relate this observed information to the unselected alternatives.

\input{Section/AppendixParts/greedy_guarantees}
\input{Section/AppendixParts/greedy_archive_robust}
\input{Section/AppendixParts/greedy_exact}

%% file: Section/AppendixParts/greedy_guarantees.tex
\subsection{When Live Feedback Identifies Cost Differences}
\label{app:pr-theorem}
Only differences between action costs determine a greedy decision.
Estimating every latent coordinate or every absolute cost is unnecessary.
The following condition isolates the information that the observed
channels must convey through a quantitative relation between the
predictive model and the true law.

\begin{assumption}[Observable cost contrasts]
\label{ass:greedy-observability}
There are a deterministic $\Lambda\geq0$ and predictable $\xi_t\geq0$
such that, almost surely, for $i=I_t$ and every $k\in\mathcal A_t$,
\begin{equation}
\label{eq:greedy-observability}
 |(c^0_{t,i}-c^0_{t,k})-(c^\star_{t,i}-c^\star_{t,k})|
 \leq\Lambda\sqrt{(a_t^0-a^\star_t)^2+(c^0_{t,i}-c^\star_{t,i})^2}+\xi_t.
\end{equation}
\end{assumption}
The exact case is $\xi_t=0$. The condition permits a common
action-independent cost error to remain unidentified.
Appendix~\ref{app:greedy-bridge} derives it from shared emission
moments. Classical guarantees for greedy contextual learning also use
information conditions: Assumption~3 and Theorem~1 of
\citet{bastani2021mostly} use covariate diversity in a two-arm linear
model. For the switching model here, our condition instead uses the
target observed under every action.

For a predictable $\varphi_t\in[-B,B]$, replace only the internal cost
by $C^\varphi_{t,0}=(\varphi_t-y_t)^2$. Let $k_t^\varphi\in\mathcal A_t$ minimize
its conditional expected cost and put
$\operatorname{Reg}_T(\varphi)=\mathbb E\sum_t
(C_{t,I_t}-C^\varphi_{t,k_t^\varphi})$.
This is a dynamic conditional-risk comparator on the deployed history,
not the history of a counterfactual policy. In particular
$\varphi_t=y^\star_t:=\mathbb E[y_t\mid\mathcal P_t]$ is admissible.

\GreedyRegret*
For fixed $B,C_*,\Lambda$, $\epsilon=1/(T+1)$,
$\mathcal B_T=o(T)$ and $\mathbb E\sum_t\xi_t=o(T)$ imply
$[\operatorname{Reg}_T(\varphi)]_+=o(T)$ as $T\to\infty$.
For $\varphi_t=y^\star_t$, conditional regret is nonnegative, so
$\operatorname{Reg}_T(y^\star)=o(T)$.
\begin{proof}
Let $k_t^*\in\arg\min_kc^\star_{t,k}$ with predictable tie breaking, and
let $r_t^{\rm g}=c^\star_{t,I_t}-c^\star_{t,k_t^*}\geq0$.
Because $c^0_{t,I_t}\leq c^0_{t,k_t^*}$,
\[
 r_t^{\rm g}\leq
 (c^\star_{t,I_t}-c^0_{t,I_t})+(c^0_{t,k_t^*}-c^\star_{t,k_t^*})
 \leq\Lambda(\sqrt{s_t}+\epsilon/\sqrt2)+\xi_t.
\]
The last step uses \eqref{eq:greedy-observability} and the triangle
inequality: smoothing changes each of the two forecast coordinates by
at most $\epsilon/2$. Cauchy--Schwarz on the product of the probability
space and counting measure gives
$\mathbb E\sum_t\sqrt{s_t}\leq\sqrt{T\mathbb E\sum_ts_t}$.

The squared-loss decomposition, conditional on $\mathcal P_t$, is
\[
 \mathbb E[(f_t-y_t)^2-(\varphi_t-y_t)^2\mid\mathcal P_t]
 =(f_t-y^\star_t)^2-(\varphi_t-y^\star_t)^2\leq(f_t-y^\star_t)^2.
\]
Evaluate the deployed cost vector at $k_t^\varphi$: only its internal
entry changes, so its minimum exceeds the comparator minimum by at
most $(f_t-y^\star_t)^2$. Also
$(f_t-y^\star_t)^2=4B^2(a_t-a^\star_t)^2
\leq2B^2\operatorname{kl}(a^\star_t,a_t)$.
Thus each conditional regret is at most
$C_*r_t^{\rm g}+(f_t-y^\star_t)^2$. Sum and apply
Lemma~\ref{lem:pr-budget}. This proves exactly \eqref{eq:pr-regret}.
For an exact unsmoothed candidate, each smoothed Bernoulli mass is at
least $(1-\epsilon)$ times its true mass. Hence its two-channel KL sum
is at most $2T\log(1/(1-\epsilon))$. With $\epsilon=1/(T+1)$ this is
at most two, and $w_{1,j}\geq\zeta/L$ gives the exact-candidate budget
$\mathcal B_T\leq2\log(L/\zeta)+2$ used in
Corollary~\ref{cor:pr-sublinear}(i).
\end{proof}

\subsection{A Structural Condition for Observable Cost Contrasts}
\label{app:greedy-bridge}
Assumption~\ref{ass:greedy-observability} holds whenever the observed
channels reveal the latent information on which cost differences depend.
Consider bounded emission laws with $M$ latent states. At a fixed history,
let $\vect\tau\in[0,1]^M$ contain the state-conditional means of the
normalized target, and let $\vect\sigma_k\in[0,1]^M$ contain the
state-conditional normalized costs of action $k$, evaluating action zero
at the common $f_t$. Suppose the truth and all candidates share these
emission moments, but may have different predictive state probabilities
$\vect\pi,\vect\pi_j\in\Delta^M$. With $\bar{\vect\pi}=\sum_jw_{t,j}\vect\pi_j$,
\[
 a^\star_t=\vect\tau^\top\vect\pi,\quad a^0_t=\vect\tau^\top\bar{\vect\pi},\quad
 c^\star_{t,k}=\vect\sigma_k^\top\vect\pi,\quad c^0_{t,k}=\vect\sigma_k^\top\bar{\vect\pi}.
\]
Transitions, initial beliefs, contexts, and mixture weights may vary.
Differences between fitted emission moments contribute the explicit
residual derived below.

\begin{proposition}[Cost-contrast bridge]
\label{prop:greedy-bridge}
Suppose for each available $i,k$ there are coefficients $b_0,b_1,b_2$
such that
\begin{equation}
\label{eq:greedy-bridge}
 \vect\sigma_i-\vect\sigma_k=b_0\mathbf1+b_1\vect\tau+b_2\vect\sigma_i,\qquad
 b_1^2+b_2^2\leq\Lambda^2.
\end{equation}
Then Assumption~\ref{ass:greedy-observability} holds with $\xi_t=0$.
It suffices to verify the condition for actions the router can select.
\end{proposition}
\begin{proof}
Take the inner product of \eqref{eq:greedy-bridge} with
$\bar{\vect\pi}-\vect\pi$, whose inner product with $\mathbf1$ is zero.
The two remaining terms are $b_1(a^0_t-a^\star_t)+b_2(c^0_{t,i}-c^\star_{t,i})$,
whose absolute value is at most
$\Lambda\sqrt{(a^0_t-a^\star_t)^2+(c^0_{t,i}-c^\star_{t,i})^2}$ by
Cauchy--Schwarz.
\end{proof}

The condition tests cost contrasts, not full-state identification. With
$\mat J_M=\mat I_M-\mathbf1\mathbf1^\top/M$ and
$\mat F_i=[\vect\tau^\top;\vect\sigma_i^\top]\in\mathbb R^{2\times M}$,
it is equivalent to
$(\vect\sigma_i-\vect\sigma_k)^\top\mat J_M\in\operatorname{row}(\mat F_i\mat J_M)$
with a bounded minimum-norm coefficient, which a singular value
decomposition computes. Column rank $M$ of $[\mathbf1^\top;\mat F_i]$ is
sufficient but not necessary: unobservable state directions are harmless
when they do not affect cost differences. HMM rank and singular-value
conditions, such as Conditions~1--2 of \citet{hsu2012spectral}, provide
related identification tools; Proposition~\ref{prop:greedy-bridge}
establishes the decision-specific condition used here.

For $M=2$, if $|\tau_1-\tau_2|\geq d>0$, every contrast has the
representation with $b_2=0$ and
$b_1=[(\sigma_i-\sigma_k)_1-(\sigma_i-\sigma_k)_2]/(\tau_1-\tau_2)$.
Since its numerator has magnitude at most two, $\Lambda=2/d$ is valid
uniformly over action sets and internal predictions. This example
allows arbitrarily many experts and regime changes, and it does not
require querying every expert repeatedly: the always-observed target
carries the information. When target means coincide, selected-cost
information can still make $\mat F_i\mat J_M$ informative.

\noindent\textbf{Estimated emission moments.}
Suppose the true emission moments obey \eqref{eq:greedy-bridge}, but every
candidate's target moments differ coordinatewise by at most $e_{a,t}$ and
each available cost moment by at most $e_{c,t}$, with arbitrary predictive
state probabilities. Applying the exact bridge to the true moments at
$\bar{\vect\pi}$ and accounting for the perturbations gives
\begin{equation}
\label{eq:greedy-bridge-error}
 \xi_t=2e_{c,t}+\Lambda\sqrt{e_{a,t}^2+e_{c,t}^2}.
\end{equation}
The two cost errors perturb a contrast by at most $2e_{c,t}$, and the
observed two-coordinate vector changes by at most
$\sqrt{e_{a,t}^2+e_{c,t}^2}$. If the true bridge itself has a
coordinatewise residual at most $e_{b,t}$, add $2e_{b,t}$, since
$\|\bar{\vect\pi}-\vect\pi\|_1\leq2$.

%% file: Section/AppendixParts/greedy_archive_robust.tex
\subsection{Historical and Live Evidence with Approximate Candidates}
\label{app:archive-robust}
Full-feedback calibration need not identify one exactly correct model.
It suffices to suppress candidates outside a predictively accurate set.
We prove this for the unchanged greedy rule and average moment-score
update without fixed share; the reference used in the proof is an existing
fitted candidate, not an additional component of the algorithm.

Condition on the fitting segment, but average over calibration and
deployment. Fix a nonempty set $\mathcal G\subseteq\{1,\ldots,L\}$ and
a reference $j^\circ\in\mathcal G$, both fixed under this conditioning.
Write $\mathcal J=\{1,\ldots,L\}\setminus\mathcal G$. The set
$\mathcal G$ can contain several approximate candidates; it need not
contain the true parameter. All candidate costs use the same deployed
internal prediction. The comparator is the conditional-risk oracle of
Theorem~\ref{thm:pr-main}, including its Bayes internal prediction
$y^\star_t=\mathbb E[y_t\mid\mathcal P_t]$.

\noindent\textbf{Evidence and reference error.}
At any calibration or deployment date, let $\vect W\in[0,1]^m$ contain the
channels actually scored, $\vect p=\mathbb E[\vect W\mid\mathcal F]$, and let
$\vect q_j\in(0,1)^m$ be the cached smoothed forecasts. Here $\mathcal F$
is the information before that date's feedback. Calibration scores the
target, the internal cost, and all available external costs; deployment
scores the target and selected cost. The channels can be dependent.
Define
\begin{equation}
\label{eq:robust-evidence}
 e=\frac1m\sum_{\ell=1}^m
       \frac{(p_\ell-q_{j^\circ,\ell})^2}
            {q_{j^\circ,\ell}(1-q_{j^\circ,\ell})},\qquad
 \gamma_j=\frac1m\|\vect q_j-\vect q_{j^\circ}\|_2^2.
\end{equation}
The reference error $e$ is an average Bernoulli $\chi^2$ divergence.
The separation $\gamma_j$ is a difference between computed forecasts; it
does not require a queried outcome from every expert at deployment.

\begin{lemma}[Approximate-reference score contraction]
\label{lem:robust-quarter-affinity}
For $\ell(\vect q)=m^{-1}\sum_{\ell=1}^m\ell_{\rm B}(W_\ell,q_\ell)$,
\begin{equation}
\label{eq:robust-quarter-affinity}
 \mathbb E\!\left[
 e^{-(\ell(\vect q_j)-\ell(\vect q_{j^\circ}))/4}\mid\mathcal F\right]
 \leq e^{e/2-\gamma_j/4}.
\end{equation}
This holds for every $\vect p\in[0,1]^m$ and interior candidate forecasts,
without conditional independence of the channels or exact realizability.
\end{lemma}
\begin{proof}
For a scalar channel, exponential convexity in $W\in[0,1]$ bounds
the conditional expectation by its endpoint chord. Write
$p_b,q_b,q_b^\circ$ for the two Bernoulli masses. Cauchy--Schwarz gives
\[
 \sum_{b=0}^1p_b(q_b/q_b^\circ)^{1/4}
 \leq\left(\sum_b\frac{p_b^2}{q_b^\circ}\right)^{1/2}
       \left(\sum_b\sqrt{q_bq_b^\circ}\right)^{1/2}.
\]
The first squared factor is
$1+(p-q^\circ)^2/[q^\circ(1-q^\circ)]$.
For Bernoulli probabilities $x,y$, their affinity obeys
\[
 \sqrt{xy}+\sqrt{(1-x)(1-y)}
 \leq\sqrt{1-(x-y)^2}\leq e^{-(x-y)^2/2}.
\]
Indeed, set $a=\sqrt{x(1-y)}$ and $b=\sqrt{(1-x)y}$.
The squared affinity is $1-(a-b)^2$, while
$x-y=(a-b)(a+b)$ and $a+b\leq1$ by Cauchy--Schwarz.
Using $\log(1+x)\leq x$ proves the scalar bound. Conditional
H\"older with $m$ equal exponents combines the $m$ scalar bounds,
giving \eqref{eq:robust-quarter-affinity}. All denominators are positive.
\end{proof}

Write $\mathrm H$ and $\mathrm L$ for calibration and deployment.
For the reference $j^\circ$, use the score errors and separations in
\eqref{eq:robust-evidence}. The initial calibration weights are uniform;
deployment starts with $w_{1,j}=(1-\zeta)\nu_j+\zeta/L$.

\noindent\textbf{How the two feedback phases combine.}
The main-text theorem bounds only the total reference error accumulated
over calibration and deployment. How this total divides between the two
phases can be random, so deterministic phase budgets with the same total
need not be available. We therefore follow a single supermartingale
through both phases: it tracks a historical odds ratio until the
calibration boundary and a live odds ratio afterwards. Changing the
compared candidate at this deterministic boundary preserves the
conditional score-contraction argument.

Write ${\rm ref}=j^\circ$. For every candidate $i$ and live prefix
$0\leq n\leq T$, define the historical and live score-ratio factors
\[
 \mathsf H_i=\exp\Bigl\{-\frac14\sum_{s=1}^N
      (\ell^{\rm H}_{s,i}-\ell^{\rm H}_{s,{\rm ref}})\Bigr\},\qquad
 \mathsf L_i(n)=\exp\Bigl\{-\frac14\sum_{t=1}^n
      (\ell^{\rm L}_{t,i}-\ell^{\rm L}_{t,{\rm ref}})\Bigr\},
\]
the actual separations $D_{{\rm H},i}=\sum_{s=1}^N\gamma^{\rm H}_{s,i}$ and
$D_{{\rm L},i}(n)=\sum_{t=1}^n\gamma^{\rm L}_{t,i}$, and the total reference
error $E(n)=\sum_{s=1}^Ne^{\rm H}_s+\sum_{t=1}^ne^{\rm L}_t$.
The reference has $\mathsf H_{\rm ref}=\mathsf L_{\rm ref}(n)=1$ and zero separations.

\begin{lemma}[Stitched two-phase evidence]
\label{lem:robust-stitched}
Use uniform calibration priors and the no-share update of
Algorithm~\ref{alg:pr-complete}. Let $A,\gamma_{\rm H},\gamma_{\rm L}\geq0$ be
deterministic, and suppose that, almost surely, $E(T)\leq A$ and every
$j\in\mathcal J$ satisfies $D_{{\rm H},j}\geq N\gamma_{\rm H}^2$ and
$D_{{\rm L},j}(n)\geq n\gamma_{\rm L}^2$ for all $0\leq n\leq T$. Then the
mass $b_t=\sum_{j\in\mathcal J}w_{t,j}$ on inaccurate candidates obeys
$\mathbb E b_t\leq \bar b_t=\min\{1,Ue^{-(t-1)\gamma_{\rm L}^2/4}\}$, with
$U=(L-1)e^{A/2}(e^{-N\gamma_{\rm H}^2/4}+La_\zeta)$ and
$a_\zeta=(\zeta/(L(1-\zeta)))^{1/4}$, as in the main text.
\end{lemma}
\begin{proof}
Under uniform priors $\mathsf H_i=(\nu_i/\nu_{\rm ref})^{1/4}$, and the live
score update gives $R_{t,j}=R_{1,j}\mathsf L_j(t-1)$ for the fourth-root odds
$R_{t,j}=(w_{t,j}/w_{t,{\rm ref}})^{1/4}$. The mixing weight gives, by subadditivity of the fourth root,
\[
 (w_{1,j}/w_{1,{\rm ref}})^{1/4}
 \leq(\nu_j/\nu_{\rm ref})^{1/4}+a_\zeta\nu_{\rm ref}^{-1/4},
 \qquad
 \nu_{\rm ref}^{-1/4}
 =\left(1+\sum_{i\ne{\rm ref}}\nu_i/\nu_{\rm ref}\right)^{1/4}
 \leq1+\sum_{i\ne{\rm ref}}(\nu_i/\nu_{\rm ref})^{1/4}.
\]
Multiplying by the live factor yields, pathwise,
\begin{equation}
\label{eq:robust-stitched-split}
 R_{t,j}\leq\mathsf H_j\mathsf L_j(t-1)+a_\zeta\mathsf L_j(t-1)
   +a_\zeta\sum_{i\ne{\rm ref}}\mathsf H_i\mathsf L_j(t-1).
\end{equation}
Fix candidates $i,j$ and order the $N+T$ scored dates chronologically.
On calibration date $s$ multiply by
$\exp\{-(\ell^{\rm H}_{s,i}-\ell^{\rm H}_{s,{\rm ref}})/4
      +\gamma^{\rm H}_{s,i}/4-e^{\rm H}_s/2\}$,
and on live date $t$ by the same expression with ${\rm L}$ and $j$ in
place of ${\rm H}$ and $i$. After $n$ live dates the product is
\[
 M_{ij}(n)=\mathsf H_i\mathsf L_j(n)
 \exp\Bigl\{\frac{D_{{\rm H},i}+D_{{\rm L},j}(n)}4-\frac{E(n)}2\Bigr\}.
\]
Each date's separation and reference error are determined before its
feedback, so Lemma~\ref{lem:robust-quarter-affinity} bounds the
conditional expectation of every factor by one. The candidate compared
with the reference changes from $i$ to $j$ only at the deterministic
calibration boundary. The full chronological product is therefore a
nonnegative supermartingale starting at one before calibration, and
$\mathbb E M_{ij}(n)\leq1$.

Since every $e_s\geq0$, $E(n)\leq E(T)\leq A$. Moreover
$D_{{\rm H},i}\geq0$ for every $i$, $D_{{\rm H},j}\geq N\gamma_{\rm H}^2$, and
$D_{{\rm L},j}(n)\geq n\gamma_{\rm L}^2$ for $j\in\mathcal J$. Hence, pathwise,
\[
 \mathsf H_i\mathsf L_j(n)\leq M_{ij}(n)
 \exp\Bigl\{\frac A2-\frac{D_{{\rm H},i}}4-\frac{n\gamma_{\rm L}^2}4\Bigr\}.
\]
Take $n=t-1$ and expectations in \eqref{eq:robust-stitched-split}.
The first term has $i=j$, the second has $i={\rm ref}$, and the sum has
$L-1$ terms with $D_{{\rm H},i}\geq0$, so
\[
 \mathbb E R_{t,j}\leq e^{A/2-(t-1)\gamma_{\rm L}^2/4}
 \bigl(e^{-N\gamma_{\rm H}^2/4}+La_\zeta\bigr).
\]
Finally $w_{t,j}\leq R_{t,j}$ and $b_t\leq1$. Summing over the at most
$L-1$ candidates in $\mathcal J$ and capping the expectation at one
gives $\mathbb E b_t\leq \bar b_t$.
\end{proof}

\begin{proposition}[From joint prediction to moment errors]
\label{prop:robust-joint-certificate}
Let $b_{c,t}$ bound every good candidate's raw normalized cost error,
and let $b_{a,t}$ bound its smoothed normalized target-mean error.
Suppose each good candidate's conditional target law and each available
target--expert pair law have truth-to-candidate $\chi^2$ divergence at
most $\kappa_t$, along the actual observed history.
Then the following envelopes are valid:
\begin{equation}
\label{eq:robust-joint-certificate}
 e_t^{\rm L}\leq\frac{\kappa_t+\epsilon}{1-\epsilon},\qquad
 b_{c,t}=\frac{\sqrt{\kappa_t}}2,\qquad
 b_{a,t}=\frac{\sqrt{\kappa_t}+\epsilon}{2},
\end{equation}
The analogous reference-score bound during calibration requires the
same divergence bounds for the reference at each calibration date.
Here $\chi^2(P\Vert Q)=\int(dP/dQ-1)^2\,dQ$ requires $P\ll Q$.
For a deterministic cumulative reference-score budget, these conditional
divergence bounds are required along the observed histories almost surely,
or on the joint certificate event treated below.
\end{proposition}
\begin{proof}
For any bounded normalized moment $F\in[0,1]$, put $p=\mathbb E_PF$
and $q=\mathbb E_QF$. Cauchy--Schwarz yields
\[
 (p-q)^2=\{\mathbb E_Q[(dP/dQ-1)(F-q)]\}^2
 \leq\chi^2(P\Vert Q)\operatorname{Var}_Q(F)
 \leq\kappa q(1-q).
\]
This is also data processing through the auxiliary Bernoulli kernel
with success probability $F$. Hence the Bernoulli divergence of the
unsmoothed moment is at most $\kappa$, and $|p-q|\leq\sqrt\kappa/2$.
For $q'=s_\epsilon(q)$, each Bernoulli mass of $q'$ is at least
$(1-\epsilon)$ times the corresponding mass of $q$. Thus
\[
 1+\chi^2(\operatorname{Bern}(p)\Vert\operatorname{Bern}(q'))
 \leq\frac{1+\kappa}{1-\epsilon}.
\]
Endpoints follow by absolute continuity or a limit. Average this
inequality over the scored channels. Normalized costs are bounded
functions of the common target--expert pair; the common internal
prediction is fixed when conditioning. Finally smoothing changes a
target moment by at most $\epsilon/2$. These observations give
\eqref{eq:robust-joint-certificate}.
\end{proof}

\HistoricalRegret*
\begin{proof}
Proposition~\ref{prop:robust-joint-certificate} gives, for every good
candidate, raw cost error at most $\sqrt{\kappa_t}/2$ and smoothed
target-mean error at most $(\sqrt{\kappa_t}+\epsilon)/2$.
Let $b_t=\sum_{j\notin\mathcal G}w_{t,j}$. Since all normalized moments
are in $[0,1]$, splitting good and bad candidates yields
\[
 \max_k|c^0_{t,k}-c^\star_{t,k}|\leq\sqrt{\kappa_t}/2+b_t,\qquad
 (a_t-a^\star_t)^2\leq(\sqrt{\kappa_t}+\epsilon)^2/4+b_t.
\]
The second inequality uses convexity of the square. Greedy optimality
therefore bounds normalized routing regret by $\sqrt{\kappa_t}+2b_t$.
The squared-loss decomposition in the proof of Theorem~\ref{thm:pr-main}
adds at most $4B^2(a_t-a^\star_t)^2$ for the Bayes internal comparator.
Sum these bounds and apply $\mathbb E b_t\leq \bar b_t$ from
Lemma~\ref{lem:robust-stitched}. This gives exactly the displayed bound.
The joint-certificate event is handled below without changing the
conditional probability law; its complement contributes at most
$C_*T\delta$.
\end{proof}

\noindent\textbf{Certificates that hold with high probability.}
Let $\mathcal C$ be a fitting-conditioned event of probability at least
$1-\delta$ on which the main theorem's deterministic error and separation
certificates hold. This event can depend on future observations, so we
do not condition the probability law on it: that would change the
conditional means used in Lemma~\ref{lem:robust-quarter-affinity}.
Instead, restrict the nonnegative products $M_{ij}(n)$ to the event.
Their unconditional bounds imply
$\mathbb E[M_{ij}(n)\mathbf1_{\mathcal C}]\leq1$.
The pathwise bounds in the stitched proof hold on $\mathcal C$, and
multiplying \eqref{eq:robust-stitched-split} by its indicator therefore gives
$\mathbb E[b_t\mathbf1_{\mathcal C}]\leq \bar b_t$ with the main-text envelope.

Let $\mathfrak r_t=\mathbb E[C_{t,I_t}\mid\mathcal P_t]
-\min_k\mathbb E[C^{y^\star}_{t,k}\mid\mathcal P_t]$.
The Bayes internal forecast minimizes conditional squared loss, so
$0\leq\mathfrak r_t\leq C_*$ and
$\operatorname{Reg}_T(y^\star)=\mathbb E\sum_t\mathfrak r_t$.
On $\mathcal C$, the preceding good/bad-mass calculation gives
\[
 \mathfrak r_t\leq C_*\sqrt{\kappa_t}
 +B^2(\sqrt{\kappa_t}+\epsilon)^2+(2C_*+4B^2)b_t.
\]
Taking expectations after multiplying by $\mathbf1_{\mathcal C}$
gives the stated bound; the complement contributes at most $C_*T\delta$.
Every other predictable internal comparator has conditional risk at
least that of $y^\star_t$, so its regret has the same upper bound.
The resulting expected-regret bound incorporates certificate failure
through the additive term $C_*T\delta$.

\subsection{A Joint-Prediction Certificate for the Anchored SLDS}
\label{app:archive-robust-slds}
Proposition~\ref{prop:robust-joint-certificate} uses the target marginal
and each target--expert pair, rather than a density on all experts jointly.
The target alone suffices for the internal cost; external costs require
the joint pair. We now account for numerical moments and express these
divergences in the Gaussian components maintained by an SLDS.

\noindent\textbf{Numerical moment integration.}
The preceding density certificate concerns exact moments of the computed
predictive law. Suppose each computed raw normalized moment differs
from that integral by at most $\tau_t$, and remains in $[0,1]$.
Put $v_\epsilon=(\epsilon/2)(1-\epsilon/2)$. A sufficient replacement is
\[
 e_t^{\rm L}\leq
 2\left(1+\frac{\tau_t}{v_\epsilon}\right)
             \frac{\kappa_t+\epsilon}{1-\epsilon}
       +\frac{2\tau_t^2}{v_\epsilon},\qquad
 b_{c,t}=\frac{\sqrt{\kappa_t}}2+\tau_t,\quad
 b_{a,t}=\frac{\sqrt{\kappa_t}+\epsilon}2+\tau_t.
\]
Indeed, ideal and computed smoothed moments differ by at most $\tau_t$;
both Bernoulli variances are at least $v_\epsilon$, and their variances
differ by at most $\tau_t$. The squared triangle inequality bounds
the computed numerator by twice the ideal squared error plus
$2\tau_t^2$. Dividing by the computed variance proves the score bound.
Apply it in both phases when numerical moments are used.
For $\epsilon=1/(T+1)$, $\kappa_t=O(T^{-1})$ and
$\tau_t=O(T^{-1})$ preserve a bounded cumulative reference-error
budget over $O(T)$ dates. Here $\tau_t$ is a certified integration-error
bound for the moments used by the router.

\noindent\textbf{An explicit Gaussian-mixture calculation.}
For $P=\mathcal N(\vect\mu_P,\mat S_P)$ and $Q=\mathcal N(\vect\mu_Q,\mat S_Q)$ in
$\mathbb R^d$, with both covariances positive definite,
\begin{equation}
\label{eq:robust-gaussian-chi}
 1+\chi^2(P\Vert Q)=
 \frac{\det\mat S_Q}{\sqrt{\det\mat S_P\det(2\mat S_Q-\mat S_P)}}
 \exp\{(\vect\mu_P-\vect\mu_Q)^\top(2\mat S_Q-\mat S_P)^{-1}(\vect\mu_P-\vect\mu_Q)\}
\end{equation}
provided $2\mat S_Q-\mat S_P\succ0$; otherwise the divergence is infinite.
To verify the formula, whiten $Q$, diagonalize
$\mat S_Q^{-1/2}\mat S_P\mat S_Q^{-1/2}$, and complete the square in each coordinate
of $\int p^2/q$. The factors are
$[s(2-s)]^{-1/2}\exp\{\delta^2/(2-s)\}$ for eigenvalue $s$ and
whitened mean difference $\delta$. Multiplying them proves the formula,
including noncommuting original covariance matrices.

For mixture laws $P=\sum_rp_rP_r$ and $Q=\sum_rq_rQ_r$ on a common
component index with $q_r>0$ whenever $p_r>0$, data processing gives
\begin{equation}
\label{eq:robust-mixture-chi}
 1+\chi^2(P\Vert Q)
 \leq\sum_r\frac{p_r^2}{q_r}\{1+\chi^2(P_r\Vert Q_r)\}.
\end{equation}
Indeed, the right side is $\int\widetilde p^2/\widetilde q$ for the
joint component/output laws; marginalizing the component reduces the
divergence by conditional Jensen. To apply the Gaussian component
formula to an SLDS, assume that inverse-link predictions are Gaussian
conditional on the complete observed history, the regime history, and
an initial mixture component. In particular, contexts and availability
must carry no additional latent-state information beyond the modeled
history, or must enter through Gaussian-compatible observations.
Conditioning on arbitrary state-dependent side information need not
preserve Gaussianity. The mixture index includes the regime history
and any initial component, including fitting-era histories when needed
to represent the initial conditional law. For a dense IMM
candidate, lift its current-mode weight to these histories using the
true within-mode conditional history weights in the proof. Each lifted
candidate component then has its computed current-mode Gaussian law.
The resulting marginal candidate law is unchanged. Modes of zero true
weight can be included as extra zero-true-mass components. This is a
mathematical comparison, not an extra filtering operation.

For example, suppose the discrete current-mode divergence is at most
$\kappa_z$, and every compared Gaussian component obeys
\[
 \|\mat S_Q^{-1/2}(\mat S_P-\mat S_Q)\mat S_Q^{-1/2}\|_{\rm op}\leq h<1,
 \qquad\|\mat S_Q^{-1/2}(\vect\mu_P-\vect\mu_Q)\|_2\leq r.
\]
Equations~\eqref{eq:robust-gaussian-chi}--\eqref{eq:robust-mixture-chi}
give the sufficient pairwise certificate
\begin{equation}
\label{eq:robust-slds-certificate}
 \kappa\leq(1+\kappa_z)
 \exp\left\{\frac{dh^2}{2(1-h^2)}+\frac{r^2}{1-h}\right\}-1.
\end{equation}
To check it, the eigenvalues $s$ lie in $[1-h,1+h]$ and
$s(2-s)=1-(s-1)^2$. Therefore
$-\frac12\log\det(\mat S(2\mat I-\mat S))\leq dh^2/[2(1-h^2)]$ for
$\mat S=\mat S_Q^{-1/2}\mat S_P\mat S_Q^{-1/2}$, and the
mean quadratic is at most $r^2/(1-h)$. The mode-weight factors sum
to $1+\kappa_z$. Here $d=1$ for the target and $d=2$ for a
target--expert pair. The bounded observation link is a common measurable
map, so it cannot increase this divergence. There is no multiplicative
factor equal to the number of experts, although estimation complexity
and the certificate constants can depend on that number.

Errors in Gaussian means, covariances, and mode weights, whether from
parameter fitting or filtering, all enter this certificate. Full
covariance removes the cross-expert blocks discarded by the star
approximation, but a dense IMM still collapses mixtures of past regime
histories: it is exact for a single Gaussian regime with correct
parameters and initialization. With multiple regimes, $\kappa_t$
accounts for the remaining mixture-collapse and parameter-fitting errors.

\subsection{An Explicit Sublinear Archive--Horizon Tradeoff}
\label{app:archive-robust-rate}
The next consequence states what predictive accuracy and archive size
suffice.
Use exact predictive moment integration, uniform calibration priors,
$L\geq2$, $\epsilon=1/(T+1)$,
and a deterministic uniform certificate $\kappa$ in both phases.
Let
\[
 H=\frac{(N+T)(\kappa+\epsilon)}{1-\epsilon}.
\]
Suppose each $j\in\mathcal J$ has historical separation
$D_{{\rm H},j}\geq N\gamma_{\rm H}^2$ for some $\gamma_{\rm H}>0$.
Set all live separation lower bounds to zero. Then
\begin{equation}
\label{eq:robust-uniform-mass}
 \sum_{t=1}^T\bar b_t\leq
 T\min\left\{1,e^{H/2}(L-1)
       [e^{-N\gamma_{\rm H}^2/4}+La_\zeta]\right\}.
\end{equation}
Substitute $A=H$ and $\gamma_{\rm L}=0$ into the envelope of
Lemma~\ref{lem:robust-stitched}.
In particular, if
\begin{equation}
\label{eq:robust-archive-size}
 \frac{N\gamma_{\rm H}^2}{4}\geq\frac H2+\log(2(L-1)T),\qquad
 a_\zeta\leq\frac{e^{-H/2}}{2L(L-1)T},
\end{equation}
then $\sum_t\bar b_t\leq1$ and
\begin{equation}
\label{eq:robust-explicit-rate}
 \operatorname{Reg}_T(y^\star)\leq
 C_*T\sqrt\kappa+B^2T(\sqrt\kappa+\epsilon)^2
                      +2C_*+4B^2.
\end{equation}
The archive requirement can be made explicit. Multiplying the first
inequality in \eqref{eq:robust-archive-size} by four and collecting the
terms in $N$ gives, when $(1-\epsilon)\gamma_{\rm H}^2>2(\kappa+\epsilon)$,
\[
 N\geq\left\lceil
 \frac{2T(\kappa+\epsilon)+4(1-\epsilon)\log(2(L-1)T)}
      {(1-\epsilon)\gamma_{\rm H}^2-2(\kappa+\epsilon)}
 \right\rceil.
\]
Raising the second inequality to the fourth power and solving for
$\zeta$ gives the sufficient choice
\[
 0<\zeta\leq
 \frac{1}{1+16L^3(L-1)^4T^4e^{2H}}.
\]
For fixed $L,B,C_*,\gamma_{\rm H}>0$ and $\kappa\leq c/T$ with fixed $c$,
choose $N=\lceil c_N\log(T+1)\rceil$ with $c_N>4/\gamma_{\rm H}^2$.
The archive requirement then holds for all sufficiently large $T$:
its threshold is $(4/\gamma_{\rm H}^2)\log T+O(1)$.
Also $H=O(1)$, so a sufficiently small
$\zeta=c_\zeta(T+1)^{-4}$ with fixed $c_\zeta>0$ meets the mixing requirement.
Substitution into \eqref{eq:robust-explicit-rate} gives
$\operatorname{Reg}_T(y^\star)=O(\sqrt T)$ with no live observability condition.
The fitting archive, which supplies the predictive certificate, is
additional to these $N$ calibration dates.
For a fitting-conditioned joint certificate failure probability
$\delta$, add $C_*T\delta$. Retaining the $O(\sqrt T)$ rate then
requires $\delta=O(T^{-1/2})$ or smaller.
For numerical integration, replace the score budget and moment errors
by the certified $\tau_t$ bounds in
Appendix~\ref{app:archive-robust-slds}; a uniform
$\tau_t=O(T^{-1})$ preserves the rate but changes the displayed constants.

The general bound holds for the fixed mixing weight $\zeta=0.1$ used in the
experiments, whereas the sublinear specialization uses the displayed
horizon-dependent calibration size and mixing weight.
When $\gamma_{\rm L}>0$, the main-text envelope instead gives
$\sum_{t=1}^T \bar b_t\leq U/(1-e^{-\gamma_{\rm L}^2/4})$,
controlling the retained mass with a fixed mixing weight.

%% file: Section/AppendixParts/greedy_exact.tex
\subsection{How a Query Improves the Internal Forecast}
\label{app:pr-transfer}
\begin{proposition}[Value of a query for a future target]
\label{prop:pr-query-value}
Let $\mathcal G$ be the sigma-algebra generated by the observed history
and the current target, and
let $H$ be a potential queried observation. For a square-integrable
future target $Y'$ under a fixed exogenous evaluation context,
\begin{align}
\label{eq:pr-query-value}
 &\mathbb E[(Y'-\mathbb E[Y'\mid\mathcal G])^2\mid\mathcal G]
 -\mathbb E[(Y'-\mathbb E[Y'\mid\mathcal G,H])^2\mid\mathcal G]\\
 &\quad=\mathbb E[(\mathbb E[Y'\mid\mathcal G,H]
                  -\mathbb E[Y'\mid\mathcal G])^2\mid\mathcal G]\geq0.\nonumber
\end{align}
Under a jointly Gaussian conditional law, the value is
$\operatorname{Cov}(Y',H\mid\mathcal G)^2/
\operatorname{Var}(H\mid\mathcal G)$ when the denominator is positive,
and zero when $H$ is $\mathcal G$-measurable.
\end{proposition}
\begin{proof}
Write $Y'-\mathbb E[Y'\mid\mathcal G]$ as
$(Y'-\mathbb E[Y'\mid\mathcal G,H])
+(\mathbb E[Y'\mid\mathcal G,H]-\mathbb E[Y'\mid\mathcal G])$, expand
the square, and condition on $(\mathcal G,H)$; the cross term is zero.
In the Gaussian case with positive conditional variance, the second
summand equals
$\operatorname{Cov}(Y',H\mid\mathcal G)
(H-\mathbb E[H\mid\mathcal G])/\operatorname{Var}(H\mid\mathcal G)$,
whose conditional second moment is the stated ratio. If $H$ is
$\mathcal G$-measurable, the two conditional means agree.
\end{proof}
The covariance is conditional
on the free target feedback. The identity values additional information
about the same future random variable, not an action that changes its
distribution. It measures the reduction in Bayes prediction risk before
consultation fees. The deployed filter estimates the corresponding
conditional laws, while greedy routing minimizes current estimated cost.

\subsection{Scope of the Comparator}
\label{app:pr-scope}
The oracle shares the router's observed history and selects the available
action with the smallest conditional expected cost, using the Bayes
internal forecast. Its choice can change at every round as the observed
history and expert availability evolve.

%% file: Section/AppendixParts/related_work_extended.tex
\section{Additional Related Work}
\label{app:related-work}
The supervised L2D objective connects the routing decision to the
statistical properties of the loss used to train it. Work on consistent
surrogates and sample complexity includes
\citet{Cao_Mozannar_Feng_Wei_An_2023,mozannar2023should,mao2024realizablehconsistentbayesconsistentloss,mao2024principledapproacheslearningdefer,wei2024exploiting}.
Extensions consider multi-task prediction, extractive question answering,
adversarial robustness, and imbalanced expert use
\citep{montreuil2024twostagelearningtodefermultitasklearning,montreuil2026optimal,montreuil2025adversarial,montreuil2026adversarial,strong2024towards,palomba2025a,cortes2026optimized}.
\citet{mao2025thesis} develops a unified treatment of multi-class
abstention and multi-expert deferral.

For greedy contextual bandits, Theorem~1 of
\citet{bastani2021mostly} uses a two-arm linear model with covariate
diversity and a margin condition. Our proofs instead formulate
information conditions on target and selected-cost moments, conditional
on the observed time-series history. Likewise, Theorem~2 of
\citet{shivaswamy2012history} quantifies the benefit of historical
observations for HUCB1, which retains an optimistic selection rule.
Our archive bound analyzes the stated greedy moment-score update and
allows approximate predictive candidates. The dependence assumptions
and comparator are specified directly rather than imported from the
stationary-arm or independent-context settings.

%% file: Section/AppendixParts/experiments.tex
% Restructured experimental appendix (2026-10-01).
% Removed studies and their numbers are listed in
% audit_20260926/appendix_cleanup_20261001/discarded_experiments.md.
% Sources: out/probabilistic_l2d_20260929/fullcov_comparison_release_v1 (climate,
% routing and covariance studies), out/greedy_mechanism_controls_20260930_v1
% (synthetic regime study), out/greedy_studies_20260930_release_v1 (reentry).

\section{Experiments}
\label{app:experiments-greedy}
\label{app:greedy-ablations}
Appendices~\ref{app:exp-datasets}--\ref{app:greedy-baselines} describe the
datasets, protocol, metrics, and baselines. The studies are organized as follows:
\begin{enumerate}[label=\roman*),leftmargin=*,nosep]
\item Baseline performance and consultation fees
(Appendices~\ref{app:exp-main}--\ref{app:jena-consultation-cost}).
\item Greedy routing and regime-switching controls
(Appendices~\ref{app:greedy-policy-comparison}--\ref{app:completed-mechanism-controls}).
\item Expert withdrawal and reentry (Appendix~\ref{app:extended-deployed-reentry}).
\item Covariance representation and expert-pool scaling
(Appendices~\ref{app:greedy-backend-comparison}--\ref{app:jena-pool-scaling}).
\item Runtime and storage against baselines (Appendix~\ref{app:jena-computational-scaling}).
\end{enumerate}

\subsection{Datasets}
\label{app:exp-datasets}
\noindent\textbf{Real datasets.}
We use three public univariate series: Delhi (daily mean temperature),
Jena (the \texttt{T (degC)} channel of the Jena archive), and Melbourne
(daily minimum temperature, 1981--1990). Invalid targets and
nonincreasing timestamps are removed before sampling. Delhi retains 1,568
observations and Melbourne 3,647. For Jena, 327 invalid records are excluded
before retaining every 24th valid observation and the first 6,000 of them,
which gives a nominal four-hour sampling interval; lags refer to retained
observation indices. Each series is split chronologically as in
Table~\ref{tab:greedy-climate-splits}. The prefix trains the experts and
fixes the target transformation, the fitting block estimates the
L2D-SLDS parameters, and the warmup and validation blocks are used for
model selection and are then replayed with full feedback before the heldout
block, without refitting structural parameters.

\begin{table}[htbp]
\caption{Chronological sample counts.}
\label{tab:greedy-climate-splits}
\centering
\small
\begin{tabular}{lrrrrrl}
\toprule
Dataset & Prefix & Fit & Warmup & Validation & Heldout & Heldout target dates \\
\midrule
Delhi & 627 & 313 & 78 & 78 & 472 & 2016-01-02--2017-04-24 \\
Jena & 2400 & 1200 & 300 & 300 & 1800 & 2010-12-02--2011-09-27 \\
Melbourne & 1458 & 729 & 182 & 182 & 1096 & 1987-12-31--1990-12-31 \\
\bottomrule
\end{tabular}
\end{table}

\noindent\textbf{Synthetic switching data.}
To test the regime mechanism with known parameters, we generate scalar
shared and private states $\vect s_t=(g_t,u_{t,1},u_{t,2})^\top$ with
reference zero and feature one. The regime starts uniformly and follows
\[
 \mat\Pi=\begin{pmatrix}0.96&0.04\\0.05&0.95\end{pmatrix},
 \qquad
 \mat A_1=\operatorname{diag}(0.88,0.65,0.65),\quad
 \mat A_2=\operatorname{diag}(-0.55,0.35,0.35),
\]
with process covariances $\mat Q_1=\operatorname{diag}(0.035,0.012,0.012)$,
$\mat Q_2=\operatorname{diag}(0.18,0.025,0.025)$ and initial state
$\mathcal N(\vect0,\operatorname{diag}(0.15,0.04,0.04))$. The channels are
$r_t=g_t+\varepsilon_{t,r}$, $d_{t,1}=-0.9g_t+u_{t,1}+\varepsilon_{t,1}$, and
$d_{t,2}=-1.1g_t+u_{t,2}+\varepsilon_{t,2}$, with $y_t=\tanh(-r_t)$ and
$h_{t,k}=\tanh(d_{t,k})$. Emission variances $(R_{m,r},R_{m,1},R_{m,2})$ are
$(0.025,0.012,0.15)$ in regime one and $(0.04,0.20,0.018)$ in regime two, so
the better expert changes with the regime. A stationary control averages
the two regimes' dynamics and noise parameters. Each seed generates one
trajectory with a 120-round prefix replay, 80 calibration rounds, and 400
heldout rounds, two always-available experts, bound one, and fees
$(0,0.05,0.05)$.

\input{Section/AppendixParts/experiments_licenses}

\subsection{Experimental Setting}
\label{app:exp-setting}
The following setting defines the real-data benchmarks; each subsequent
study states its changes to this protocol.

\noindent\textbf{Experts and access.}
Four local ridge forecasters are trained on the prefix: an autoregression
with lags $(1,2,3)$, a multilag autoregression, the same multilag model
trained on at most the last 365 eligible prefix observations, and an
autoregression with lag one and annual sine and cosine features. The
multilag sets are $(1,7,14,30,90)$ for Delhi, $(1,6,42,180)$ for Jena, and
$(1,7,30,365)$ for Melbourne. All use ridge penalty one with an unpenalized
intercept and keep their coefficients fixed, while their lagged inputs use
previously observed targets, so the archive contains out-of-prefix
forecasts. These forecasts are hidden during deployment unless consulted,
which makes the study a controlled access simulation. All four experts are
available at every round, except in the reentry study. The scaling study
extends this pool as specified in Appendix~\ref{app:jena-pool-scaling}.

\noindent\textbf{Objective and timing.}
For a raw observation $v$, we use $y=B\tanh((v-c)/B)$, where $c$ and $B$ are
the prefix mean and population standard deviation, and apply the same map
to every expert forecast; $(c,B)$ is $(25.453,7.532)$ for Delhi,
$(7.638,9.006)$ for Jena, and $(11.043,4.269)$ for Melbourne. The reference
in inverse-link coordinates is the previous observed target. A forecast is
issued when the preceding target is observed and matures with the current
target, so every interval is completed before the next decision. External
fees are $0.22$ in Delhi and zero in Jena and Melbourne, in the units of the
combined objective; the internal action is free.

\noindent\textbf{L2D-SLDS configuration.}
L2D-SLDS uses two candidates fitted on the fitting block, with one and two
regimes, scalar shared and private states, and constant feature one. Each
fit uses five EM iterations; the one-regime E-step is exact, and the
two-regime E-step uses Gibbs sampling with eight retained and eight burn-in
draws. Deployment uses the full-covariance filter of
Appendix~\ref{app:dense-filter}, mixing weight $\zeta=0.1$,
$\epsilon=1/(H+1)$ for horizon $H$, quadrature tolerance $10^{-9}$, and
deterministic tie breaking (internal first, then ascending expert index).
L2D-SLDS has no routing hyperparameter. All baselines start from the same
fitted candidates, replay the same full-feedback history, and use the same
internal learner along their own feedback history.

\subsection{Metrics and Reporting}
\label{app:exp-metrics}
The \emph{combined cost} of a round is the squared error of the issued
forecast on the bounded scale plus the fee of the selected action, and it is
the objective of every method. We also report its two parts, the
\emph{prediction loss} and the \emph{fee}, the \emph{query fraction} (the
fraction of rounds that consult an external expert), and the \emph{internal
MSE}, the squared error of the internal forecast $f_t$ along each system's
own deployed history. Costs are not pooled across datasets, whose
transformations and fees differ. Every method is evaluated with seeds 11,
13, and 17; tables report the mean and sample standard deviation of the
three per-run averages. In the original real-data comparisons, the fits
are fixed and L2D-SLDS is deterministic, so its three runs coincide.
The Jena scaling study instead repeats structural fitting with each seed
and reports variation across these fits on the same time series.
Bold marks the lowest displayed mean within each comparison, including
ties, and does not indicate statistical significance. Query rates and
cost differences are descriptive and are not ranked.

\input{Section/AppendixParts/experiments_baselines}

\subsection{Comparison with Baselines}
\label{app:exp-main}
\label{app:greedy-climate-systems}
We compare L2D-SLDS, which routes greedily on predicted costs from one
joint law, with contextual-bandit routers that learn per-action loss
estimates from the same history and feedback. The contextual baselines use
the features $[1,\tanh(f_{\rm ref}),\tanh(f_{\rm ref})^2]$ and the
same full-covariance internal learner as L2D-SLDS, following their own
selective-feedback histories.

Table~\ref{tab:main-climate-cost} shows that L2D-SLDS has the lowest mean
combined cost on all three heldout blocks. In Delhi, where access is
charged, it consults on only eight of 472 rounds, with prediction loss
$0.950$ and mean fee $0.004$; the gain therefore comes from deciding when
consultation is worth its fee. In Jena and Melbourne, where access is free,
it consults at every round, so the gain comes from choosing the right expert.
Figure~\ref{fig:greedy-heldout-cost} shows how the advantage develops over
the heldout period. In Jena, the comparison curves pass below zero early
before L2D-SLDS gains its advantage; in Melbourne, Joint IX stays close to
L2D-SLDS, whereas the linear baselines remain above it through most of the
period.

\begin{figure}[htbp]
\centering
\includegraphics[width=\textwidth]{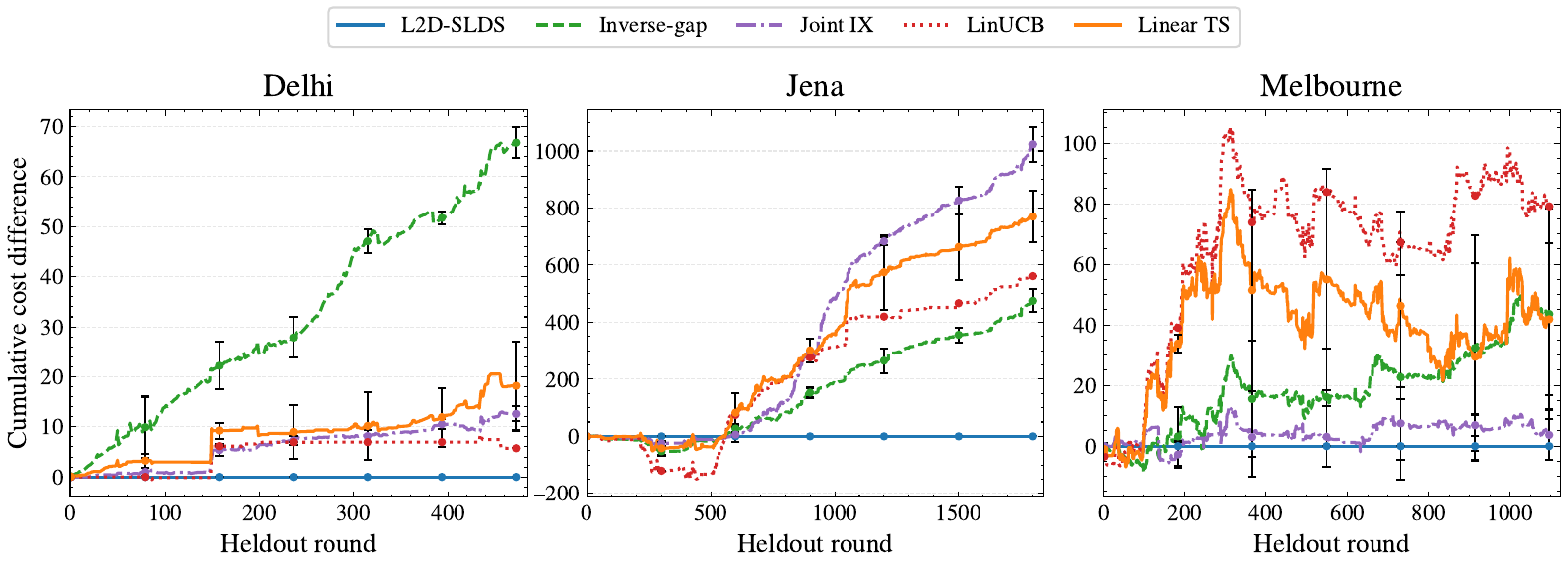}
\caption{Heldout cumulative combined cost of each method minus that of
L2D-SLDS, averaged over seeds 11, 13, and 17; positive values favor
L2D-SLDS. Error bars show one sample SD of the paired differences. Each
method follows its own deployed feedback history, and panel scales differ.}
\label{fig:greedy-heldout-cost}
\end{figure}

\input{Section/AppendixParts/experiments_consultation_cost}

\subsection{Ablation: Greedy Routing Rule}
\label{app:greedy-policy-comparison}
The comparison above changes the routing score and selection policy. To isolate the
decision rule, we keep the full-covariance learner, its fitted candidates,
historical replay, and score update fixed, and replace only the greedy rule
by two randomized policies. \emph{Inverse-gap}, following \citet{foster2020beyond}, gives each
non-minimizing action probability $1/(n+\gamma(c_{t,k}-\min_ic_{t,i}))$,
where $n$ is the number of available actions, and the remaining mass to the
minimizer. \emph{Joint IX} runs exponential weights with implicit
exploration \citep{kocak2014efficient} over a catalogue of deterministic policies, including each
candidate's greedy and always-internal policies and one fixed policy per
expert. Their scale parameters are selected on
validation from 16 values each ($\gamma$ multipliers $2^{9/2}$, $2^{9/2}$, and
$16$, and learning-rate multipliers $2^{-5/2}$, $16$, and $2^{1/2}$ for Delhi,
Jena, and Melbourne).

\begin{table}[htbp]
\caption{Routing-rule ablation with a common full-covariance learner. Mean
$\pm$ sample SD over three seeds. Prediction loss plus fee gives the
combined cost; internal MSE follows each system's feedback history.
Bold marks the lowest prediction loss and internal MSE within each dataset.}
\label{tab:greedy-policy-metrics}
\centering
\small
\begin{tabular}{llrrrr}
\toprule
Dataset & Policy & Prediction loss & Fee & Query fraction & Internal MSE \\
\midrule
Delhi & L2D-SLDS & $\mathbf{0.950}\pm0.000$ & $0.004\pm0.000$ & $0.017\pm0.000$ & $\mathbf{0.968}\pm0.000$ \\
 & Inverse-gap & $0.956\pm0.011$ & $0.140\pm0.005$ & $0.636\pm0.022$ & $0.970\pm0.001$ \\
 & Joint IX & $0.966\pm0.002$ & $0.015\pm0.001$ & $0.067\pm0.007$ & $0.981\pm0.000$ \\
\midrule
Jena & L2D-SLDS & $\mathbf{2.502}\pm0.000$ & $0.000\pm0.000$ & $1.000\pm0.000$ & $\mathbf{4.068}\pm0.000$ \\
 & Inverse-gap & $2.766\pm0.022$ & $0.000\pm0.000$ & $0.919\pm0.011$ & $4.099\pm0.010$ \\
 & Joint IX & $3.071\pm0.033$ & $0.000\pm0.000$ & $0.986\pm0.002$ & $4.140\pm0.003$ \\
\midrule
Melbourne & L2D-SLDS & $\mathbf{2.321}\pm0.000$ & $0.000\pm0.000$ & $1.000\pm0.000$ & $\mathbf{2.669}\pm0.000$ \\
 & Inverse-gap & $2.361\pm0.032$ & $0.000\pm0.000$ & $0.909\pm0.012$ & $2.680\pm0.014$ \\
 & Joint IX & $2.325\pm0.008$ & $0.000\pm0.000$ & $0.972\pm0.006$ & $2.751\pm0.001$ \\
\bottomrule
\end{tabular}
\end{table}

Table~\ref{tab:greedy-policy-metrics} shows that, with the learner's
initialization and update rule held fixed, greedy routing has the lowest prediction loss and the lowest
internal MSE on every series. In Delhi it also pays far less in fees,
$0.004$ against $0.140$ for inverse-gap, because it consults only when the
predicted improvement exceeds the fee, whereas randomized exploration
consults on most rounds. In Jena, its prediction loss is $2.502$ against
$2.766$ and $3.071$.

\subsection{Regime Switching on Synthetic Data}
\label{app:completed-mechanism-controls}
We compare the full dictionary with a \emph{one-regime} control on both
switching and stationary data. The full dictionary contains the generating
switching model and the averaged stationary model of
Appendix~\ref{app:exp-datasets}; the control removes the switching candidate.
Parameters are known, so this comparison isolates the regime mechanism
from structural fitting.

\begin{table}[htbp]
\caption{Synthetic regime study. Combined cost, mean $\pm$ sample SD over
three trajectories; $\Delta C$ is one regime minus full, so positive values
favor the full dictionary.}
\label{tab:completed-synthetic-regime}
\centering\small
\begin{tabular}{lccc}
\toprule
Data & Full & One regime & $10^3\Delta C$ \\
\midrule
Switching & $\mathbf{0.101}\pm 0.004$ & $0.178\pm 0.015$ & $76.861\pm 17.047$ \\
Stationary & $0.116\pm 0.005$ & $\mathbf{0.114}\pm 0.007$ & $-2.183\pm 1.993$ \\
\bottomrule
\end{tabular}
\end{table}

On switching data, the full dictionary reduces cost by $43\%$, from $0.178$
to $0.101$, and wins on every seed
(Table~\ref{tab:completed-synthetic-regime}). The gain concentrates after
regime changes: over the 20 rounds following each heldout change, the full
dictionary costs $0.110\pm0.007$ against $0.171\pm0.009$. On stationary data, where
the one-regime control is correctly specified, it is ahead by only
$0.002$.

\subsection{Expert Withdrawal and Reentry}
\label{app:extended-deployed-reentry}
\label{app:extended-deployed-mechanisms}
We test how L2D-SLDS handles a changing pool. For heldout horizon $H$,
experts 1, 2, and 3 are withdrawn on
$[\lfloor0.20H\rfloor,\lfloor0.35H\rfloor)$,
$[\lfloor0.45H\rfloor,\lfloor0.60H\rfloor)$, and
$[\lfloor0.70H\rfloor,\lfloor0.85H\rfloor)$, respectively, with the same
fits, fees, and baseline configurations unchanged. We compare L2D-SLDS with
the baselines under this schedule, and we ablate retention: on return,
\emph{reset} replaces the returning expert's private marginal in each
regime by its marginal at the start of deployment and removes its
cross-covariances, while keeping everything else.

\begin{table}[htbp]
\caption{Expert withdrawal and reentry on the real datasets. Combined cost,
mean $\pm$ sample SD over three seeds. Right: $\Delta C$ is reset minus
retain, so positive values favor retention.}
\label{tab:extended-reentry-cost}
\centering\small
\setlength{\tabcolsep}{3pt}
\begin{tabular}{lrrrr|rr}
\toprule
Dataset & L2D-SLDS & LinUCB & Disc.\ LinUCB & Linear TS & Reset & $10^3\Delta C$ \\
\midrule
Delhi & $\mathbf{0.954}\pm0.000$ & $0.966\pm0.000$ & $1.119\pm0.000$ & $0.992\pm0.020$ & $\mathbf{0.954}\pm0.000$ & $0.000\pm0.000$ \\
Jena & $\mathbf{2.906}\pm0.000$ & $3.089\pm0.000$ & $3.221\pm0.000$ & $3.166\pm0.074$ & $2.915\pm0.000$ & $8.693\pm0.000$ \\
Melbourne & $\mathbf{2.321}\pm0.000$ & $2.394\pm0.000$ & $2.377\pm0.000$ & $2.355\pm0.023$ & $2.327\pm0.000$ & $6.192\pm0.000$ \\
\bottomrule
\end{tabular}
\end{table}

L2D-SLDS keeps the lowest cost on all three series when experts leave and
return (Table~\ref{tab:extended-reentry-cost}). Retaining the dormant
states lowers cost relative to reset in Jena and Melbourne and ties in Delhi.
The effect concentrates right after reentry: over the first 20 rounds
after each return, retention reduces cost from $1.987$ to $1.726$ in Jena
and from $1.859$ to $1.766$ in Melbourne.

\subsection{Ablation: Covariance Representation}
\label{app:greedy-backend-comparison}
L2D-SLDS keeps the full covariance between private states, which grows
quadratically with the number of experts. The structured representation of
Appendix~\ref{app:red-star} keeps every shared--private covariance and is
linear in $K$, but drops the dependence between private states given the
shared state. We compare the two with identical fits, data, fees, and seeds.

L2D-SLDS has the same cost under both representations in Delhi and Jena (the
Delhi difference is about $-2.5\times10^{-7}$) and a lower cost with full
covariance in Melbourne, $2.321$ against $2.325$.
Table~\ref{tab:greedy-structured-controls} repeats the routing ablation with
the structured learner: the greedy rule remains the best policy on every
series.

\begin{table}[htbp]
\caption{Routing rules with the structured-covariance learner. Mean
combined cost $\pm$ sample SD over three seeds.}
\label{tab:greedy-structured-controls}
\centering\small
\begin{tabular}{lrrr}
\toprule
Policy & Delhi & Jena & Melbourne \\
\midrule
L2D-SLDS (structured) & $\mathbf{0.954}\pm0.000$ & $\mathbf{2.502}\pm0.000$ & $\mathbf{2.325}\pm0.000$ \\
Inverse-gap & $1.095\pm0.007$ & $2.766\pm0.022$ & $2.360\pm0.033$ \\
Joint IX & $0.981\pm0.003$ & $3.073\pm0.027$ & $\mathbf{2.325}\pm0.008$ \\
\bottomrule
\end{tabular}
\end{table}

We next test this storage--accuracy trade-off as the expert pool grows.

\input{Section/AppendixParts/experiments_scaling}

%% file: Section/AppendixParts/experiments_licenses.tex
\subsection{Dataset Licenses and Reproducibility}
\label{app:exp-licenses}
We use publicly released datasets and retain their source attributions
and release-specific licenses. Delhi is the \emph{Daily Climate Time
Series Data} release by Sumanth V. Rao, listed under CC0 (Public
Domain).\footnote{\url{https://www.kaggle.com/datasets/sumanthvrao/daily-climate-time-series-data}}
Jena is provided by the Max Planck Institute for Biogeochemistry, whose
weather-data download page specifies CC BY 4.0.\footnote{\url{https://www.bgc-jena.mpg.de/wetter/weather_data.html}}
The Melbourne series is attributed to the Australian Bureau of Meteorology
through the Time Series Data Library; the Sam Faraday release used here is
listed under the Community Data License Agreement, Permissive, Version
1.0.\footnote{\url{https://www.kaggle.com/datasets/samfaraday/daily-minimum-temperatures-in-me}}
Our preprocessing and chronological splits are specified in
Appendices~\ref{app:exp-datasets} and~\ref{app:exp-setting}.

We provide end-to-end reproduction code covering data preparation,
expert construction, historical fitting and calibration, baseline tuning,
selective-feedback evaluation, and the generation of tables and curves.
The accompanying instructions specify dependencies, configurations,
seeds, and execution order; source and data hashes bind each run to its
inputs. Dataset downloads retain their original terms, separately from
the experimental code.

%% file: Section/AppendixParts/experiments_baselines.tex
\subsection{Baseline Algorithms and Tuning}
\label{app:greedy-baselines}
The contextual baselines estimate fee-free squared prediction losses,
normalized by $C_*$. Their action scores add $\beta_k/C_*$ once,
and selection is restricted to the internal action and available experts.
All use the context features
$[1,\tanh(f_{\rm ref}),\tanh(f_{\rm ref})^2]$.
Joint linear models concatenate these features, an action indicator,
and action--context interactions; disjoint models retain one regression
per action. Unperturbed ridge controls start with a ridge design matrix
and zero response sums; the ensemble uses the randomized priors below.

For observed feedback pairs $(s,k)$, let
$L_{s,k}=(h_{s,k}-y_s)^2/C_*$ and let $\vect\phi_{s,k}$ denote the
joint feature vector. A joint ridge control maintains
\[
 \mat V=\lambda\mat I+\sum_{(s,k)\ {\rm observed}}
                \vect\phi_{s,k}\vect\phi_{s,k}^{\top},
 \qquad
 \widehat{\vect\theta}=\mat V^{-1}
          \sum_{(s,k)\ {\rm observed}}\vect\phi_{s,k}L_{s,k}.
\]
Disjoint LinUCB uses the same construction separately for each action,
with the three-dimensional context vector in place of joint features.
Shared LinUCB minimizes
$\widehat{\vect\theta}^{\top}\vect\phi_{t,k}
-\alpha\sqrt{\vect\phi_{t,k}^{\top}\mat V^{-1}\vect\phi_{t,k}}
+\beta_k/C_*$ over $\mathcal A_t$.
Linear TS instead draws
$\widetilde{\vect\theta}\sim
\mathcal N(\widehat{\vect\theta},s^2\mat V^{-1})$
and minimizes $\widetilde{\vect\theta}^{\top}\vect\phi_{t,k}+\beta_k/C_*$.
With four experts, the joint feature dimension is
$3+5+5\cdot3=23$; disjoint regressions have three coordinates per action.

\noindent\textbf{Linear and randomized selection.}
LinUCB \citep{li2010contextual} uses a disjoint ridge estimate and selects
the smallest predicted loss minus its uncertainty bonus plus fee. Shared
LinUCB applies the same criterion to the joint feature map, as in the hybrid
model of \citet{li2010contextual}. Linear Thompson sampling
\citep{agrawal2013thompson} draws a Gaussian parameter around the joint
ridge estimate with covariance equal to the squared posterior scale times
the inverse design matrix. Ensemble sampling \citep{lu2017ensemble}
maintains 20 randomized-prior, perturbed-response linear models and samples
one member for each decision.

\noindent\textbf{Non-stationary linear controls.}
Discounted LinUCB \citep{russac2019weighted} exponentially downweights
past observations, using
both the discounted design and squared-weight design for its uncertainty
radius. Discounting occurs once per feedback round, before adding all
observed action losses from that round. CUSUM-LinUCB and GLR-LinUCB are contextual adaptations of the CUSUM
\citep{liu2018change} and GLR \citep{besson2022efficient} change-detection
bandits: they monitor bounded absolute prediction errors relative
to a frozen reference fitted after each reset. CUSUM resets the affected
action; a GLR alarm resets all action regressions. These controls use
forced random selection probabilities of $0.05$ and $0.10$, respectively.

\noindent\textbf{NeuralUCB.}
The action-conditioned network concatenates the context and action
indicator, normalizes the input, and uses one shared ReLU hidden layer
and a scalar output. Symmetric initialization gives zero initial output.
The uncertainty bonus uses the diagonal accumulated gradient covariance;
training performs ten minibatch steps every ten observed rounds, with
batch size at most 64 and regularization toward initialization. This is
the implemented NeuralUCB control, following the neural gradient-score
construction of \citet{zhou2020neuralcontextualbanditsucbbased}.

\noindent\textbf{Historical and live feedback.}
The fitting prefix initializes each baseline from all archived external
losses. Calibration adds the internal loss of its causal common learner
and every archived external loss. At deployment, the target always
reveals the cached internal prediction's loss; a consultation additionally
reveals the selected expert's loss. No unqueried external loss enters an
update. The probabilistic internal learner uses the same fitted snapshots
and follows each router's own observation history. Backend choices for
each comparison are specified with its results.

\begin{table}[htbp]
\caption{Declared baseline search grids. Cartesian products give exactly
16 configurations per method. Each dataset selects its configuration on
chronological validation with seed 101, before heldout evaluation.}
\label{tab:greedy-baseline-grids}
\centering\small
\begin{tabular}{lp{0.70\textwidth}}
\toprule
Method & Tuned settings and fixed settings\\
\midrule
LinUCB / Shared LinUCB & Bonus multiplier $\{0,0.05,0.2,1\}$;
ridge $\{0.01,0.1,1,10\}$.\\
Linear TS & Posterior scale $\{0.01,0.05,0.2,1\}$;
ridge $\{0.01,0.1,1,10\}$.\\
Ensemble & Perturbation standard deviation $\{0.01,0.05,0.2,1\}$;
ridge $\{0.01,0.1,1,10\}$; 20 members.\\
Discounted LinUCB & Discount $\{0.90,0.97,0.99,0.999\}$;
bonus multiplier $\{0.01,0.05,0.2,1\}$; ridge 1,
noise scale $0.5$, parameter-norm bound 1, confidence parameter $0.05$.\\
CUSUM-LinUCB & Bonus $\{0.05,0.5\}$; threshold $\{0.1,0.3,1,3\}$;
slack $\{0.005,0.02\}$; warmup and reference length 25, ridge 1.\\
GLR-LinUCB & Bonus $\{0.05,0.5\}$; detector level
$\{0.001,0.01,0.05,0.2\}$; minimum window $\{10,25\}$;
reference length 25, ridge 1.\\
NeuralUCB & Bonus $\{0,0.05,0.2,1\}$; learning rate
$\{10^{-4},10^{-3}\}$; hidden width $\{16,64\}$; ridge 1.\\
\bottomrule
\end{tabular}
\end{table}

Validation trials start from fresh copies of the same historical learner.
Selection minimizes mean combined validation cost, breaking ties by the
declared configuration order; all trials are retained. The selected
settings are frozen across evaluation seeds. L2D-SLDS uses its stated
greedy rule without a routing-policy search parameter.

%% file: Section/AppendixParts/experiments_consultation_cost.tex
% Source: out/jena_consultation_sweep_20261001_v1/analysis/summary.csv.
% Integrity checks: scripts/analyze_jena_consultation_sweep.py.
% 110 cost--method pairs, 330 heldout evaluations; no structural refitting.
\subsection{Consultation-Cost Sensitivity on Jena}
\label{app:jena-consultation-cost}

We vary the consultation fee on Jena to test the decision to defer, holding
the expert pool and data fixed. As access becomes costly, an expert's
predicted improvement must justify its fee.

\noindent\textbf{Settings.}
We reuse the Jena archive, chronological splits, four experts, target
transformation, and fitted SLDS parameters from
Appendices~\ref{app:exp-datasets} and~\ref{app:exp-setting}.
The internal action remains free, and every external expert has fee
$\beta\in\{0,0.05,0.10,0.20,0.40,0.80,1.60,3.20,6.40\}$.
All nine methods use full-covariance inference and receive the same
full-feedback history, followed by selected-expert feedback during
deployment. At each fee, each contextual baseline searches its original
16 configurations on historical validation with seed 101; its selected
configuration is saved before heldout evaluation. L2D-SLDS has no routing
hyperparameter to tune. Historical calibration is replayed at the stated
fee, while structural parameters remain fixed. All methods are evaluated
over three seeds, with the mean and sample standard deviation reported
as in Appendix~\ref{app:exp-metrics}.

We also evaluate $\beta=4B^2+1\simeq325.411$ as a bounded-loss endpoint.
Since both the target and every issued prediction lie in $[-B,B]$,
the internal squared loss is at most $4B^2$, whereas any external action's
combined cost is at least $\beta$. Greedy expected-cost minimization must
therefore choose the internal action at this endpoint. Separately labelled
always-defer and never-defer reference controls constrain greedy selection
to the external experts or to the internal action, respectively; these
controls are not variants obtained by changing the consultation fee.

Figure~\ref{fig:jena-consultation-cost} in Section~\ref{sec:evaluation} plots
combined cost and query percentage against the fee.

\begin{table}[htbp]
\caption{Jena cost sensitivity. Mean $\pm$ sample SD over three seeds.
The best contextual baseline is identified separately at each fee; this
column is a descriptive minimum over the eight baselines, not an additional
deployed policy. Bold marks the lowest displayed combined cost, including
ties. Query percentages are descriptive and are not ranked.}
\label{tab:jena-consultation-cost}
\centering\small
\setlength{\tabcolsep}{4pt}
\begin{tabular}{rllrr}
\toprule
Fee & L2D-SLDS cost & Best contextual baseline & Baseline cost & L2D-SLDS queries (\%) \\
\midrule
0 & $\mathbf{2.502}\pm0.000$ & Ensemble & $2.768\pm0.021$ & $100.0\pm0.0$ \\
0.05 & $\mathbf{2.552}\pm0.000$ & Ensemble & $2.789\pm0.029$ & $95.4\pm0.0$ \\
0.10 & $\mathbf{2.585}\pm0.000$ & LinUCB & $2.779\pm0.000$ & $91.6\pm0.0$ \\
0.20 & $\mathbf{2.669}\pm0.000$ & LinUCB & $2.874\pm0.000$ & $85.8\pm0.0$ \\
0.40 & $\mathbf{2.843}\pm0.000$ & Shared LinUCB & $3.112\pm0.000$ & $75.3\pm0.0$ \\
0.80 & $\mathbf{3.078}\pm0.000$ & LinUCB & $3.290\pm0.000$ & $60.6\pm0.0$ \\
1.60 & $\mathbf{3.319}\pm0.000$ & Linear TS & $3.521\pm0.035$ & $39.9\pm0.0$ \\
3.20 & $\mathbf{3.752}\pm0.000$ & Discounted LinUCB & $3.950\pm0.000$ & $13.7\pm0.0$ \\
6.40 & $\mathbf{4.183}\pm0.000$ & LinUCB (tied) & $\mathbf{4.183}\pm0.000$ & $0.0\pm0.0$ \\
325.411 & $\mathbf{4.186}\pm0.000$ & LinUCB (tied) & $\mathbf{4.186}\pm0.000$ & $0.0\pm0.0$ \\
\bottomrule
\end{tabular}
\end{table}

L2D-SLDS moves from 100\% consultation at zero fee to 0\% at $\beta=6.4$,
with intermediate query percentages of 75.3\%, 39.9\%, and 13.7\% at
$\beta=0.4$, $1.6$, and $3.2$, respectively
(Figure~\ref{fig:jena-consultation-cost}). Its mean combined cost is below
that of every contextual baseline at all eight fees below $6.4$ and ties
the best baseline at the two highest fees
(Table~\ref{tab:jena-consultation-cost}). For example, at $\beta=0.8$ it
achieves cost $3.078$ against the best baseline's $3.290$, while querying
on 60.6\% of rounds. At $\beta=6.4$, its cost equals that of the
never-defer reference, $4.183$; always deferring costs $8.898$.
Thus the reduction in consultation follows the fee-aware decision rule
itself, rather than a prescribed query budget or a forced endpoint.

%% file: Section/AppendixParts/experiments_scaling.tex
% Audited sources: out/jena_review_studies_20261003_v1 and
% out/jena_baseline_scaling_20261004_v1. Tables and figure are generated by
% scripts/build_jena_scaling_paper.py; no new fitting is performed.

\subsection{Expert-Pool Scaling on Jena}
\label{app:jena-pool-scaling}
We compare prediction cost, deployment time, and posterior storage as the
pool grows. Full and structured covariance share fitted parameters and
follow their own selective-feedback histories.

\noindent\textbf{Setting.}
We retain the Jena split of Table~\ref{tab:greedy-climate-splits}, zero fees,
and the original four experts, then extend the pool to
$K\in\{4,8,16,32,64\}$ with prefix-trained ridge forecasters.
The additional library crosses four lag families,
$(1,2,3)$, $(1,6,42,180)$, $(1)$ with annual sine/cosine, and
$(1,2,6,12,42,180)$ with annual sine/cosine,
with penalties $0.01,0.1,10,100$ and prefix windows of all available
observations or the last $365,730,1460$ eligible observations.
Training rows start at lag 180. Lagged targets use the prefix mean and
standard deviation; the intercept is unpenalized.
Enumerating penalties, windows, then lag families in the stated order, we
take the first $K-4$ additional forecasts to obtain nested pools.
For each pool, seeds 11, 13, and 17 independently initialize the one- and
two-regime fits; each receives five EM iterations with the unchanged
Gibbs settings of Appendix~\ref{app:exp-setting}. Both covariance variants
replay the same full-feedback archive before 1,800 heldout rounds.

\begin{figure}[htbp]
\centering
\includegraphics[width=\textwidth]{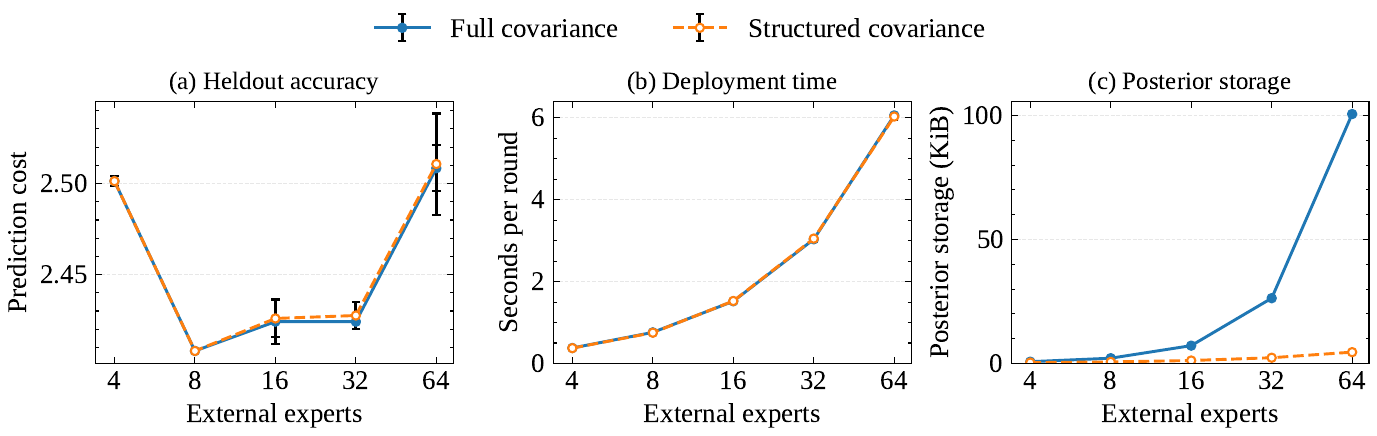}
\caption{Jena expert-pool scaling. Mean and one sample SD across three
fitting seeds on the same heldout series: (a) prediction cost at zero fee,
(b) live selection and feedback time, excluding fitting and historical
replay, and (c) retained posterior-array storage. Pool sizes are spaced
logarithmically.}
\label{fig:jena-pool-scaling}
\end{figure}

\begin{table}[htbp]
\centering\small
\caption{Heldout prediction cost and posterior storage on Jena. Costs are
mean $\pm$ sample SD over three fitting seeds; posterior-array sizes have
zero SD. Bold marks the lower displayed cost or storage within each pair.}
\label{tab:jena-scaling-accuracy}
\input{tables/jena_scaling/accuracy_storage_rows}
\end{table}

The two representations have identical mean cost at four and eight experts;
their relative mean-cost difference remains below $0.15\%$ at every tested
pool size (Table~\ref{tab:jena-scaling-accuracy}). At 64 experts, structured
covariance reduces posterior storage from $100.570$ to $4.570$ KiB, a
$95.5\%$ reduction, while the mean costs are $2.5084$ and $2.5107$.
Over the tested range, a sixteen-fold increase in pool size gives an
approximately sixteen-fold increase in mean live runtime for either
representation (Figure~\ref{fig:jena-pool-scaling}).

\subsection{Runtime and Storage Against Baselines}
\label{app:jena-computational-scaling}
We compare complete routing systems, including their common internal
forecasting learner, and separately measure routing-policy overhead.

\noindent\textbf{Measurement protocol.}
For every pool and fitting seed above, we run all eight baselines and both
L2D-SLDS representations on the first 70 heldout rounds, discarding the
first ten from timing summaries. Each method runs in a fresh process on
matched hardware with one allocated CPU and one BLAS thread
(Intel Xeon E5-2620 v3, 2.40 GHz). Baselines retain their original
validation-selected configurations and use the full-covariance internal
learner; only structured L2D-SLDS changes that representation.
Table~\ref{tab:jena-baseline-latency} reports live selection and feedback,
including forecasting but excluding historical replay and fitting.
Accuracy is evaluated on the full 1,800-round horizon above.

\begin{table}[htbp]
\centering\small
\setlength{\tabcolsep}{4pt}
\caption{End-to-end seconds per round on Jena, mean $\pm$ sample SD of
three per-seed timing averages. Each average uses 60 measured rounds after
ten warmup rounds. All methods include their internal learner's computation.
Bold marks the lowest displayed mean in each column, including ties.}
\label{tab:jena-baseline-latency}
\input{tables/jena_scaling/latency_rows}
\end{table}

\noindent\textbf{Runtime.}
At 64 experts, mean system latency ranges from $6.118$ to $6.267$ seconds
across methods. Predictive-moment computation accounts for
$98.2$--$99.8\%$ of system time across the tested methods and pool sizes,
which explains the similar end-to-end timings. Greedy routing uses these
moments directly: its additional policy computation takes approximately
$0.020$--$0.021$ ms at 64 experts
(Table~\ref{tab:jena-baseline-storage}). Policy time includes action
selection and any baseline feedback update, but excludes forecasting.

\begin{table}[H]
\centering\small
\setlength{\tabcolsep}{4pt}
\caption{Storage and policy overhead at the largest tested pool ($K=64$).
Posterior and policy-buffer sizes are in KiB and have zero SD across seeds;
policy time and peak process memory are mean $\pm$ sample SD.
Bold marks the lowest displayed value in each column, including ties.}
\label{tab:jena-baseline-storage}
\input{tables/jena_scaling/memory_policy_rows}
\end{table}

\noindent\textbf{Storage.}
Posterior storage counts retained filter arrays at initialization; policy
storage counts final unique NumPy buffers beyond the forecasting model.
All methods additionally store $69.727$ KiB of model-parameter arrays at
$K=64$. Greedy routing needs no separate policy NumPy buffers, whereas
the baselines retain their policy estimates. Each baseline's posterior
arrays belong to the common full-covariance internal learner.
Peak resident memory (RSS) measures the whole process,
including imports, setup, and replay: it is about $118$--$119$ MiB for
L2D-SLDS and the linear baselines, and $228.05$ MiB for NeuralUCB at this
pool size.

%% file: tables/jena_scaling/accuracy_storage_rows.tex
% Generated by scripts/build_jena_scaling_paper.py.
\begin{tabular}{rrrrrr}
\toprule
& \multicolumn{2}{c}{Prediction cost} & \multicolumn{2}{c}{Posterior (KiB)} & \\
\cmidrule(lr){2-3}\cmidrule(lr){4-5}
$K$ & Full & Structured & Full & Structured & Reduction \\
\midrule
4 & $\mathbf{2.5013}\pm0.0029$ & $\mathbf{2.5013}\pm0.0029$ & 0.727 & \textbf{0.352} & 51.6\% \\
8 & $\mathbf{2.4081}\pm0.0003$ & $\mathbf{2.4081}\pm0.0003$ & 2.133 & \textbf{0.633} & 70.3\% \\
16 & $\mathbf{2.4242}\pm0.0124$ & $2.4259\pm0.0099$ & 7.195 & \textbf{1.195} & 83.4\% \\
32 & $\mathbf{2.4241}\pm0.0012$ & $2.4275\pm0.0076$ & 26.320 & \textbf{2.320} & 91.2\% \\
64 & $\mathbf{2.5084}\pm0.0126$ & $2.5107\pm0.0277$ & 100.570 & \textbf{4.570} & 95.5\% \\
\bottomrule
\end{tabular}

%% file: tables/jena_scaling/latency_rows.tex
% Generated by scripts/build_jena_scaling_paper.py.
\begin{tabular}{lrrrrr}
\toprule
Method & $K=4$ & $K=8$ & $K=16$ & $K=32$ & $K=64$ \\
\midrule
L2D-SLDS (full) & $0.383\pm0.010$ & $0.762\pm0.018$ & $1.544\pm0.057$ & $3.059\pm0.043$ & $\mathbf{6.118}\pm0.091$ \\
L2D-SLDS (structured) & $\mathbf{0.377}\pm0.007$ & $\mathbf{0.756}\pm0.013$ & $1.542\pm0.034$ & $\mathbf{3.027}\pm0.049$ & $\mathbf{6.118}\pm0.115$ \\
LinUCB & $\mathbf{0.377}\pm0.002$ & $0.765\pm0.014$ & $1.546\pm0.036$ & $3.162\pm0.223$ & $6.120\pm0.119$ \\
Shared LinUCB & $0.391\pm0.019$ & $0.769\pm0.012$ & $\mathbf{1.518}\pm0.004$ & $3.064\pm0.044$ & $6.203\pm0.097$ \\
Linear TS & $0.379\pm0.002$ & $0.769\pm0.012$ & $1.533\pm0.041$ & $3.088\pm0.052$ & $6.144\pm0.142$ \\
Ensemble & $\mathbf{0.377}\pm0.004$ & $0.774\pm0.008$ & $1.519\pm0.008$ & $3.060\pm0.074$ & $6.177\pm0.091$ \\
Discounted LinUCB & $0.380\pm0.003$ & $0.764\pm0.014$ & $1.548\pm0.033$ & $3.066\pm0.047$ & $6.191\pm0.131$ \\
CUSUM-LinUCB & $0.384\pm0.006$ & $0.771\pm0.015$ & $1.522\pm0.001$ & $3.050\pm0.063$ & $6.247\pm0.174$ \\
GLR-LinUCB & $0.387\pm0.007$ & $0.767\pm0.034$ & $1.534\pm0.031$ & $3.167\pm0.116$ & $6.159\pm0.105$ \\
NeuralUCB (diagonal) & $0.383\pm0.005$ & $0.761\pm0.013$ & $1.585\pm0.065$ & $3.089\pm0.058$ & $6.267\pm0.254$ \\
\bottomrule
\end{tabular}

%% file: tables/jena_scaling/memory_policy_rows.tex
% Generated by scripts/build_jena_scaling_paper.py.
\begin{tabular}{lrrrr}
\toprule
Method & Posterior (KiB) & Policy buffers (KiB) & Policy time (ms) & Peak RSS (MiB) \\
\midrule
L2D-SLDS (full) & 100.570 & \textbf{0.000} & $0.0210\pm0.0014$ & $118.73\pm0.19$ \\
L2D-SLDS (structured) & \textbf{4.570} & \textbf{0.000} & $\mathbf{0.0203}\pm0.0019$ & $118.53\pm0.38$ \\
LinUCB & 100.570 & 6.602 & $1.73\pm0.29$ & $118.53\pm0.20$ \\
Shared LinUCB & 100.570 & 542.945 & $69.95\pm0.37$ & $118.41\pm0.12$ \\
Linear TS & 100.570 & 542.945 & $5.68\pm0.20$ & $118.35\pm0.19$ \\
Ensemble & 100.570 & 666.227 & $5.29\pm0.23$ & $118.35\pm0.38$ \\
Discounted LinUCB & 100.570 & 1083.328 & $9.71\pm0.40$ & $118.47\pm0.11$ \\
CUSUM-LinUCB & 100.570 & 8.125 & $1.85\pm0.25$ & $118.47\pm0.11$ \\
GLR-LinUCB & 100.570 & 8.125 & $2.22\pm0.15$ & $\mathbf{118.27}\pm0.21$ \\
NeuralUCB (diagonal) & 100.570 & 61697.133 & $8.78\pm0.29$ & $228.05\pm0.14$ \\
\bottomrule
\end{tabular}